\documentclass[10pt, journal, twocolumn]{IEEEtran}

\usepackage[utf8]{inputenc}
\usepackage{amsmath}
\usepackage{amsthm}
\usepackage{amsfonts}
\usepackage{amssymb}
\usepackage{graphicx}
\usepackage{graphics}
\usepackage{float}
\usepackage{tikz}
\usetikzlibrary{shapes, snakes, patterns}
\usepackage{xcolor}
\newcommand{\cfbox}[2]{%
	\colorlet{currentcolor}{.}%
	{\color{#1}%
		\fbox{\color{currentcolor}#2}}%
}
\usepackage[export]{adjustbox}
\usepackage{subcaption}
\usepackage{algorithm}
\usepackage{algpseudocode}
\usepackage{epstopdf}
\newtheorem{theorem}{Theorem}[]
\newtheorem{lemma}[]{Lemma}
\newtheorem{corollary}[]{Corollary}
\newtheorem{remark}[]{Remark}
\newtheorem{Assumption}[]{Assumption}
\newtheorem{definition}[]{Definition}
\usepackage[justification=centering]{caption}
\usepackage{enumerate} 
\usepackage{cite}
\usepackage{suffix}
\usepackage{mathtools}
\usepackage{tabularx, booktabs}
\usepackage{multirow}
\usepackage{diagbox}
\newcolumntype{C}{>{\centering\arraybackslash}X} 
\usepackage[font=small, labelfont=bf]{caption}

\title{Structured and Unstructured Outlier Identification for Robust PCA: A Non iterative, Parameter free Algorithm.}
\begin{document}
\author{\IEEEauthorblockN{Vishnu Menon,
		Sheetal Kalyani\\}
	\IEEEauthorblockA{
		Department of Electrical Engineering, Indian Institute of Technology Madras\\
		Chennai, India - 600036\\
		Email: ee16s301@ee.iitm.ac.in,
		skalyani@ee.iitm.ac.in,
		}}
\maketitle
\begin{abstract}
	Robust PCA, the problem of PCA in the presence of outliers has been extensively investigated in the last few years. Here we focus on Robust PCA in the outlier model where each column of the data matrix is either an inlier or an outlier. Most of the existing methods for this model assumes either the knowledge of the dimension of the lower dimensional subspace or the fraction of outliers in the system. However in many applications knowledge of these parameters is not available. Motivated by this we propose a parameter free outlier identification method for robust PCA which a) does not require the knowledge of outlier fraction, b) does not require the knowledge of the dimension of the underlying subspace, c) is computationally simple and fast d) can handle structured and unstructured outliers. Further, analytical guarantees are derived for outlier identification and the performance of the algorithm is compared with the existing state of the art methods in both real and synthetic data for various outlier structures.
\end{abstract}
\section{Introduction}
Principal Component Analysis (PCA) \cite{jolliffe2002principal} is a very widely used technique in data analysis and dimensionality reduction. Singular Value Decomposition (SVD) of the data matrix $\mathbf{M}$ \cite{shlens2014tutorial} is known to be very sensitive to extreme corruptions in the data \cite{candes2011robust}, \cite{xu2010robust}, \cite{rahmani2016coherence} and hence robustifying the PCA process becomes a necessity. Robust PCA is typically an ill posed problem and it is of significant importance in a wide variety of fields like computer vision, machine learning, survey data analysis and so on. Recent survey papers \cite{vaswani2018static}, \cite{lerman2018overview} outline the various existing techniques for robust subspace recovery and robust PCA. Of the numerous approaches to robust PCA over the years \cite{ammann1993robust}, \cite{de2003framework}, one way to model extreme corruptions in the given data matrix $\mathbf{M}$, is using the following decomposition \cite{wright2009robust}, \cite{chandrasekaran2011rank}, \cite{zhou2010stable}, \cite{candes2011robust}: $\mathbf{M}  = \textbf{L}+\textbf{S}$, where $\textbf{S}$ encapsulates all the corruptions and is assumed to be sparse and $\textbf{L}$ is low rank. Thus robust PCA becomes a process of decomposing the given matrix into a low rank matrix plus a sparse matrix. The problem can be formulated as a convex
problem, using techniques of convex relaxation inspired from
compressed sensing \cite{candes2006compressive}, as \cite{wright2009robust}, \cite{candes2011robust}
\begin{equation}\label{ereformul}
\begin{aligned}
& \underset{\textbf{L}, \textbf{S}}{\text{minimize}}
& & \|\textbf{L}\|_*+\lambda \|\textbf{S}\|_1
& \text{s.t} && \mathbf{M} = \textbf{L}+\textbf{S}, 
\end{aligned}
\end{equation}
where $ \|\textbf{L}\|_*$ is the nuclear norm computed as the sum of singular values of a matrix and $\|\textbf{S}\|_1$ is the $l_1$ norm of vector formed by vectorizing the matrix. In \cite{candes2011robust}, an optimal value for $\lambda$ was proposed and theoretical guarantees for the exact recovery of the low rank matrix was given assuming the popular uniform sparsity model. To solve (\ref{ereformul}), several algorithms were proposed including  \cite{yi2016fast}, \cite{hsu2011robust}, \cite{chiang2016robust} with the aim of reducing the complexity of the process and improving speed and performance. Non-convex algorithms have also been proposed for robust PCA\cite{netrapalli2014non}, \cite{kang2015robust} which are significantly faster than convex programs. 

Another popular model, the one that we will adopt in this paper, is the outlier model\footnote{Throughout the paper, the term outlier model indicates the model where each column of $\mathbf{M}$ is either an inlier or an outlier}. In this model each column in $\mathbf{M}$ is considered as a data point in $\mathbb{R}^n$. The points that lie in a lower dimensional subspace of dimension $r$ are the inliers and others which do not fit in this subspace are the outliers. 
Several methods have been developed over the years, like methods based on influence functions \cite{jolliffe2002principal}, the re-weighted least squares method \cite{ma2015generalized}, methods based on random consensus (RANSAC) \cite{fischler1987random}, based on rotational invariant $l_1$ norms \cite{ding2006r} etc for the outlier model. In \cite{xu2010robust}, a convex formulation of the process is given and iterative methods have been proposed to solve it. Also the problem has been extended to identifying outliers when the inlying points come from a union of subspaces as in \cite{soltanolkotabi2012geometric,soltanolkotabi2014robust,you2017provable,heckel2015robust}. Recent works have attempted to develop simple non iterative algorithms for robust PCA with the outlier model  \cite{rahmani2016coherence}. Other methods which aims at solving robust PCA through this model include \cite{lerman2014fast}, \cite{zhang2014novel} and works based on thresholding like \cite{cherapanamjeri2017thresholding},\cite{heckel2015robust}. Most of the algorithms proposed are either iterative and complex and/or would require the knowledge of either the outlier fraction or the dimension of the low rank subspace or would have free parameters that needs to be set according to the data statistics. In this paper, we aim to propose an algorithm for removal of outliers that is computationally simple, non iterative and parameter free. Classical methods for PCA may be applied for subspace recovery after outlier removal. 
\subsection{Related work}
We briefly describe some of the key literature in the area of robust PCA and highlight how our proposed work differs from and/or is inspired by them. The popular work \cite{candes2011robust}, assuming a uniform sparsity model on the corruptions, solves (\ref{ereformul}) using Augmented Lagrange Multiplier (ALM) \cite{lin2010augmented} which is an iterative process that requires certain parameters to be set. Ours uses an outlier model and hence we cannot compare our method with the work  in \cite{candes2011robust}. In an outlier model, \cite{xu2010robust} proposes solving the following convex optimization problem for robust PCA:
\begin{equation}
\begin{aligned}
& \underset{\textbf{L}, \textbf{S}}{\text{minimize}}
& & \|\textbf{L}\|_*+\lambda \|\textbf{S}\|_{1, 2}
& \text{s.t} && \mathbf{M} = \textbf{L}+\textbf{S}, 
\end{aligned}
\end{equation}
where $\|\textbf{S}\|_{1, 2}$ is the sum of $l_2$ norms of the columns of the matrix. The paper also proposes a value for the parameter, namely $\lambda = \frac{3}{7\sqrt{\gamma N}}$, where $\gamma$ is the fraction of outliers in the system. While \cite{xu2010robust} assumes the knowledge of $\gamma$, in many cases $\gamma$ is typically unknown. Another recent work \cite{cherapanamjeri2017thresholding} that bases its algorithms on thresholding also requires the knowledge of the target rank, i.e. the dimension of the subspace. The work in \cite{soltanolkotabi2012geometric} analyzes the removal of outliers from a system where the inliers come from a union of subspaces and involves solving multiple $l_1$ optimization problems. While there exists a lot of existing techniques and algorithms \cite{yang2010fast} for solving the $l_1$ optimization problem, most of them requires certain parameters to be set and are iterative. After solving the optimization problem, a data point is classified as an inlier or outlier using thresholding in \cite{soltanolkotabi2012geometric}. Although the proposed threshold in \cite{soltanolkotabi2012geometric} is independent of the dimension of the subspace $r$ or the number of outliers, the underlying optimization problem is not parameter free and since multiple optimization problems have to be solved, the procedure is also rather complex. Another self-representation based algorithm for detecting outliers from a union of subspaces is proposed in \cite{you2017provable}, based on random walks in a graph, but it is iterative and requires multiple parameters to be set.

A fast algorithm for robust PCA was recently proposed in \cite{rahmani2016coherence} which involves looking at the coherence of the data points with other points and identifying outliers as those points which have less coherence with the other points. The authors give theoretical guarantees for the working of the algorithm for the outlier model. In the two methods that have been proposed for identifying the true subspace, knowledge of either the number of outliers or the dimension of the underlying subspace is required. Recently a parameter free algorithm, for outlier removal was proposed in \cite{menon2018fast} based on a threshold on the minimum angle formed by outlier points. But this method can only detect unstructured outliers and the threshold is conservative. Another work that is partly similar to the proposal in this paper is the method described in \cite{heckel2015robust} for outlier detection. This proposes a tuning free threshold on the maximum coherence value of a point with other points to classify it as an outlier or inlier. However in \cite{heckel2015robust}, the threshold is loose and more importantly like \cite{menon2018fast} it can detect only unstructured outliers. The outlier removal algorithm in this work is in spirit a parameter free extension to the work in \cite{rahmani2016coherence} and can detect both structured and unstructured outliers.
\subsection{Motivation and proposed approach}
The main motivation behind this work is to build parameter free algorithms for robust PCA. By parameter free we mean an algorithm which does not require the knowledge of parameters such as the dimension of true subspace or the number of outliers in the system nor it has a tuning parameter which has to be tuned according to the data. Tuning parameters in any algorithm present a challenge, as the user then would have to decide either through cross validation  \cite{arlot2010survey} or prior knowledge on how to set them. Especially in an unsupervised scenario where the algorithm needs to adapt to the data at hand on the run, setting an appropriate value of a parameter becomes an issue and incorrect settings can lead to a huge change in performance. How to estimate a proper set of parameter values when you do not have data to train and validate or prior knowledge about the nature of data is an important question. Recently, there have been attempts to make algorithms parameter free in the paradigm of sparse signal recovery \cite{DBLP:conf/icml/KallummilK18, vats2014path}, \cite{stoica2012spice}, \cite{lederer2015don} and these were shown to have results comparable with the ones when the true parameters such as the sparsity of the signal were known. Motivated by this, in this paper we propose a parameter free algorithm for robust PCA. While there exists a vast literature on robust PCA algorithms, to make parameter free variants of them, one would have come up with novel modifications for each of them separately. 
In this work, we focus on obtaining a computationally efficient parameter free algorithm for outlier removal in robust PCA. The recent work in \cite{rahmani2016coherence} is both simple and non iterative in the sense that it is a one shot process which does not involve an iterative procedure to solve an optimization problem like in \cite{soltanolkotabi2012geometric} and can accommodate both structured and unstructured outliers\footnote{The mathematical definitions of structured and unstructured outliers are described in Section \ref{s1}}. However it is not parameter free and we aim to make a parameter free variant with similar or better capabilities. We propose algorithms for robust PCA in the outlier model for the cases of both structured and unstructured outliers, which does not require the knowledge of number of outliers or dimension of the underlying subspace. Our key contributions are:
\begin{itemize}
\item [i] We develop a parameter free threshold which when used in the algorithm can guarantee the removal of all unstructured outliers with a high probability.
\item [ii] We further develop a technique which too is parameter free and can separate the remaining data into two clusters to filter out structured outliers.
\item [iii] We will show through simulations how the algorithm works efficiently in scenarios of unstructured and structured outliers as well as in a mixture of both compared to other state of the art algorithms.
\item [iv] Further we will propose a technique that can adapt the threshold to the dataset and demonstrate the efficiency of the algorithm in real data applications like video activity detection and image separation where we will highlight how a parameter free algorithm can give a significant advantage compared to other methods.
\end{itemize}
Our hope is that this algorithm serves as a starting point for further progress in parameter free algorithms for robust PCA.
\section{Problem setup and notations}\label{s1}
	We are given $N$ data points, each from an $n$ dimensional space $\mathbb{R}^n$, denoted by $\mathbf{m}_i \in \mathbb{R}^n$, arranged in a data matrix $\mathbf{M} = [\mathbf{m}_1, \mathbf{m}_2...\mathbf{m}_N] \in \mathbb{R}^{n\text{x}N}$. In this paper we will be working with  $l_2$ normalized data points, namely $\mathbf{x}_i = \dfrac{\mathbf{m}_i}{\|\mathbf{m}_i\|_2}$. Here $ \|.\|_2$, denotes the $l_2$ norm, $\|.\|_{F}$ indicates Frobenius norm. $\Gamma(.)$ denotes the gamma function. Also $\mathbb{E}[Y]$ denotes the expectation, $var(Y)$ the variance and $\sigma_Y$ the standard deviation of the random variable $Y$. $\mathcal{N}(\mu, \sigma^2)$ denote a normal distribution with mean $\mu$ and variance $\sigma^2$. Let $F_{\mathcal{N}}(.)$ denote the standard normal cdf, $F_{\mathcal{N}}(y) =\dfrac{1}{\sqrt{2\pi}} \int\limits_{-\infty}^{y}e^{\frac{-x^2}{2}}dx$. $w.p$ indicates with probability and $\lfloor x\rfloor$ floors $x$. Also $O()$ denotes the Big O notation for complexity and $abs(x)$ denotes the absolute value of $x$. Let the normalized data matrix be denoted as $\mathbf{X} = [\mathbf{x}_1, \mathbf{x}_2...\mathbf{x}_N]$.  Let $\mathbb{S}^{n-1}$ denote the unit hypersphere in $\mathbb{R}^n$. Then $\mathbb{S}^{n-1} = \{\mathbf{x}\text{	} |\text{	}\mathbf{x}\in \mathbb{R}^n, \|\mathbf{x}\|_2 = 1\}$, i.e. the $l_2$ ball in $\mathbb{R}^n$ and all points in $\mathbf{X} \in \mathbb{S}^{n-1}$. We assume that out of the $N$ data points, $(1-\gamma)N$ of them lie in a low dimensional subspace $\mathcal{U}$ of dimension $r$, those we will refer to as inliers and the rest $\gamma N$ points, the outliers, lie in the high dimensional space. The parameters $\gamma$ which is the ratio of number of outliers to the total number of data points and $r$, dimension of the true subspace, are unknown. Let $\mathcal{I}$ denote the index set of inliers and $\mathcal{O}$ denote the index set of outliers, i.e. $\mathcal{I} = \{i\text{		}| \text{		}\mathbf{x}_i \text{ is an inlier}\}$ and $\mathcal{O} = \{i\text{		}| \text{		}\mathbf{x}_i \text{ is an outlier}\}$. Hence the matrix $\mathbf{X}$ can be segregated as $\mathbf{X} = [\mathbf{X}_\mathcal{I}, \mathbf{X}_\mathcal{O}] $, where $\mathbf{X}_\mathcal{I}$ are the set of inlier points and  $\mathbf{X}_\mathcal{O}$ are the set of outlier points. We will denote $N_\mathcal{I} =|\mathcal{I}| =(1-\gamma)N$ and $N_\mathcal{O} = |\mathcal{O}| =\gamma N$, where $|.|$ denotes the cardinality of a set. 
	
	The problem we will be focusing on is to remove the set of outliers from the matrix or to find  $\mathcal{O}$ without the knowledge of both the parameters $\gamma$ and $r$. We first list some essential definitions.
	\begin{definition}
		Let $\theta_{ij}$ denote the principal angle between two data points $\mathbf{x}_i$ and $\mathbf{x}_j$, i.e., 
		\begin{equation}\label{etheta}
		\theta_{ij} = cos^{-1}(\mathbf{x}_i^T\mathbf{x}_j) \hskip70 pt \theta_{ij} \in [0, \pi]
		\end{equation}
	\end{definition}
	\begin{definition}
		The acute angle between two points denoted by $\phi_{ij}$ is defined as: 
		\begin{align}\label{ephi}
		\phi_{ij} &= cos^{-1}(|\mathbf{x}_i^T\mathbf{x}_j|)\\
		&=\begin{cases}
		\theta_{ij} & \text{for } \theta_{ij}\leq \frac{\pi}{2}\\
		\pi-\theta_{ij} & \text{for } \theta_{ij} > \frac{\pi}{2}\\
		\end{cases}
		\end{align}
	\end{definition}
	Clearly $\phi_{ij}\in [0, \frac{\pi}{2}]$. Also $\phi_{ii}= \theta_{ii} =0$. 
	\begin{definition}
		The minimum angle subtended by a point denoted as $q_i$ is given by, 
		\begin{align}\label{eq}
		q_i = \underset{j=1, ..N, j\neq i}{\min\text{			}}\phi_{ij} \hskip30pt \forall i \in \{1, 2, ...N\}
		\end{align}
	\end{definition}
	\begin{definition}
	The number of acute angles formed by a point above a threshold $\zeta$ is defined as:
		\begin{equation}\label{ena}
		na_{i}^{\zeta} = |\{\phi_{ij}\text{	}|\text{	} \phi_{ij}>\zeta, j = 1,2...N\}|
		\end{equation}
	\end{definition}
	Now we will also define two properties that characterizes an algorithm for outlier removal.
	\begin{definition}[Outlier Identification Property, OIP($\alpha$) ]\label{doip}
		An algorithm for outlier removal is said to have Outlier Identification Property OIP($\alpha$), when the outlier index set estimate of the algorithm contains all the true outlier indices i.e. $\hat{\mathcal{O}} \supseteq \mathcal{O}$ with a probability at least $1-\alpha$. 
	\end{definition}
	\begin{definition}[Exact recovery Property, ERP($\alpha$)]\label{derp}
		An algorithm for outlier removal is said to have Exact Recovery Property, ERP($\alpha$) when it recovers all the inlier points or $\hat{\mathcal{I}} = \mathcal{I}$ with a probability at least $1-\alpha$.
	\end{definition}
	ERP($\alpha$) is a stronger condition than OIP($\alpha$). An algorithm which has ERP($\alpha$) will also have OIP($\alpha$) as in this case, $\hat{\mathcal{O}}=\mathcal{O}$ with a probability at least $1-\alpha$. 
	
		In this paper the following assumption is made on unstructured outliers (same as Assumption 1 in \cite{rahmani2016coherence}).
\begin{Assumption}\label{amain}
The subspace $\mathcal{U}$ is chosen uniformly at random from the set of all $r$ dimensional subspaces and the normalized inlier points are sampled uniformly at random from the intersection of $\mathcal{U}$ and $\mathbb{S}^{n-1}$. The normalized outlier points are sampled uniformly at random from $\mathbb{S}^{n-1}$.
\end{Assumption}
On structured outliers we make the following assumption:
\begin{Assumption}\label{a2}
The normalized structured outlier set is a subset of points sampled from points distributed uniformly on $\mathbb{S}^{n-1}$ such that the maximum principal angle in the outlier set  is bounded between $[\theta_{min}^{\mathcal{O}},\theta_{max}^{\mathcal{O}}]$ where $\theta_{max}^{\mathcal{O}} <\frac{\pi}{2}$. It can be defined as $\mathbf{X}_\mathcal{O} = \Big\{\mathbf{x}_1,\mathbf{x}_2,...\mathbf{x}_{N_{\mathcal{O}}}\text{	}|\text{	}\mathbf{x}_i \in \mathbb{S}^{n-1}\text{	} \forall i,\text{	} \theta_{ij} \in [\theta_{min}^{\mathcal{O}}, \theta_{max}^{\mathcal{O}}] \text{	}\forall i,j \in \mathcal{O}, i\neq j \Big\}$.
\end{Assumption}
For unstructured outliers, the outlier angles are distributed around $\frac{\pi}{2}$ and lie between $[0,\pi]$, but here a structure causes the angles to be lie in the interval $[\theta_{min}^{\mathcal{O}}, \theta_{max}^{\mathcal{O}}]$ with the mean angle being less than $\frac{\pi}{2}$. The outlier generating mechanism may be anything that can generate such an outlier set\footnote{For example this encompasses the structured outliers defined in \cite{rahmani2016coherence}. As $\mu$ in Assumption 2 of \cite{rahmani2016coherence} decreases, $\theta_{max}^{\mathcal{O}}$ also decreases}. As the outliers become more clustered  $\theta_{max}^{\mathcal{O}}$ reduces and $\theta_{min}^{\mathcal{O}}\to 0$. 
\section{Algorithm and features}\label{s2}
We will first discuss in brief the coherence pursuit (CoP) algorithm in \cite{rahmani2016coherence}, since our work can be regarded as a parameter free variant of CoP. The basic principle behind CoP algorithm \cite{rahmani2016coherence} is that the inlier points are more coherent amongst themselves and the outliers are less coherent. Hence for each point a metric is computed as the norm (either $l_1$ or $l_2$ norm) of a vector in $\mathbb{R}^{N-1}$ whose components are the coherence values that a point has with all the other data points. The expectation is that once these metrics are sorted in descending order, the inliers come first as the outlier metrics are supposed to be much less compared to the inlier metrics. Then the authors have proposed two schemes to remove the outliers and recover the true underlying subspace. The first scheme tries to remove the outliers and then perform PCA to get the true subspace. Here the outlier removal process assumes the knowledge of the maximum number of outliers in the system. The second scheme is an adaptive column sampling technique that generates an $r$ dimensional subspace from inlier points, with the assumption that the parameter $r$ is known. Another work of interest is the outlier detection technique used in \cite{heckel2015robust}, which uses a threshold on the maximum coherence on the outlier point. The two different thresholds proposed for noisy and noiseless cases are parameter free but these are loose and the algorithm is limited to unstructured outliers.

The proposed scheme works with angles, instead of coherence and the score that we compute is the minimum angle subtended by a point instead of the norm as is done in \cite{rahmani2016coherence}. We develop a high probability lower bound for outlier scores $q_i$ independent of the unknown parameters and use it to remove unstructured outliers. We further develop a second stage to remove structured outliers using the metric $na_{i}^{\zeta}$. 
\begin{figure}[t]
	\begin{subfigure}[b]{0.24\textwidth}
		\includegraphics[width=\linewidth, height=2cm]{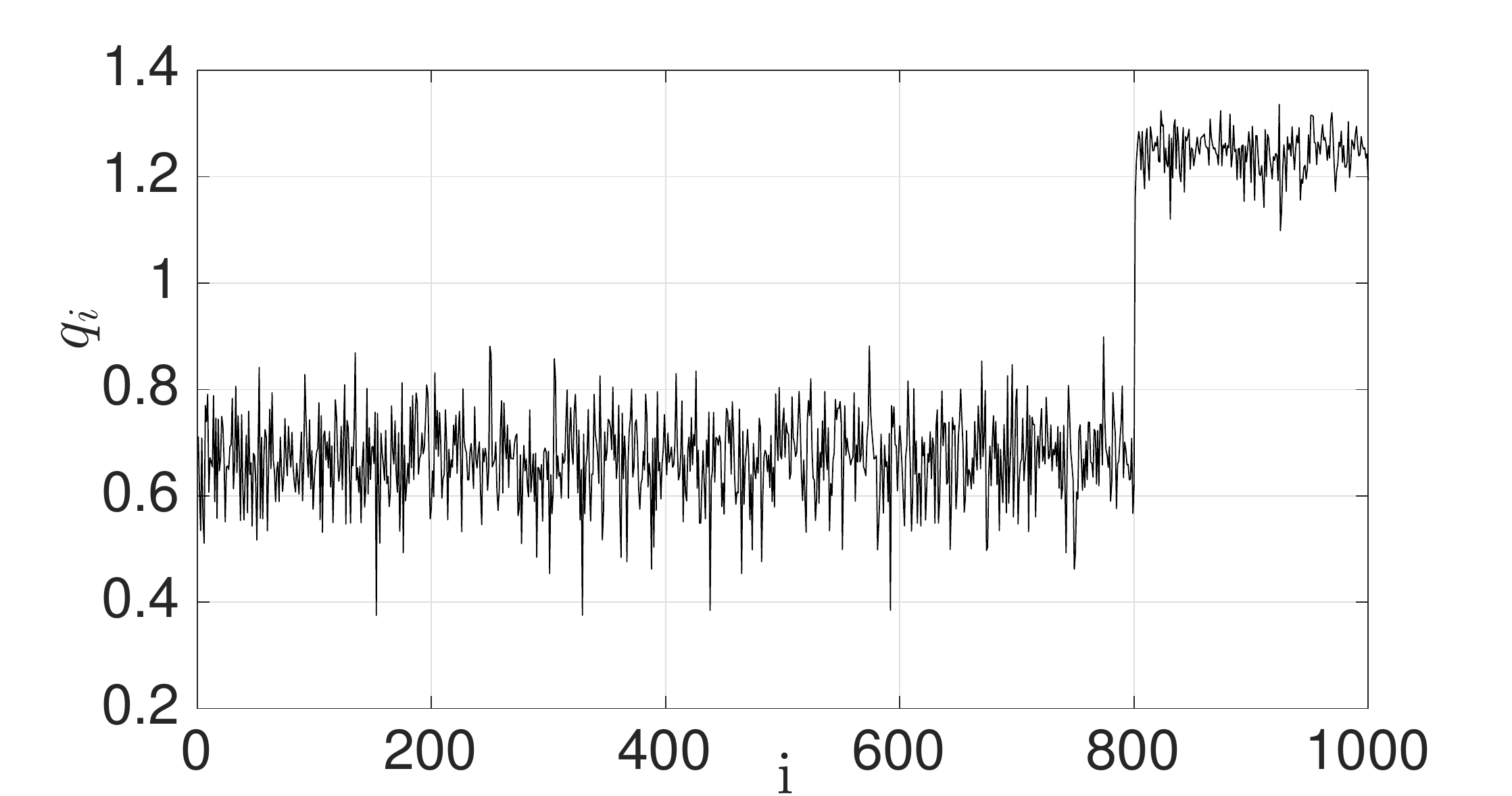}
		\caption{Inliers - upto 800}
		\label{}
	\end{subfigure}
	\hfill
	\begin{subfigure}[b]{0.24\textwidth}
		\includegraphics[width=\linewidth, height=2cm]{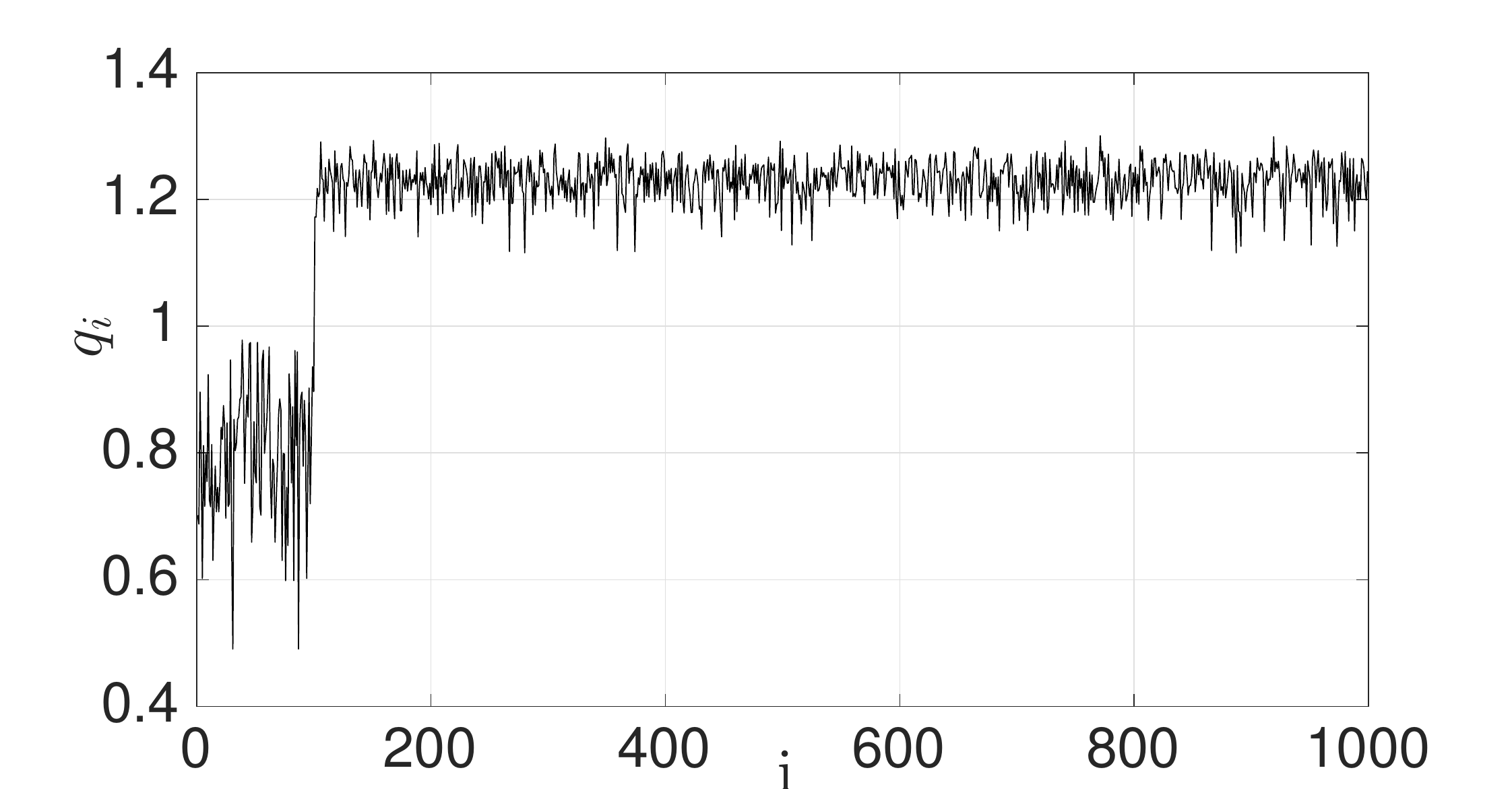}
		\caption{Inliers - upto 100}
		\label{}
	\end{subfigure}
	\caption{Behaviour of $q_i$ , $N=1000,\frac{r}{n}=0.1$}
	\label{fqsep}
\end{figure}
\begin{figure}[t]
	\begin{subfigure}[b]{0.24\textwidth}
		\includegraphics[width=\linewidth, height=2cm]{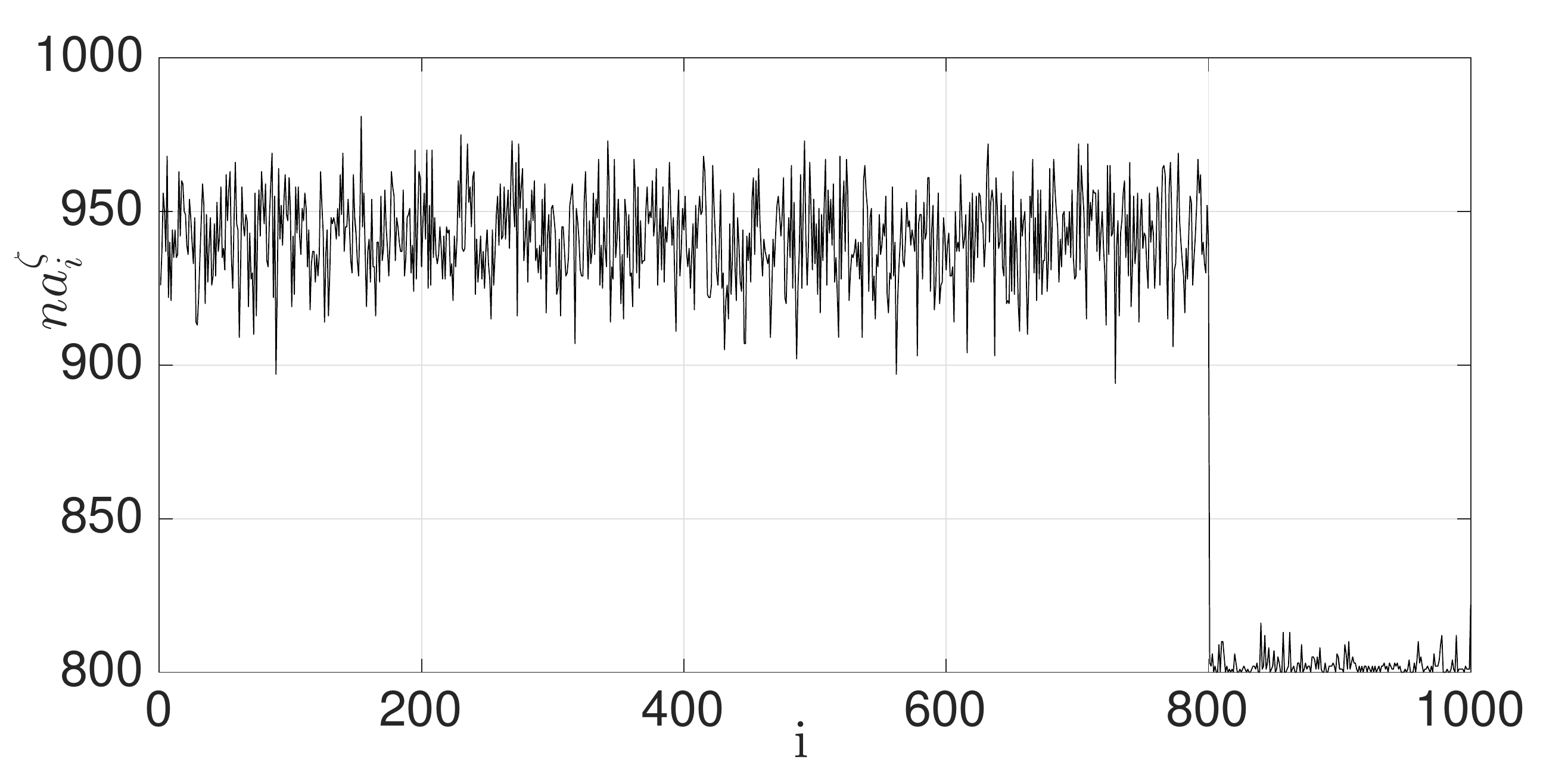}
		\caption{Inliers - upto 800}
		\label{}
	\end{subfigure}
	\hfill
	\begin{subfigure}[b]{0.24\textwidth}
		\includegraphics[width=\linewidth, height=2cm]{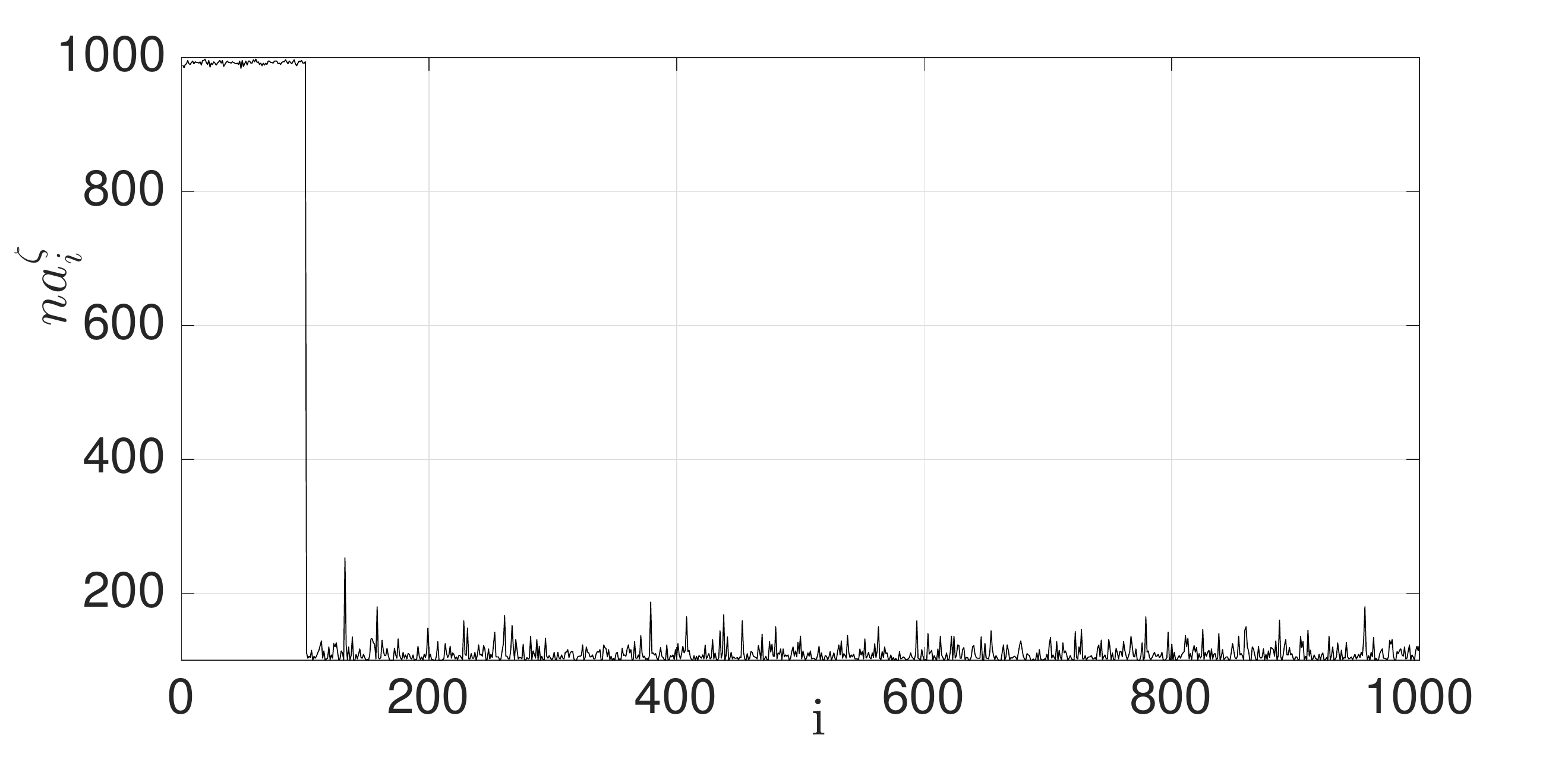}
		\caption{Inliers - upto 100}
		\label{}
	\end{subfigure}
	\caption{Behaviour of $na^\zeta_i$ , $N=1000,\frac{r}{n}=0.1$}
	\label{fnasep}
\end{figure}
 \subsection{Basic Principle and description}
 The folklore ``two high dimensional points are almost always orthogonal to each other'' has been rigorously proved in \cite{cai2013distributions} and this is what we exploit. First we will describe the principle behind outlier removal when the outliers and inliers follow Assumption 1. The proposed algorithm works on the principle that, outlier points subtend larger angles (very close to $\frac{\pi}{2}$) with rest of the points, but inlier points, since they lie in a smaller dimensional subspace, subtend smaller angles with other inlier points and hence would have a much smaller score $q_i$ as compared to an outlier. An example of the clear separation between $q_i$ values for inliers and outliers can be seen in Fig. \ref{fqsep} for randomly chosen outliers which shows that the property holds even at low inlier fraction. In the proposed method we will exploit this property to develop an algorithm that removes outliers and is also parameter free. In the algorithm we classify a point $\mathbf{m}_i$ as an outlier whenever $q_i$ is greater than a threshold $\zeta$ given by
 \begin{equation}\label{eth}
 \begin{aligned}
\zeta &= \dfrac{\pi}{2} - \dfrac{C_N}{\sqrt{n-2}},\\
 \end{aligned}
 \end{equation}
where $C_N = F_{\mathcal{N}}^{-1}\Big(1- \frac{1}{2N^2(N-1)}\Big)$ (See Theorem \ref{tthr} for more details on the derivation of $\zeta$). 
The proposed scheme which removes unstructured outliers is given in tabular form as Algorithm 1.
\begin{algorithm}[h]
	\caption{Removal of Outliers using Minimum Angle (ROMA)}
	\textbf{Input}:Data matrix $\mathbf{M}$\\
	\textbf{Procedure:}
	\begin{algorithmic}[1]
		\State Define $\mathbf{X}$, with columns $\mathbf{x}_i=\frac{\mathbf{m}_i}{\|\mathbf{m}_i\|_2}$
		\State Calculate $\phi_{ij}$ for $i, j=1, 2..N$ as in (\ref{ephi})
		\State Threshold, $\zeta \gets\dfrac{\pi}{2} - \dfrac{C_N}{\sqrt{n-2}}$
		\State Calculate $q_i$ for $i=1, 2..N$ as in (\ref{eq})
		\State $\hat{\mathcal{O}} \gets \{i\text{	}| q_i >\zeta\}$, $\hat{\mathcal{I}} \gets \{i\text{	}| q_i \leq \zeta\}$
	\end{algorithmic}
	\textbf{Output}: $\hat{\mathcal{I}}, \hat{\mathcal{O}} $ 
\end{algorithm}
The key steps are, 
\begin{itemize}
	\item [i] First the input data matrix is column normalized and the acute angles subtended by each point with other points as in (\ref{ephi}) are calculated for all data points. 
	\item[ii] Then the score for each point $q_i$ is computed by taking the minimum of the angles subtended by that point as in (\ref{eq}). 
	\item [iii] All the points with its $q_i$ value greater than $\zeta$ are classified as outliers and the rest as inliers.
\end{itemize}
The algorithm which we will call Removal of Outliers using Minimum angle (ROMA), is a parameter free algorithm, which removes all the randomly distributed outliers with a probability of at least $1-\frac{1}{N}$ and requires as input only the data matrix. Unlike the work in \cite{heckel2015robust}, ROMA is based on the distribution of angles between high dimensional points as opposed to correlations. The threshold in \cite{heckel2015robust} is based on applying Markov inequality which being loose makes the threshold loose, while $\zeta$ is based on the distribution of $\phi_{ij}$ and this enables the proposed algorithm to recover more inliers since the threshold is tight. However like \cite{heckel2015robust}, it is ineffective when the outliers are structured as in Assumption 2. Hence along with ROMA, we need a second stage to identify structured outliers as well. The key steps of the proposed algorithm are given below and is summarized in Algorithm 2.
\begin{algorithm}[h]
	\caption{ROMA with number of angles above threshold - ROMA\_N}
	\textbf{Input}:Data matrix $\mathbf{M}$\\
	\textbf{Stage1:} Execute ROMA to get $\hat{\mathcal{I}}$\\
	\textbf{Procedure for Stage 2:}
	\begin{algorithmic}[1]
		\State Calculate $na^\zeta_i \forall i \in \hat{\mathcal{I}}$ as in equation (\ref{ena}) using $\zeta$ from (\ref{eth})
		\State $i^* \gets \underset{i,j \in \hat{\mathcal{I}}, i \neq j}{\text{argmin}} \text{	}\phi_{ij}$
		\State $o^* \gets \underset{j \in \hat{\mathcal{I}}}{\text{argmax}} \text{	}\phi_{i^*j}$
		\State $\hat{\mathcal{O}}_{op} \gets \{i \in \hat{\mathcal{I}}\text{	}| abs(na_i^{\zeta}-na_{i*}^{\zeta})> abs(na_i^{\zeta}-na_{o^*}^{\zeta})\}$
		\State $\hat{\mathcal{I}}_{op}  \gets \{i \in \hat{\mathcal{I}} \text{	}| abs(na_i^{\zeta}-na_{i*}^{\zeta})\leq abs(na_i^{\zeta}-na_{o^*}^{\zeta})\}$
	\end{algorithmic}
	\textbf{Output}: $\hat{\mathcal{I}}_{op}, \hat{\mathcal{O}}_{op} $ 
\end{algorithm}
\begin{itemize}
	\item[i] After applying ROMA to remove unstructured outliers, find $na_i^{\zeta}$ value for each remaining data point. 
	\item[ii] Find two cluster heads - $i^*$ indexing one of the two points that forms the smallest angle amongst all angles, $o^*$ indexing the point which makes the largest angle with $i^*$. 
	\item[iii] Classify the points into two clusters $\hat{\mathcal{I}}$ and $\hat{\mathcal{O}}$ according to the $na_i^{\zeta}$ value -  classify to $\hat{\mathcal{I}}$ if $na_i^{\zeta}$ is closer to $na_{i^*}^{\zeta}$ and to $\hat{\mathcal{O}}$ if it is closer to $na_{o^*}^{\zeta}$.
\end{itemize}
The algorithm, which we will call ROMA\_N, is based on the principle that even when you have structured outliers, the angle between an inlier and an outlier is statistically same as that of angles between two points chosen uniformly at random from $\mathbb{S}^{n-1}$ and hence with very high probability is above $\zeta$. Thus for a structured outlier, the number of angles above $\zeta$ would be above $N_{\mathcal{I}}$ with high probability. An example of outlier and inlier $na_i^{\zeta}$ values is shown in Fig. \ref{fnasep} to highlight this characteristic. When the outlier structure is such that the maximum angle formed in the structure is less than $\zeta$, then the $na_i^{\zeta}$ value of every structured outlier is exactly $N_{\mathcal{I}}$. Even if that is not the case, all the outliers will have similar score which will be close to $N_\mathcal{I}$. If one were aware of $N_{\mathcal{I}}$, one could use that as a threshold to classify the points based on this metric. But since our algorithm is parameter free and hence unaware of $N_{\mathcal{I}}$, we need the cluster heads as chosen in step ii. The inlier head $i^*$ (one that subtends the minimum angle) need not be the real inlier head - it may happen that the clusters are reversed in the case when outliers are clustered closer. The theoretical requirements for this algorithm to work successfully for structured outliers is discussed in section \ref{sna}.
\subsection{Feature - Parameter free}
The main feature of the algorithm is that it does not have any dependencies on the unknown parameters. As seen clearly, the threshold we have proposed only requires $N$  and $n$ for its computation and is also independent of noise statistics. The technique proposed for removal of structured outliers is also parameter free. Once all the outliers have been identified and removed, the clean points can be used to obtain a low rank representation using classical PCA by SVD. For the noiseless case, PCA also does not require the knowledge of any parameter. In the presence of additive Gaussian noise $\textbf{w}_i$'s in the data, i.e. when $\mathbf{m}_i = \mathbf{m}^0_i+\textbf{w}_i$, $\textbf{w}_i \sim \mathcal{N}(0, \sigma_{w}^2\textbf{I}_n)$, there are several methods for selecting the number of principal components after SVD like BIC \cite{rissanen1978modeling}, geometric AIC \cite{kanatani1998geometric}, and other recent methods proposed in \cite{choi2017selecting}, \cite{donoho2013optimal}. 
\subsection{Feature - Simplicity}
ROMA is a simple to implement algorithm and the main complexity lies in computing all the angles. This requires computation of $N(N-1)$ angles and each involves an inner product of an $n$ dimensional vector and hence the complexity is $O(N^2n)$. The algorithm does not involve solving a complex optimization problem and is not iterative which is a significant advantage. The second step to implement robust PCA, would be an SVD on the inlier points recovered by the algorithm, which if implemented without truncation, has a time complexity of $O(min(Nn^2,N^2n))$ \cite{holmes2007fast}. Hence in any case overall complexity for the process is $O(N^2n)$. Section \ref{snumsim} contains running time comparisons with other algorithms.
\section{Theoretical Analysis of the algorithm}\label{s3}
In this section we address the following points:
\begin{itemize}
	\item[i] Derive the lower bound $\zeta$ on $q_i$ values for outliers under Assumption 1 which ensures outlier detection for ROMA.
	\item [ii] Derive the theoretical conditions for ROMA to follow the properties of OIP($\alpha$) and ERP($\alpha$) given in Definitions \ref{doip} and \ref{derp} under Assumption 1.
	\item[iii] An analysis of the how Gaussian noise affects the properties of the algorithm.
	\item[iv] Analyze the second stage of the algorithm by deriving the properties of the metric $na_i^{\zeta}$ under Assumption 2 and also deriving the theoretical requirements for algorithm performance.
\end{itemize}
As a starting point, we will state two lemmas that describe the distribution of the principal angles $\theta_{ij}$'s made by the points. This involves a slight modification of Lemma 12 in \cite{cai2013distributions}, to distinguish the angles formed by an inlier and outlier.
\begin{lemma}\label{lthetao}
	$\theta_{ij}$'s are identically distributed with an expected value of $\dfrac{\pi}{2}$ and it's pdf is given by: 
	\begin{equation}\label{epdf}
				h(\theta) = \dfrac{1}{\sqrt{\pi}}\dfrac{\Gamma(\frac{n}{2})}{\Gamma(\frac{n-1}{2})} (sin\theta)^{n-2} \hskip20pt \theta \in [0, \pi]
	\end{equation}
	in both the following cases:
	\begin{itemize}
		\item[a)] Inliers and outliers follow Assumption 1 and either $i$ or $j \in \mathcal{O}$
		\item[b)]Outliers follow Assumption 2 and $i \in \mathcal{O}$ and $j \in \mathcal{I}$ or vice versa. 
	\end{itemize}
Also $h(\theta)$ is well approximated by a Gaussian pdf with mean $\frac{\pi}{2}$ and variance $\frac{1}{n-2}$.
\end{lemma}
\begin{proof}
	Please refer to appendix \ref{app1}
\end{proof}
\begin{lemma}\label{lthetai}
	Under assumption 1, when $i, j \in \mathcal{I}$, $\theta_{ij}$'s are identically distributed with an expected value of $\dfrac{\pi}{2}$ and it's pdf is given by
	$h(\theta) = \dfrac{1}{\sqrt{\pi}}\dfrac{\Gamma(\frac{r}{2})}{\Gamma(\frac{r-1}{2})} (sin\theta)^{r-2} \text{ }, \theta \in [0, \pi]$.
	. Also $h(\theta)$ is well approximated by a Gaussian pdf with mean $\frac{\pi}{2}$ and variance $\frac{1}{r-2}$ whenever $r \geq 5$.
\end{lemma}
\begin{proof}
Please refer to appendix \ref{app1}
\end{proof}
For the algorithm we use the acute angles $\phi_{ij}$'s instead of $\theta_{ij}$'s. The properties of $\phi_{ij}$ has been characterized in appendix \ref{app3}. The following is an important result on $\phi_{ij}$:
\begin{lemma}\label{lphilower}
	$\phi_{ij}$ has the the following property:
	\begin{align}\label{ephiprop}
	\phi_{ij} > \dfrac{\pi}{2} - \dfrac{c}{\sqrt{n-2}} \hskip40pt w.p\text{	} 2F_{\mathcal{N}}(c)-1
	\end{align}
	 Under the following conditions on $i$ and $j$
	\begin{itemize}
		\item[a)] Inliers and outliers follow Assumption 1 and either $i$ or $j \in \mathcal{O}$
		\item [b)] Outliers follow Assumption 2 and $i \in \mathcal{O}$ and $j \in \mathcal{I}$ or vice versa.
	\end{itemize}
\end{lemma}
\begin{proof}
	Please refer to Appendix \ref{app3}.
\end{proof}
Under assumption \ref{amain}, from Lemmas \ref{lthetao} and \ref{lthetai}, for an outlier point, the principal angle it makes with any other point, be it an inlier or outlier is typically concentrated around $\frac{\pi}{2}$ especially at large $n$. On the other hand, because the dimension $r$ of the subspace $\mathcal{U}$ is much smaller than $n$, the angle that an inlier makes with another inlier is more spread around $\frac{\pi}{2}$. Classification of a point as an inlier or outlier using minimum principle angles $\theta_{ij}$'s would require multiple classification regions which can be avoided by using the acute angle $\phi_{ij}$. Here, the minimum acute angle that an outlier makes becomes very close to $\frac{\pi}{2}$ and a point may be classified as an outlier when $\underset{j=1, 2...N, j\neq i}{\min}\phi_{ij}\geq \zeta$, where $\zeta$ is some threshold. Hence the problem of outlier identification reduces to finding one appropriate threshold to be applied on $q_i$ defined in (\ref{eq}), which when used can classify all outlier points correctly with high probability. Further for the algorithm to be parameter free, we derive $\zeta$ such that it only depends on the number of data points $N$ and the ambient dimension $n$, which are of course always known. The next subsection gives the derivation of  $\zeta$. 
\subsection{Derivation of $\zeta$}\label{ssmain}
The following theorem gives the lower bound $\zeta$ on $q_i$, $i \in \mathcal{O}$ for unstructured outliers:
\begin{theorem}\label{tthr}
	Under Assumption \ref{amain}, ROMA with the classification rule that $\mathbf{x}_i$ is classified an outlier when $q_i > \zeta$, identifies all the outliers with probability at least $1-\frac{1}{N}$, when 
	\begin{equation*}
	\zeta  = \dfrac{\pi}{2} - \dfrac{C_N}{\sqrt{n-2}},
	\end{equation*}
	where $C_N = F_{\mathcal{N}}^{-1}\Big(1- \frac{1}{2N^2(N-1)}\Big)$.
\end{theorem}
\begin{proof}
 ROMA's classification rule is as follows - declare $\mathbf{x}_i$ to be an outlier if 
	\begin{equation*}
	\underset{j \in \{1,2...N\},j \neq i}{min} \phi_{ij} > \zeta
	\end{equation*}
Our aim is to derive an appropriate threshold, such that ROMA classifies all outliers correctly with probability at least $1-\frac{1}{N}$.  Hence we look at the probability of failure - failure being an outlier misclassified as an inlier. Suppose we are classifying a point $i$ and $i \in \mathcal{O}$, misclassification occurs when $q_i \leq \zeta$ whose  probability is as follows
\begin{align*}
\mathbb{P}(q_i \leq \zeta)&=\mathbb{P}\Big(\underset{j \in \{1,2...N\},j \neq i}{min} \phi_{ij}\leq \zeta\Big) \\
& =\mathbb{P}\Big(\bigcup_{j \in \{1,2...N\},j \neq i} \phi_{ij}\leq \zeta\Big) \\
&\leq \sum\limits_{j \in \{1,2...N\},j \neq i}\mathbb{P}(\phi_{ij} \leq \zeta)
\end{align*}
The last step is a union bound, which is fairly tight. \footnote{Using, $\sum\limits_{j}\mathbb{P}(\phi_{ij} \leq \zeta) -\sum\limits_{j}\sum\limits_{k>j}\mathbb{P}(\phi_{ij} \leq \zeta,\phi_{ik} \leq \zeta)\leq \mathbb{P}(q_i \leq \zeta) \leq \sum\limits_{j}\mathbb{P}(\phi_{ij} \leq \zeta) - \underset{k}{max}\sum\limits_{j\neq k}\mathbb{P}(\phi_{ij} \leq \zeta,\phi_{ik} \leq \zeta)$ and pairwise independence of $\phi_{ij}$'s,one can show that the terms subtracted from the sum in the lower and upper bounds are of O($\frac{1}{N^4}$) and O($\frac{1}{N^5}$) implying that the union bound is tight.} We know that under Assumption \ref{amain} using Lemma \ref{lphilower}, for any $j \in {1,2,...N}, j \neq i$
\begin{equation*}
\phi_{ij} > \dfrac{\pi}{2} - \dfrac{c}{\sqrt{n-2}} \hskip40pt w.p\text{	} 2F_{\mathcal{N}}(c)-1
\end{equation*}
Hence $\mathbb{P}(\phi_{ij} \leq \zeta) = 2(1-F_{\mathcal{N}}(c))$, when $\zeta = \dfrac{\pi}{2} - \dfrac{c}{\sqrt{n-2}}$. Then the task is to derive an appropriate value for $c$ to plug into this expression for $\zeta$. The probability of a failure, i.e. $i \in \mathcal{O},  i \in \hat{\mathcal{I}}$ can be bounded as
\begin{equation}\label{efailbound}
\begin{aligned}
\mathbb{P}(i \in \mathcal{O}, i \in \hat{\mathcal{I}}) &\leq 2(N-1)(1-F_{\mathcal{N}}(c))\\
\end{aligned}
\end{equation}
(\ref{efailbound}) gives us the probability bound on one classification failing to identify the outlier correctly. We have $N$ such  classifications and the requirement is for the algorithm to have $N$ successes, i.e. the correct identification of all outliers, whose probability
\begin{align*}
\mathbb{P}( N\text{ successes}) &= 1-\mathbb{P}(\text{at least 1 fail})\\
&=1 - \mathbb{P}(\bigcup_{i \in 1,2...N}\{i \in \mathcal{O}, i \in \hat{\mathcal{I}}\} )\\
&\geq 1 - N\mathbb{P}(i \in \mathcal{O}, i \in \hat{\mathcal{I}}) \hskip15pt\text{(Union bound)}\\
&\geq 1- N\times 2(N-1)(1-F_{\mathcal{N}}(c))\hskip5pt\text{(from(\ref{efailbound}))}
\end{align*}
For ROMA to identify all outliers correctly with a probability of at least $1-\frac{1}{N}$, we require
\begin{align*}
1 - 2N(N-1)(1-F_{\mathcal{N}}(c)) &= 1-\frac{1}{N}\\
\Rightarrow (1-F_{\mathcal{N}}(c)) &= \dfrac{1}{2N^2(N-1)}\\
\Rightarrow F_{\mathcal{N}}(c) &= 1- \dfrac{1}{2N^2(N-1)}\\
\Rightarrow  c =C_N &=  F_{\mathcal{N}}^{-1}\Big(1- \frac{1}{2N^2(N-1)}\Big)
\end{align*}
Hence we arrive at the expression for $\zeta$ in the theorem.
\end{proof}
 Summarizing, we have derived a threshold $\zeta$, which does not depend on the unknown parameters $\gamma$ and $r$, such that for unstructured outliers, ROMA identifies all outliers correctly with probability at least $1-\frac{1}{N}$, i.e. ROMA output $\hat{\mathcal{O}} \supseteq \mathcal{O}\text{	}w.p \geq 1-\frac{1}{N}$. The next point of interest would be to see when the identification is exact, i.e. $\mathcal{O} = \hat{\mathcal{O}}$ and $\hat{\mathcal{I}} = \mathcal{I}$. 
\subsection{\textbf{ROMA theoretical guarantees}}\label{sguaran}
In this section, we will be looking at the properties of algorithm for exactly recovering the true inlier set under Assumption \ref{amain}. We will look at ROMA and the properties given in Definitions \ref{doip} and \ref{derp}. 
\begin{remark}\label{roip}
	 Under Assumption \ref{amain}, when $\zeta$ is as given by (\ref{eth}), ROMA has OIP($\frac{1}{N}$) regardless of the number of outliers in the system or the dimension of the underlying subspace. 
\end{remark}
To have the property of ERP$(\alpha)$, ROMA would have to recover inliers with a probability of at least $1-\alpha$.  We will build towards this with the set of lemmas and theorems below. These results give theoretical bounds, but as seen through simulations these are not necessary conditions for recovery of all inliers. ROMA can recover a large number of inliers and have good subspace recovery characteristics in worse conditions than these (see Section \ref{snumsim}). First lets look at the probability of the inlier set estimate being non empty through the next lemma.
\begin{lemma}\label{lnonempty}
Under Assumption 1 and conditions stated in Lemma \ref{l15}, the inlier set estimate by ROMA, $\hat{\mathcal{I}}$ is non empty $w.p \geq1- \dfrac{(N_{\mathcal{I}}-1)(N_{\mathcal{I}}-2)p_{\mathcal{I}}^2 - z(2p_{\mathcal{I}}(N_{\mathcal{I}}-1)-(z+1))}{(N_{\mathcal{I}}-z)(N_{\mathcal{I}}-1-z)}$, where $z = \lfloor(N_{\mathcal{I}}-2)p_{\mathcal{I}}\rfloor$, $p_{\mathcal{I}} =  \Big(2F_{\mathcal{N}}\Big(C_N\sqrt{\frac{r-2}{n-2}}\Big)-1\Big)$.
\end{lemma}
\begin{proof}
We will look at the probability of $\hat{\mathcal{I}}$ being empty under Assumption \ref{amain}, $\hat{\mathcal{I}}$ being non empty is its compliment. 
\begin{align*}
\mathbb{P}(\hat{\mathcal{I}} = \Phi) &= \mathbb{P}(\bigcap_{i \in \mathcal{I}}q_i>\zeta) \leq \mathbb{P}(q_i>\zeta) 
\end{align*}
If conditions in Lemma 15 are satisfied, then by applying the upper bound on $\mathbb{P}(\bigcap_{i \in \mathcal{I}}q_i>\zeta)$ derived in Lemma 15 in Appendix \ref{app2}, we get the result. Otherwise a looser upper bound in Lemma 14 can be used to get $\mathbb{P}(\hat{\mathcal{I}} = \Phi) \leq p_{\mathcal{I}}^2$ .
\end{proof}
The conditions in Lemma \ref{l15} are mild and satisfied usually. For example if $n=100, r=10, N=1000, N_{\mathcal{I}}=200$, the probability of the set being non empty is $>0.946$ and for $n=300, r=6, N=400, N_{\mathcal{I}}=100$ it is $>0.99$ from this lemma. Simulation results in section \ref{snumsim} show that the probability of the set being non empty is even higher in practice. For recovering the inliers, when the expected $q_i$ value for inliers is less than the threshold $\zeta$, a significant fraction are classified as inliers and the subspace can be recovered efficiently. The smaller the rank of the true subspace, the better the results will be in terms on inlier recovery.
\begin{lemma}\label{lrough}
Under Assumption 1, ROMA recovers a sizable amount of inliers when the following condition is satisfied: 
	\begin{equation}
	r \leq \dfrac{2(n-2)}{\pi C_N^2}+2 
	\end{equation}
	\begin{proof}
		Please refer to appendix \ref{app2}
	\end{proof}
\end{lemma}	
Evaluating this at $n=300,N=400$ gives $r  \leq 7.9$, but even at much higher $r$ in noisy scenarios, ROMA is seen to recover all inliers as demonstrated in Section \ref{scomp}. In the next theorem we will derive the theoretical condition when ROMA is guaranteed not to have ERP($\alpha$), which means it cannot recover all inliers with a high probability of $1-\alpha$. 
\begin{theorem}\label{trevcond}
	Under Assumption 1, the algorithm ROMA is guaranteed not to have ERP($\alpha$) when $\alpha \leq (N_{\mathcal{I}}-2)p_{\mathcal{I}}^2-(N_\mathcal{I}-3)p_{\mathcal{I}}$, where $p_{\mathcal{I}} =  \Big(2F_{\mathcal{N}}\Big(C_N\sqrt{\frac{r-2}{n-2}}\Big)-1\Big)$. In other words, the algorithm cannot recover all inliers with a probability greater than or equal to $1-(N_{\mathcal{I}}-2)p_{\mathcal{I}}^2-(N_\mathcal{I}-3)p_{\mathcal{I}}$.
\end{theorem}
\begin{proof}
	Please refer to appendix \ref{app2}
\end{proof}
This theorem gives us conditions when the algorithm is guaranteed to not have ERP($\alpha$). In this case the outlier index estimate, $\hat{\mathcal{O}} \supset \mathcal{O}$ and $\hat{\mathcal{O}}$ has inliers as well, i.e. $\hat{\mathcal{O}} \cap \mathcal{I} \neq \Phi$. For example, lets take a case of $n=100$, $r=20$, $N=1000$ and $N_{\mathcal{I}} = 100$, plugging in these values to the condition in the theorem gives us that if $\alpha \leq 0.1327$, the algorithm is guaranteed not to have $ERP(\alpha)$, which means under these conditions the algorithm cannot guarantee full inlier recovery with a probability greater than $1-.1327 = .8673$. Now suppose the rank was reduced to 10, the bound on $\alpha$ goes to a negative value, which means $\alpha$ is free to take any value between $0$ and $1$ and full inlier recovery may be possible with a high probability but theorem cannot give a definitive value for this probability. Another instance would be to see this value at $r=40$, which evaluates to the condition that if $\alpha \leq 0.986$ the algorithm is guaranteed not to have $ERP(\alpha)$ or the algorithm cannot guarantee full inlier recovery with a probability greater than $1-.986 = .014$. In all these cases the algorithm still has OIP($\frac{1}{N}$). All the prior art also derives similar conditions on performance guarantee, for instance coherence pursuit\cite{rahmani2016coherence} guarantees subspace recovery when the inlier density $\frac{N_\mathcal{I}}{r}$ is sufficiently larger that the outlier density. i.e. $\frac{N-N_\mathcal{I}}{n}$ while outlier pursuit \cite{xu2010robust} gives conditions on $\gamma$ and $r$ for successful subspace recovery. The theorem does not state the conditions in which the algorithm is guaranteed to have ERP($\alpha$), it merely gives us extreme cases where it is not. The following lemma should give us an idea about the ERP($\alpha$) of ROMA.
\begin{lemma}\label{lerpb}
Under Assumption 1, $\mathbb{P}(\hat{\mathcal{I}}=\mathcal{I}) \geq 1-N_\mathcal{I}\mathbb{P}(q_{i, i\in {\mathcal{I}}}>\zeta)$. Hence ROMA has the property of ERP($N_\mathcal{I}\mathbb{P}(q_{i, i\in {\mathcal{I}}}>\zeta)$).
\end{lemma}
\begin{proof}
	Please refer appendix \ref{app2}  
\end{proof}
For proceeding further it is required to characterize the complementary cdf (ccdf) of $\underset{j\in \mathcal{I}, j\neq i}{\min} \phi_{ij} $ for $i \in \mathcal{I}$. Since the principal angles $\theta_{ij}$'s are only pairwise independent and not mutually independent as noted in \cite{cai2013distributions} and \cite{cai2012phase}, finding this ccdf analytically is mathematically very difficult. However one can find $\mathbb{P}({\underset{j\in \mathcal{I}, j\neq i}{\min} \phi_{ij} > \zeta})$ empirically through simulations to obtain more insight about inlier recovery properties of ROMA. An empirical calculation of ERP($\alpha$) can be seen in Table \ref{table5}.
\subsection{Impact of noise on the algorithm}
\begin{remark}
	The algorithm ROMA, retains OIP($\frac{1}{N}$) even in presence of Gaussian noise irrespective of noise variance. 
\end{remark}   
When Gaussian noise is added to an outlier data point, and the noisy outlier is normalized, it is just like selecting it at random from an $n$ dimensional hypersphere. Hence all the theory and bounds on the outlier score will not change. Noise will however affect the inlier identification of the algorithm as noise is bound to increase the statistic $q_i$ for $i \in \mathcal{I}$. The expected value of increase of the inlier angle can be seen through the following lemma.
\begin{lemma}\label{lnoise}
	If the data points are corrupted by additive Gaussian noise, i.e. $\mathbf{m}^{observed}_i = \mathbf{m}_i + \textbf{e}_i $,  $\textbf{e}_i \sim \mathcal{N}(\textbf{0},\sigma^2\textbf{I}_n)$, under assumption 1, the statistical properties of $\theta_{ij}$, $i$ or $ j \in \mathcal{O}$, are unaffected, however the inlier angles increase on an average and the average worst case change $\Delta\theta_{w.c}$ is bounded by $\Delta\theta_{w.c} \leq \cos^{-1}(1-\frac{1}{2\sqrt{snr}})$, where $snr = \dfrac{\|\mathbf{m}_i\|_2^2}{n\sigma^2}$.
\end{lemma}
\begin{proof}
	Please refer Appendix \ref{app2}
\end{proof}
This means that the inlier $q_i$ values would increase in worst case by $\Delta\theta_{w.c}$ and a number of them could fall above the threshold $\zeta$ depending on the noise variance and the rank of the inlier subspace and so would be classified as outliers. So in presence of noise, algorithm would recover less inliers than no noise case, as expected. For instance the conditions in Lemma \ref{lrough}, would change slightly as follows:
\begin{lemma}\label{lroughnoise}
Under Assumption 1 with added Gaussian noise in inliers, ROMA recovers a sizable amount of inliers when the following condition is satisfied:
	\begin{equation}
	r \leq \dfrac{2(n-2)}{\pi\Big(C_N+\sqrt{n-2}\cos^{-1}(1-\frac{1}{2\sqrt{snr}})\Big)^2}+2 
	\end{equation}
	\begin{proof}
		Please refer to appendix \ref{app2}
	\end{proof}
\end{lemma}	
\subsection{Theoretical analysis of stage 2}\label{sna}
This subsection has analysis of ROMA\_N under the structured outliers assumption, Assumption 2.  Since ROMA can filter out unstructured outliers, stage 2 input comprises of only structured outliers, the number of them denoted by $N_{\mathcal{O}}^s$. We will first state a theorem\footnote{Though this theorem holds for any outliers following either assumption 1 or 2, here we focus on the structured case.} which characterizes the behavior of the statistic $na_i^\zeta$. 
\begin{theorem}\label{tna}
	When $na_i^\zeta$ is defined by (\ref{ena}), $\forall i \in \mathcal{I}$, $na_i^{\zeta}\geq N_{\mathcal{O}}^s$ and $\forall i \in \mathcal{O}$, $na_i^{\zeta} \geq N_{\mathcal{I}}$ both $w.p \geq 1-\dfrac{N_{\mathcal{O}}^sN_{\mathcal{I}}}{N^2(N-1)}$.
\end{theorem}
\begin{proof}
	For any $i \in \mathcal{O}$, $j \in \mathcal{I}$, we know from Lemma \ref{lphilower}, that 
	\begin{equation*}
	\phi_{ij} > \dfrac{\pi}{2} - \dfrac{c}{\sqrt{n-2}} \hskip40pt w.p\text{	} 2F_{\mathcal{N}}(c)-1
	\end{equation*}
	Looking at the value for $na_i^{\zeta}$, when $i \in \mathcal{O}$, if $\phi_{ij} \forall j \in \mathcal{I}$ is greater than $\zeta$, then $na_i^{\zeta}\geq N_{\mathcal{I}}$. Hence for any $i \in \mathcal{O}$,
	\begin{align*}
	\mathbb{P}(na_i^{\zeta}\geq N_{\mathcal{I}}) &\geq \mathbb{P}(\bigcap_{j \in \mathcal{I}}\phi_{ij}>\zeta)\\
	&=1-\mathbb{P}(\bigcup_{j \in \mathcal{I}}\phi_{ij}\leq \zeta)\\
	&\geq 1-N_{\mathcal{I}}\mathbb{P}(\phi_{ij}\leq \zeta)
	\end{align*} 
	The third step is by the identical nature of distributions and union bound.  Using the value of $\zeta$, We know $\mathbb{P}(\phi_{ij}\leq \zeta)  = 2\Big(1-F_\mathcal{N}\Big(F_{\mathcal{N}}^{-1}\Big(1- \frac{1}{2N^2(N-1)}\Big)\Big)\Big)$, which when simplified, gives us the below result:
	\begin{align*}
	\mathbb{P}(na_i^{\zeta}\geq N_{\mathcal{I}}) &\geq 1- \dfrac{N_{\mathcal{I}}}{N^2(N-1)}\\
	\Rightarrow \mathbb{P}(na_i^{\zeta}< N_{\mathcal{I}}) &\leq\dfrac{N_{\mathcal{I}}}{N^2(N-1)}\hskip10pt \text{for any $i \in \mathcal{O}$}
	\end{align*}
	Hence the probability that $\forall i \in \mathcal{O}$, $\mathbb{P}(na_i^{\zeta}\geq N_{\mathcal{I}})$ is given by:
	\begin{align*}
	\mathbb{P}(na_i^{\zeta}\geq N_{\mathcal{I}}) &= 1 - \mathbb{P}(\bigcup_{i \in \mathcal{O}} na_i^{\zeta}< N_{\mathcal{I}}) \\
	&\geq 1-N_{\mathcal{O}}^s\mathbb{P}(na_i^{\zeta}< N_{\mathcal{I}}) \\
	&\geq 1-\dfrac{N_{\mathcal{O}}^sN_{\mathcal{I}}}{N^2(N-1)}
	\end{align*}
	The result for $i \in \mathcal{I}$ case is exactly the same with inlier replaced by outlier and we arrive the above result. Hence we get the following statement.
	\begin{equation}
	\begin{aligned}
	\forall i \in \mathcal{O}, \mathbb{P}(na_i^{\zeta}\geq N_{\mathcal{I}})&\geq 1-\dfrac{N_{\mathcal{O}}^sN_{\mathcal{I}}}{N^2(N-1)} \\
	\forall i \in \mathcal{I}, \mathbb{P}(na_i^{\zeta}\geq N_{\mathcal{O}}^s)&\geq 1-\dfrac{N_{\mathcal{O}}^sN_{\mathcal{I}}}{N^2(N-1)} 
	\end{aligned}
	\end{equation}
\end{proof}
This is the crux behind the working of stage 2. Stage 2 is dependent on distinction between $na_i^{\zeta}$ value of an inlier and an outlier. It is effective for identifying structured outliers, when there is distinction between $N_{\mathcal{I}}$ and $N_{\mathcal{O}}^s$. In the algorithm, we find two cluster heads such that these will not belong to the same cluster, and look for the closeness of $na_i^{\zeta}$ values of other points with these cluster heads. This is done by choosing the first reference  point as the one which forms the minimum angle in the system and the other being the point which has the maximum angle with the first reference point. The point which forms the minimum angle is assumed to be that of the inlier, which is based on the reasoning that ``inliers" as the name suggests are more closely bunched with each other. Whenever this is not the case, the clusters are reversed, i.e. $\hat{\mathcal{I}}_{op}$ becomes the outlier estimate and $\hat{\mathcal{O}}_{op}$ becomes the inlier estimate. Even in this case inliers and outliers are separated, however to identify the cluster corresponding to $\mathcal{I}$, one may use rank of each cluster as the criterion. The following results give us conditions when the algorithm, with high probability, can cluster the points efficiently into inlier and outlier set. 
\begin{lemma}\label{tsoui}
	Suppose outliers follow Assumption 2, such that $\theta_{max}^{\mathcal{O}} < \zeta$  and inliers follow Assumption 1, the following can be said about Algorithm 2. 
	 \begin{itemize}
	 	\item [a)] One of the output clusters of algorithm 2 contains only inliers $w.p \geq 1-\frac{N_{\mathcal{O}}^sN_{\mathcal{I}}}{N^2(N-1)}$.
	 	\item [b)] Further when $N_{\mathcal{I}} > N_{\mathcal{O}}^s$, if $N_{\mathcal{I}} - N_{\mathcal{O}}^s = \delta N$, the inlier cluster contains a sizable set of inliers when $\delta N>2(N_{\mathcal{I}}-1)(2F_{\mathcal{N}}(C_N\sqrt{\frac{r-2}{n-2}})-1)$.
	 	\item[c)] When $N_{\mathcal{I}} < N_{\mathcal{O}}^s$, Algorithm 2 clusters the points into inliers and outliers exactly $w.p \geq 1-\frac{2N_{\mathcal{O}}^sN_{\mathcal{I}}}{N^2(N-1)}$. 
	 \end{itemize} 
\end{lemma}
\begin{proof}
	Please refer Appendix \ref{app4}
\end{proof}
Now we will look at structured inliers, lets use the following assumption similar to the one used for structured outliers:
\begin{Assumption}\label{a3}
	The normalized structured inlier set is a subset of points sampled from points distributed uniformly on the intersection of $\mathcal{U}$ and $\mathbb{S}^{n-1}$ such that the maximum principal angle in the inlier set  is bounded between $[\theta_{min}^{\mathcal{I}},\theta_{max}^{\mathcal{I}}]$ where $\theta_{max}^{\mathcal{I}} <\frac{\pi}{2}$. It can be defined as $\mathbf{X}_\mathcal{O} = \Big\{\mathbf{x}_1,\mathbf{x}_2,...\mathbf{x}_{N_{\mathcal{O}}^s}\text{	}|\text{	}\mathbf{x}_i \in \mathcal{U}\cap\mathbb{S}^{n-1}\text{	} \forall i,\text{	} \theta_{ij} \in [\theta_{min}^{\mathcal{I}}, \theta_{max}^{\mathcal{I}}] \text{	}\forall i,j \in \mathcal{I}, i\neq j \Big\}$.
\end{Assumption}
The below result talks of the performance of the algorithm in this scenario.
\begin{lemma}\label{tsosi}
	Suppose outliers follow Assumption 2, such that $\theta_{max}^{\mathcal{O}} < \zeta$  and inliers follow Assumption 1 such that $\theta_{max}^{\mathcal{I}} < \zeta$, the following can be said about Algorithm 2 whenever $N_\mathcal{I} \neq N_{\mathcal{O}}^s$. 
	\begin{itemize}
		\item [a)] Algorithm 2 clusters the points into inliers and outliers exactly $w.p \geq 1-\frac{2N_{\mathcal{O}}^sN_{\mathcal{I}}}{N^2(N-1)}$,  
		\item [b)] Further the cluster outputs $\hat{\mathcal{I}}_{op} = \mathcal{I}$ and $\hat{\mathcal{O}}_{op} = \mathcal{O}$, when $\theta_{max}^{\mathcal{I}}  < \theta_{min}^{\mathcal{O}}$.
	\end{itemize} 
\end{lemma}
\begin{proof}
	Please refer Appendix \ref{app4}
\end{proof}
This lemma says that the algorithm can separate inliers and structured outliers with the designated inlier estimate being the true inlier set with a high probability when the inliers are more clustered than outliers. We can now state the following remark on the properties of Algorithm 2:
\begin{remark}\label{ralgo2}
	When outliers follow Assumption 2 and inliers follow Assumption 1 such that $\theta_{max}^{\mathcal{I}} < \theta_{min}^{\mathcal{O}} < \theta_{max}^{\mathcal{O}} \leq \zeta$, algorithm 2 has EIP\Big($\frac{2N_{\mathcal{O}}^sN_{\mathcal{I}}}{N^2(N-1)}$\Big).
\end{remark}
This is obvious from Lemma \ref{tsosi} and the definition of ERP. Also this means the algorithm has OIP\Big($\frac{2N_{\mathcal{O}}^sN_{\mathcal{I}}}{N^2(N-1)}$\Big). Through simulations, we note that the algorithm can separate the inliers and structured outliers even in cases when $\theta_{max}^{\mathcal{O}}> \zeta$. As stated earlier if we had known $N_{\mathcal{I}}$ a priori, which is not practical, we could have used it as a threshold on $na_{i}^{\zeta}$, and such an algorithm would have OIP\Big($\frac{2N_{\mathcal{O}}^sN_{\mathcal{I}}}{N^2(N-1)}$\Big) in any case.
\begin{table}[t]
	\caption{Empirical $\alpha$ values for ERP and OIP for $n=100,r=10,N=1000$ }
	\label{table5}
	\begin{tabularx}{\linewidth}{@{}l*{10}{C}c@{}}
		\hline
		&\multicolumn{3}{c|}{$SNR= 20 dB$}&\multicolumn{3}{c}{$SNR = 10 dB$}\\
		\multicolumn{1}{r}{$\gamma$}&\multicolumn{1}{c}{0.15}&\multicolumn{1}{c}{0.55}&\multicolumn{1}{c|}{0.95}&\multicolumn{1}{c}{0.15}&\multicolumn{1}{c}{0.55}&\multicolumn{1}{c}{0.95} \\
		\hline
		OIP(.)&$0$&$0.001$&\multicolumn{1}{c|}{$0$}&$0$&$0$&$0$\\
		ERP(.)&$0$&$0.001$&\multicolumn{1}{c|}{$.895$}&$.202$&$.298$&$0.99$\\
		\hline
	\end{tabularx}
\end{table}
\section{Numerical Simulations}\label{snumsim}
In this section, we present the simulation results of the proposed method. Here we demonstrate the properties of the proposed algorithm in terms of inlier identification and subspace recovery on synthetic data. We also compare our method with some of the existing algorithms for robust PCA, in terms of running time of the algorithm, percentage of inliers recovered and log recovery error ($LRE$) of the estimated subspace, which is defined as in \cite{rahmani2016coherence}, i.e.
\begin{equation}\label{elre}
LRE =  log_{10}(\dfrac{\|\textbf{U} - \hat{\textbf{U}}\hat{\textbf{U}}^T\textbf{U}\|_F}{\|\textbf{U}\|_F}), 
\end{equation}
where $\textbf{U}$ is basis of the true inlier subspace and $\hat{\textbf{U}}$ is the estimated basis from the algorithms. All our experiments on from subsections B and C assumes the data model described in this paper under Assumption \ref{amain}. Structured outliers are shown in subsection D. $SNR$ in $dB$ is used as a measure of noise level for noisy inliers, where the inliers points $\textbf{m}^{observed}_i = \textbf{m}_i +\textbf{e}_i, \textbf{e}_i \sim \mathcal{N}(\textbf{0},\sigma^2\textbf{I}_n)$ and $\sigma = \|\textbf{M}\|_{F}/(10^{SNR/20}\sqrt{nN})$. 
 \subsection{Validation of bounds}\label{svalid}
 Here we validate the bound in Theorem \ref{tthr} through Fig. \ref{fbvalid} where the minimum outlier $q_i$ value is plotted over a 1000 trials against a range of outlier fraction $\gamma$ from $0.05$ to $0.95$ and $r=10$ in two cases of $N=1000, n=100, SNR = 20 dB$ and $N=500, n=300, SNR= 5 dB$. Empirically the $\alpha$ value for OIP and ERP were computed for two cases of $SNR$ for $n=100,r=10,N=1000$ and the results are given in Table \ref{table5}.  As seen in the table, ROMA retains OIP($0.001$) in any SNR or $\gamma$, while it has ERP($0.001$) in high SNR and at low $\gamma$, this deteriorates at lower SNR and higher $\gamma$. Also in Fig. \ref{fbvalid2} we validate the bounds in Theorem \ref{tna}, on the $na_i^{\zeta}$ value for inliers and outliers with their respective bounds $N_{\mathcal{O}}^s$ and $N_{\mathcal{I}}$, where inliers were randomly chosen as in Assumption 1 and outliers are structured.
 \begin{figure}[h]
 	\begin{subfigure}[b]{0.23\textwidth}
 		\includegraphics[width=\linewidth, height=6cm]{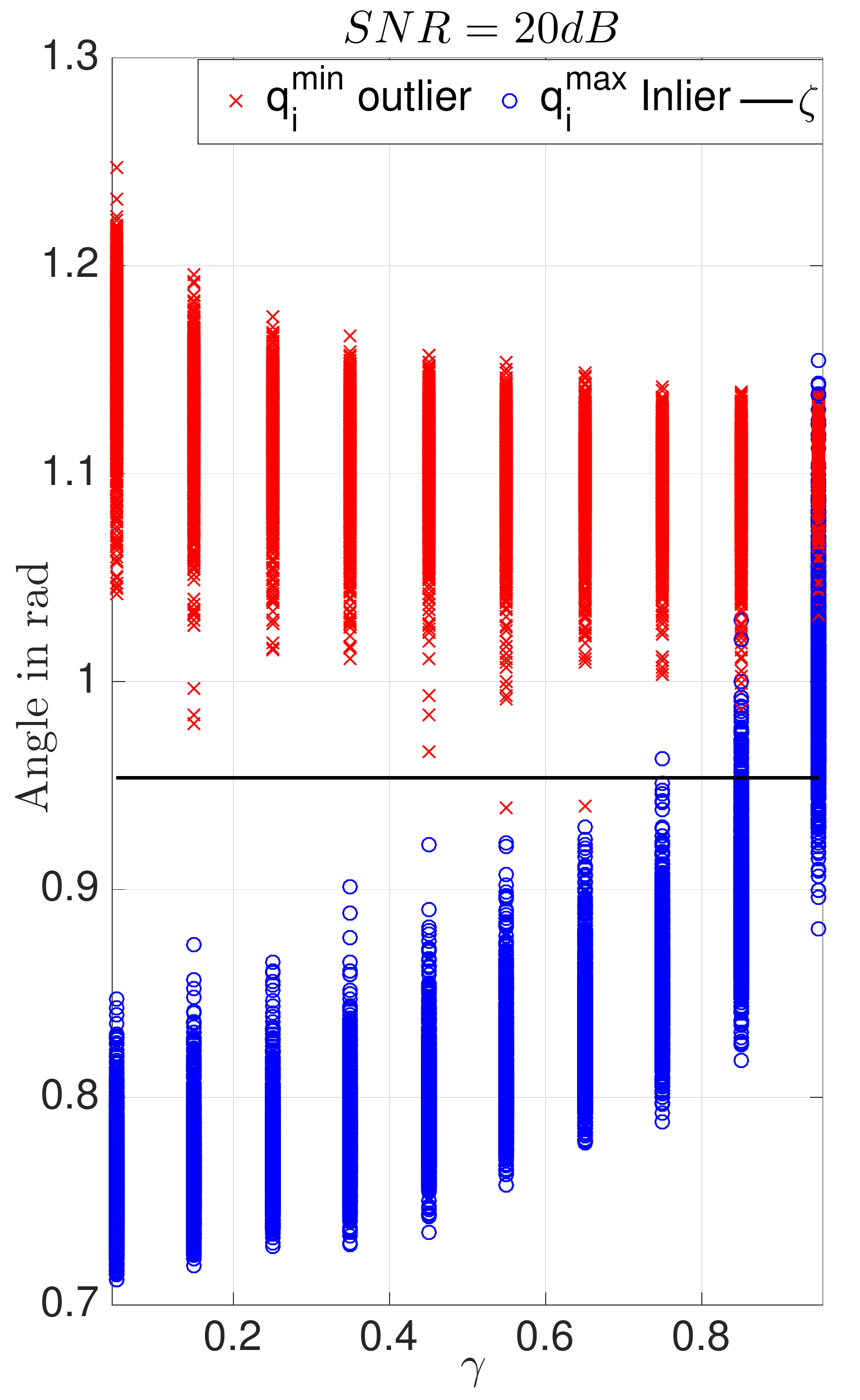}
 		\caption{n=100, r=10, N=1000}
 		\label{fval1}
 	\end{subfigure}
 	\begin{subfigure}[b]{0.23\textwidth}
 		\includegraphics[width=\linewidth, height=6cm]{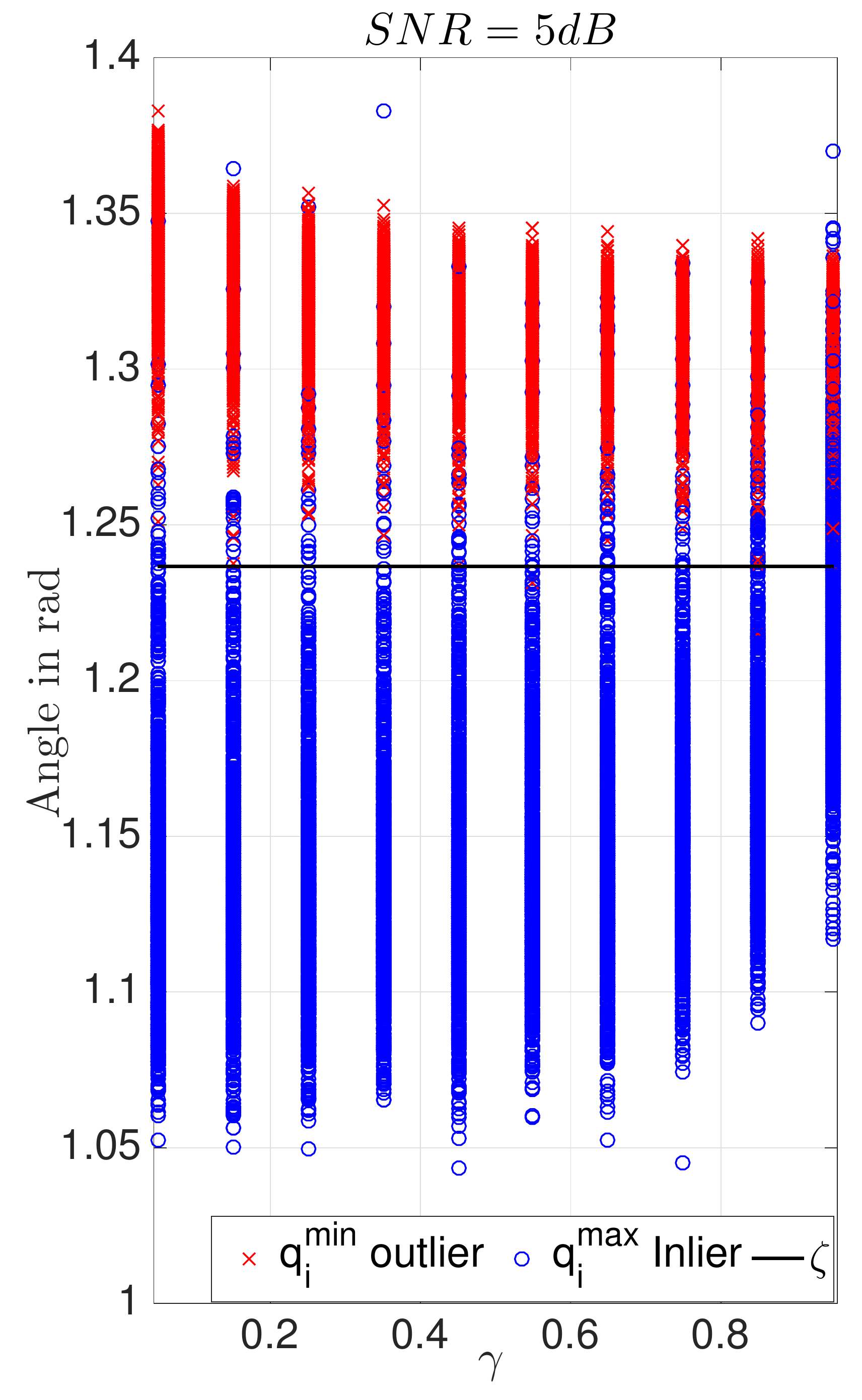}
 		\caption{n=300, r=10, N=500}
 		\label{fval2}
 	\end{subfigure}
 	\caption{Validation of $\zeta$} 
 	\label{fbvalid}
 \end{figure}
 \begin{figure}[h]
		\includegraphics[width=\linewidth, height=4cm]{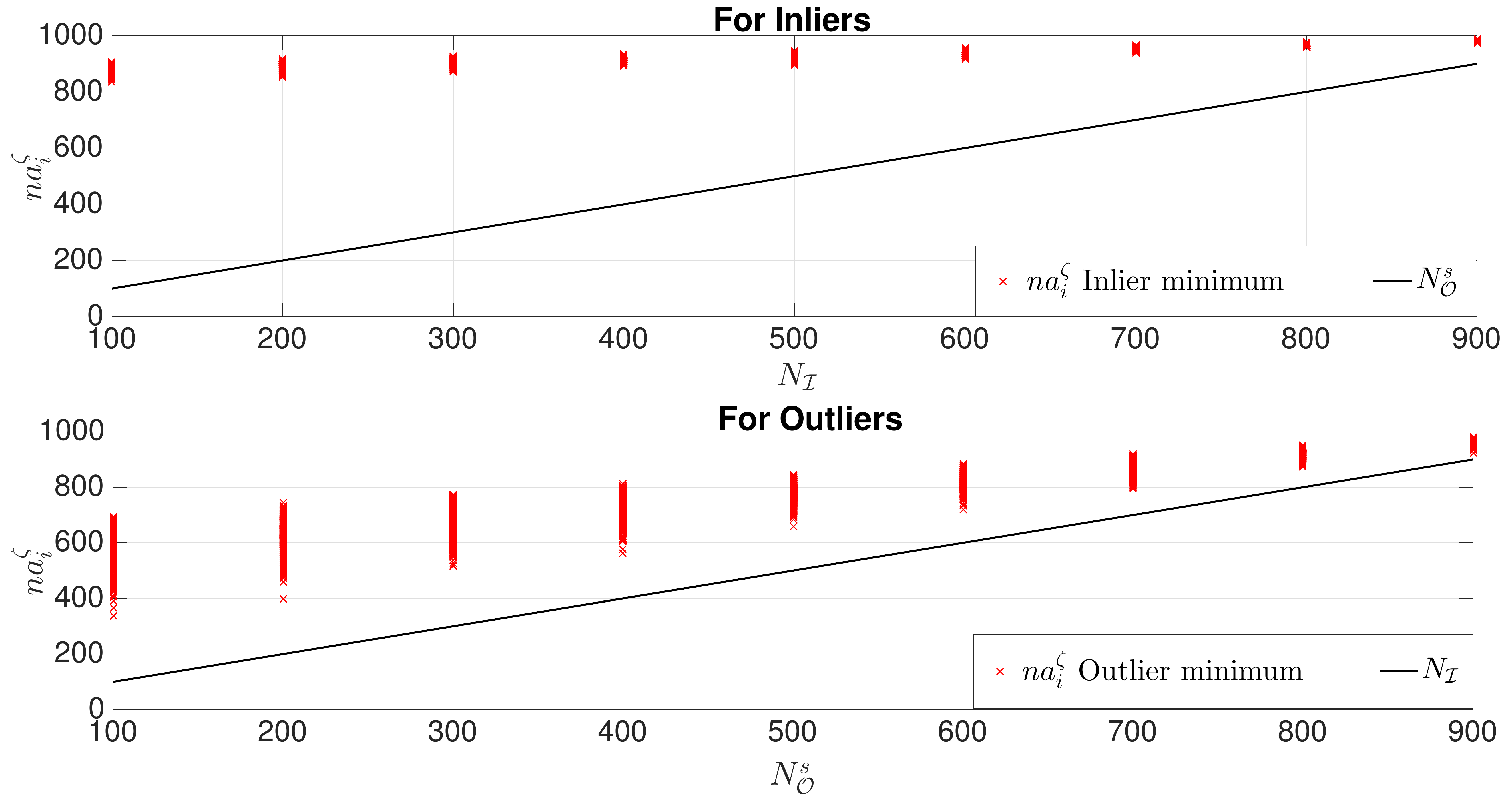}
	\caption{Validation of bounds on $na^\zeta_i$ for inlier and outlier, n=100, r=10, N=1000 for structured outliers} 
	\label{fbvalid2}
\end{figure}
\subsection{Phase transitions}
\begin{figure}[h]
	\begin{subfigure}[b]{0.23\textwidth}
		\includegraphics[width=\linewidth, height=4cm]{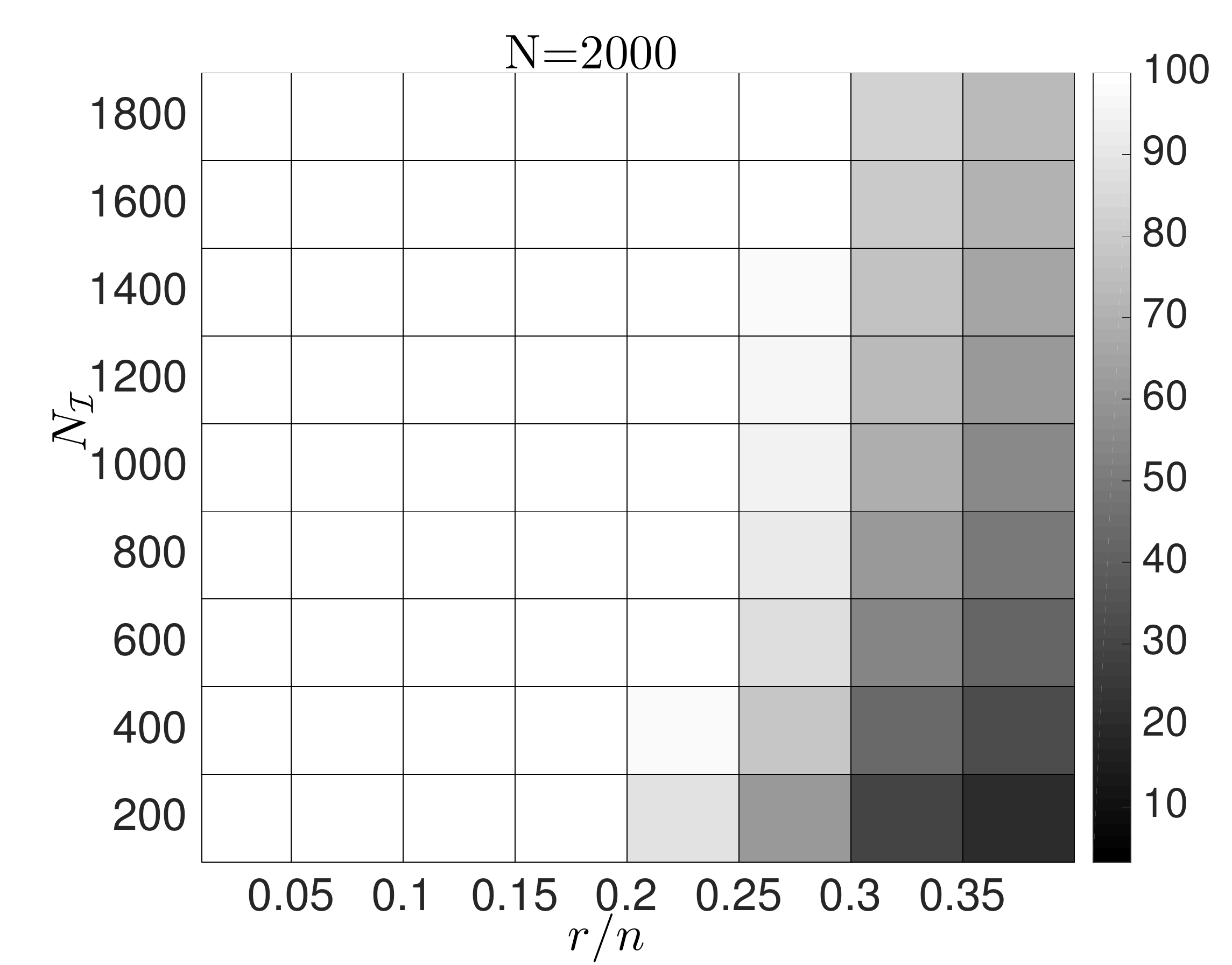}
		\caption{ \% of inliers recovered}
		\label{finlierrec}
	\end{subfigure}
	\begin{subfigure}[b]{0.23\textwidth}
		\includegraphics[width=\linewidth, height = 4 cm]{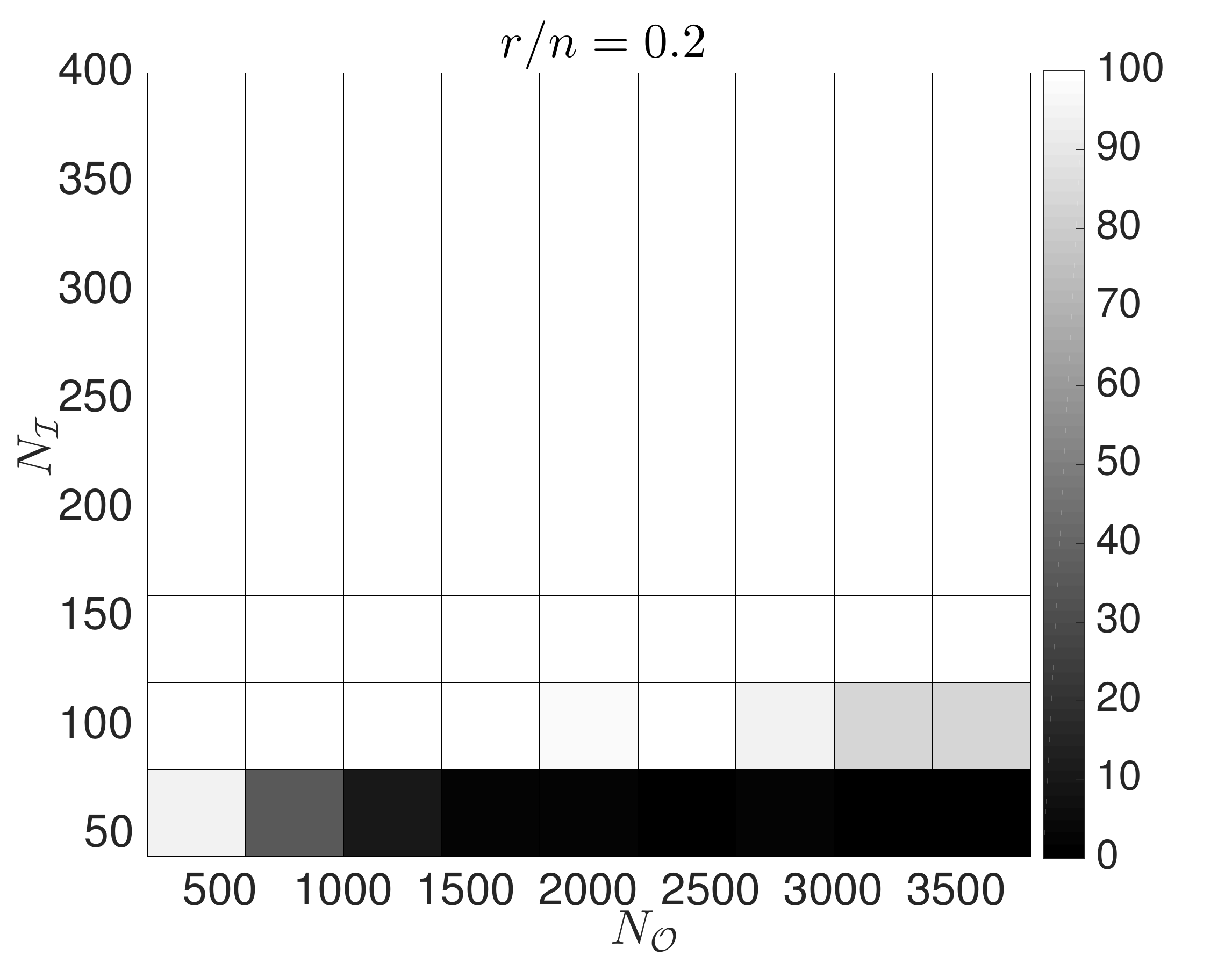}
		\caption{Exact Subspace Recovery}
		\label{flre}
	\end{subfigure}
	\caption{Phase transition plots of ROMA} 
	\label{fphase}
\end{figure}
Now, we look at the properties of the proposed algorithm in terms of percentage of inliers recovered and error in subspace recovery. First we look at the percentage of inliers recovered by ROMA against varying $\frac{r}{n}$ and the number of inliers $N_{\mathcal{I}}$, as these two are the critical parameters that determines the inlier recovery property of the algorithm. Fig. \ref{finlierrec} shows the phase transition on inlier recovery. White indicates $100\%$ inlier recovery and as the squares become darker, the inlier recovery becomes more poor. For this experiment we have varied $\frac{r}{n}$ from $0.05$ to $0.4$ by varying both $r$ and $n$, along with varying the number of inliers from $N_\mathcal{I} = 100$ to $1900$ keeping the total number of points at $N=2000$. As can be interpreted from the figure, a very high percentage of inliers are recovered for even very small $N_\mathcal{I}$, when $\frac{r}{n}$ is sufficiently low. As $\frac{r}{n}$ increases, the percentage of inliers recovered decreases. This simulation result agrees with Lemma \ref{lrough}. 

In the next experiment we have looked at the subspace recovery property of the algorithm. Here we have considered the noiseless case. After removing outliers through ROMA, the subspace is recovered after doing SVD on the remaining points and choosing the left singular vectors corresponding to the non zero singular values as the recovered subspace basis. A subspace is said to be recovered when $LRE < -5$ for the estimated subspace. Fig. \ref{flre} plots the percentage of trials in which the true subspace was recovered against $N_\mathcal{I}$ and $N_{\mathcal{O}}$ with white indicating $100\%$ success. For this phase transition plot, we have set $n=100$ and $r=20$ with 100 trials per $N_\mathcal{I},N_\mathcal{O}$ value. It is evident from Fig.  \ref{flre} that, whenever there is a good enough number of inliers ($>100$ in this case), no matter what the number of outliers is, the subspace is recovered with minimal error. The subspace recovery suffers when the number of inliers is low, as seen in the last row of Fig. \ref{flre}. 
\begin{table*}[h]
	\caption{Comparison of Algorithms - Unstructured outliers, $n=100$, $r=20$, $N=1000$}
	\label{table1}
	\begin{tabularx}{\textwidth}{@{}l*{10}{C}c@{}}
		\toprule
		Algorithm     &$LRE$ at $\gamma=0.1$&$LRE$  at $\gamma=0.5$ & $LRE$ at $\gamma=0.9$ & Average running time in seconds & Parameter knowledge & Free parameters  \\ 
		\midrule
		FMS  & -14.58      & -14.58          &-14.59   & 0.8011 & $r$   & $>1$  \\ 
		GMS & -$2.7\times10^{-4}$   & -$2.8\times10^{-4}$        &-$3.8\times10^{-4}$    & 0.0437 & $r$    & Regularization   \\ 
		ORSC  &-14.57    & -14.58          &-14.58  & 58.94  & None   & $>1$   \\ 
		CoP  & -14.59      & -14.59          &-14.59     & 0.0164  & $\gamma$ or $r$    & No   \\ 
		Heckel's  & -14.58      & -14.58          &-12.73     & 0.0185 & None    & No   \\ 
		ROMA & -14.58    & -14.58          &-14.59   & 0.0365 & None   & No   \\ 
		\bottomrule
	\end{tabularx}
\end{table*}
\subsection{Comparison with other state of the art algorithms}\label{scomp}
Here we compare the proposed algorithm with existing techniques CoP \cite{rahmani2016coherence}, Fast Median Subspace, FMS \cite{lerman2014fast}, Geometric Median Subspace GMS \cite{zhang2014novel} and the outlier removal algorithm outlier removal for subspace clustering (denoted by ORSC for convenience) in \cite{soltanolkotabi2012geometric}, in terms of the log recovery error and running time under data model in Assumption \ref{amain}. For FMS, we used algorithm 1 in \cite{lerman2014fast}, with default parameter setting i.e. $p=1,\epsilon =10^{-10}$, maximum iterations 100. For GMS as well, we used the default parameter settings and chose the last $r$ columns from the output matrix as the basis of the estimated subspace. For CoP, we implemented the first method proposed, where we used the number of data points chosen for subspace recovery as $n_s = 30$, which is a value always less than the number of inliers in our experimental settings and hence works well. For ORSC, we used the algorithm using primal-dual interior point method from the $l_1$ magic code repository \cite{candes2005l1} for solving the underlying $l_1$ optimization problem. The parameters used were changed from the default settings to improve convergence rate without degrading the performance. We also use the outlier detection method in \cite{heckel2015robust}, which we refer as Heckel's. In Table \ref{table1}, we have summarized each algorithm in terms of its performance measured in terms of log recovery error at various outlier fractions, running time and the parameters used by the algorithm for its working. Also the last column indicates other free parameters that an algorithm requires like regularization or convergence parameters. For the experiments in Table \ref{table1}, we have set $N=1000,n=100,r=20$. It is observed that ROMA performs at par with the existing methods in terms of $LRE$ without requiring the knowledge of $r$ or $\gamma$ and is nearly as quick as CoP. Heckel's algorithm also does well in this scenario of subspace recovery under unstructured outliers assuming no parameter knowledge. The algorithm ORSC which also does not use parameter knowledge, has similar $LRE$ values but is much slower compared to ROMA and also requires multiple parameters like a convergence criterion for solving the underlying $l_1$ optimization problem.
\begin{figure}[h]
	\includegraphics[width=8 cm, height = 6 cm]{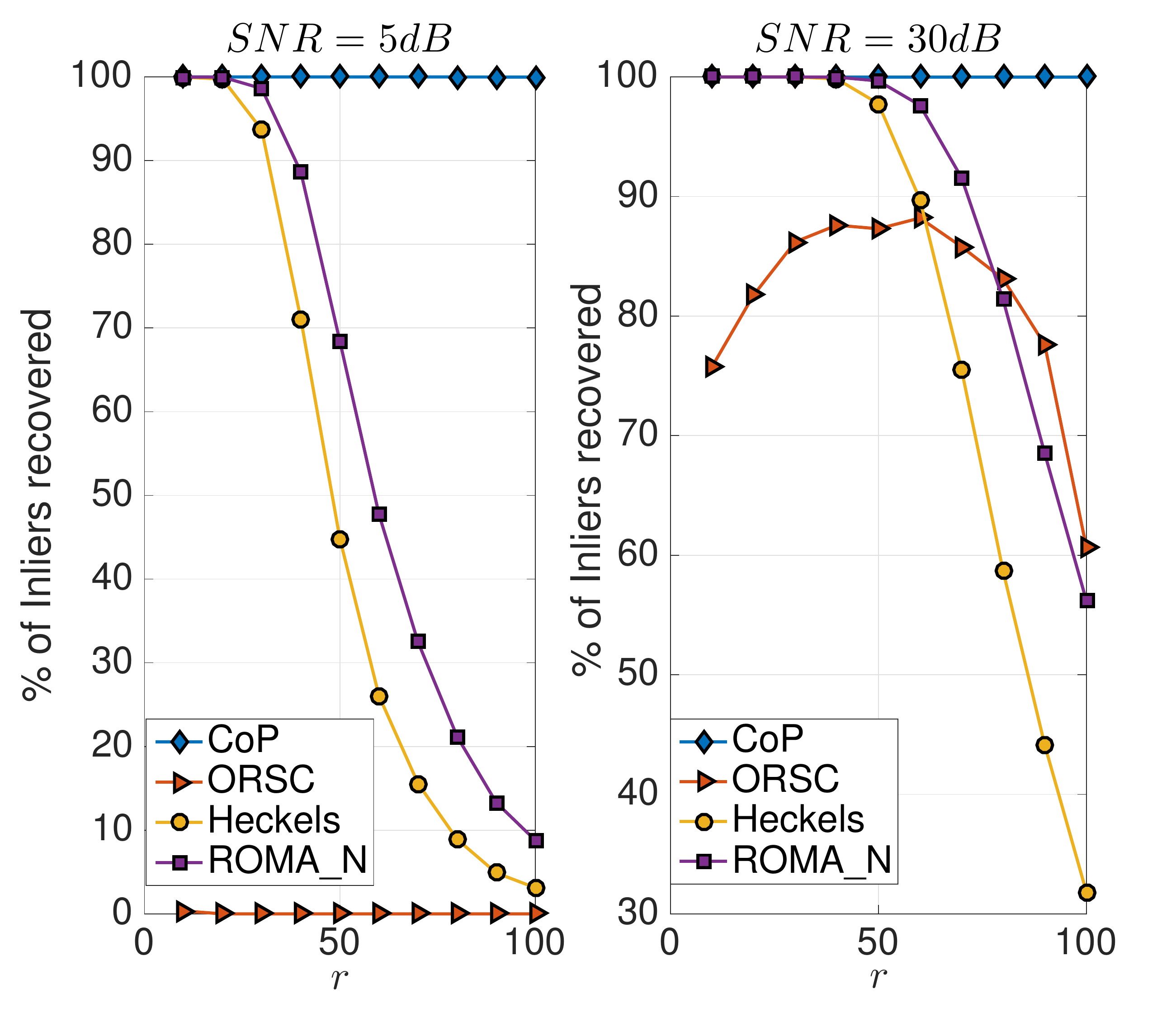}
	\caption{\% of inliers recovered comparison between algorithms at $N=400,n=300,\gamma = 0.5$ }
	\label{f1}
\end{figure}
Fig. \ref{f1} has the comparison results of algorithms in terms of inlier recovery for data following Assumption \ref{amain} but with noisy inliers. For this experiment we set $N=400,n=300,\gamma = 0.5$ and vary the rank of the subspace from $r=10$ to $100$ and look at the percentage of inliers recovered by algorithms ROMA, CoP, Heckel's and ORSC in two different SNR scenarios, $SNR =5 dB$ and $SNR=30 dB$. We give $n_s = N_{\mathcal{I}}$ as input to CoP here and hence it performs the best in terms of inlier recovery. ROMA performs much better than the other parameter free methods, Heckel's and also ORSC. CoP wrongly flags outliers as inliers in extreme cases of low SNR and high rank, for instance here with a rate of $0.04$ at $SNR=5,r=100$, which is undesirable in many applications. All other algorithms had 0 false inlier detection in all the trials. 
\subsection{Structured Outliers}
Here we consider structured outliers and compare the performance of ROMA\_N with other algorithms. We use the model in  \cite{rahmani2016coherence} to generate structured data, where the outlier is generated as $\textbf{x}_i = \frac{1}{\sqrt{1+\mu^2}}(\textbf{a}+\textbf{b}_i)$, where $\textbf{a}$ and $\textbf{b}_i$ are points chosen uniformly at random from $\mathbb{S}^{n-1}$. This means that the points are clustered around the point $\textbf{a}$, and the clustering is determined by $\mu$. As $\mu$ decreases the outliers are more clustered and $\theta_{max}^{\mathcal{O}}$ reduces. Also here the inliers are generated as $\textbf{x}_i = \frac{1}{\sqrt{1+\nu^2}}(\textbf{u}+\textbf{v}_i)$, where $\textbf{u}$ and $\textbf{v}_i$ are points chosen uniformly at random from $\mathbb{S}^{n-1}\cap\mathcal{U}$. In the experiments for Table \ref{table3} and Fig. \ref{fso}, we set the value of $\nu=0.1$. We vary $\mu$, thereby the clustering of outliers and compare the LRE values for different algorithms in Table \ref{table3} for two different cases of $N_{\mathcal{I}},N_{\mathcal{O}}^s$ combinations. We set $n=200$, $r=10$ for these experiments and $n_s = 50$ is used for CoP. As can be seen, Heckel's outlier detection algorithm fails for lower $\mu$, i.e. more clustered outliers and so too does FMS.
\begin{table}[h]
	\caption{LRE for structured outliers against $\mu$}
	\label{table3}
	\begin{tabularx}{\linewidth}{@{}l*{10}{C}c@{}}
		\hline
		&\multicolumn{3}{c|}{$N_{\mathcal{I}} = 900, N_{\mathcal{O}}^s=100$}&\multicolumn{3}{c}{$N_{\mathcal{I}} = 300, N_{\mathcal{O}}^s=700$}\\
		\multicolumn{1}{r}{$\mu$}&\multicolumn{1}{c}{0.2}&\multicolumn{1}{c}{0.5}&\multicolumn{1}{c|}{5}&\multicolumn{1}{c}{0.2}&\multicolumn{1}{c}{0.5}&\multicolumn{1}{c}{5}\\
		\hline
		CoP&$-14.5$&$-14.5$&\multicolumn{1}{c|}{$-14.5$}&$-0.01$&$-0.01$&$-14.5$\\
		FMS&$-0.5$&$-0.5$&\multicolumn{1}{c|}{$-14.6$}&$-0.5$&$-0.5$&$-0.5$\\
		Heckel&$-0.5$&$-0.5$&\multicolumn{1}{c|}{$-14.6$}&$-0.4$&$-0.05$&$-14.4$\\
		ROMA\_N&$-14.5$&$-14.5$&\multicolumn{1}{c|}{$-14.6$}&$-14.4$&$-14.4$&$-14.4$\\
		\hline
	\end{tabularx}
\end{table}
\begin{figure}[h]
	\includegraphics[width=8 cm, height = 4cm]{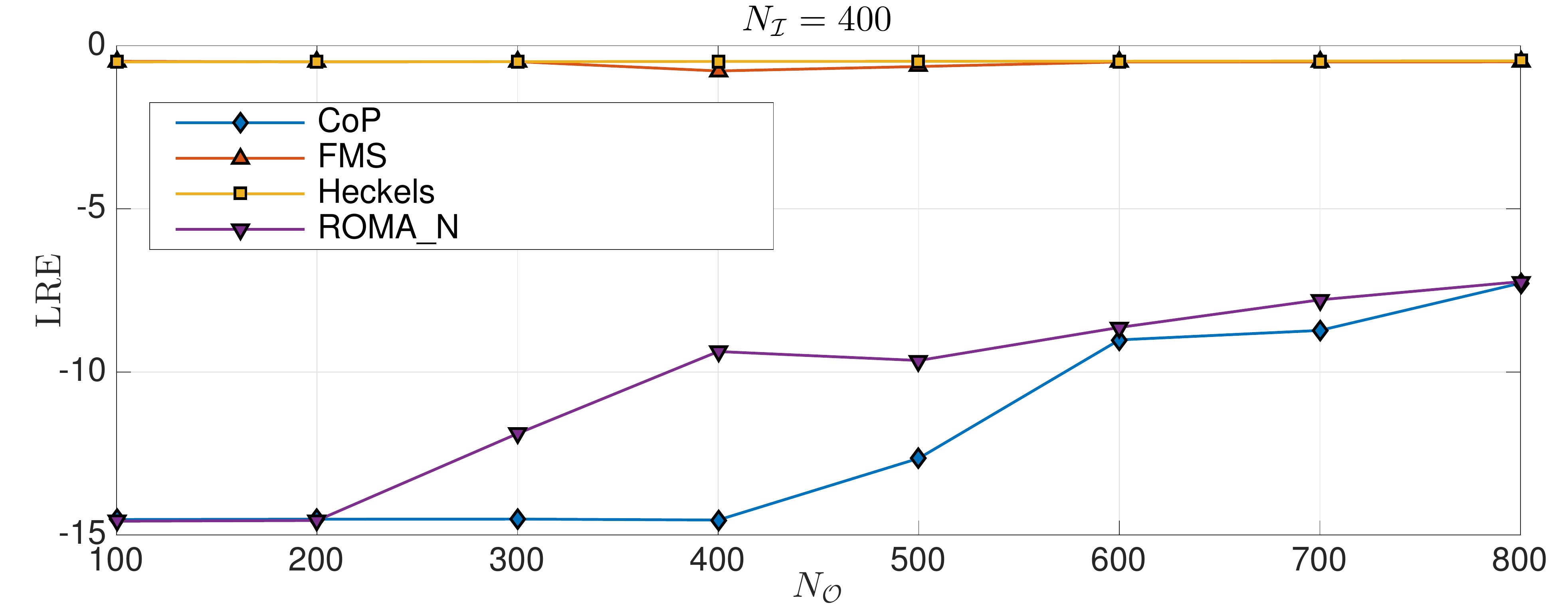}
	\caption{LRE comparison - both structured and unstructured outliers}
	\label{fso}
\end{figure} 
CoP works well in the case where the inliers are more, but fails in cases of low $\mu$ when $N_{\mathcal{O}}^s>N_{\mathcal{I}}$. In all the cases ROMA\_N works very well in terms of LRE. The next experiment in this section was performed by mixing structured and unstructured outliers to form the full outlier set. $N_{\mathcal{O}}$ was varied from $100$ to $800$, while $N_{\mathcal{I}} = 400$. $\mu$ was set at $0.2$ with $n=200,r=10$. Out of the $N_{\mathcal{O}}$ outliers, a random number, $N_{\mathcal{O}}^s$ of them were picked from the structured set and the rest were chosen uniformly at random from $\mathbb{S}^{n-1}$. The results averaged over a 100 trials are plotted in Fig. \ref{fso}. Here too ROMA\_N performs as well as CoP, while the other algorithms have poor LRE.
\section{Real data Experiments}
\begin{figure*}[h]
	\begin{subfigure}[b]{0.083\textwidth}
		\includegraphics[width=\linewidth, height=1cm]{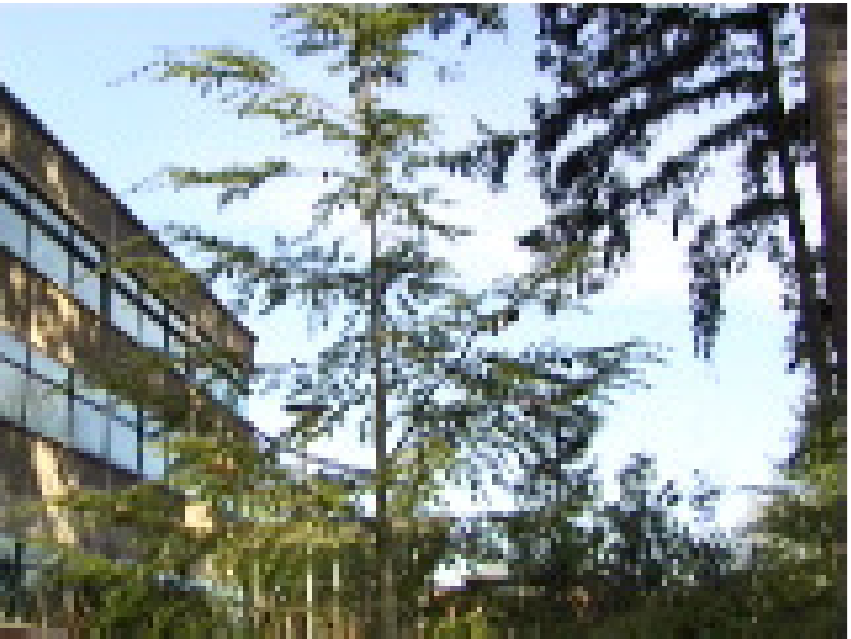}
	\end{subfigure}
	\begin{subfigure}[b]{0.083\textwidth}
	\includegraphics[width=\linewidth, height=1cm]{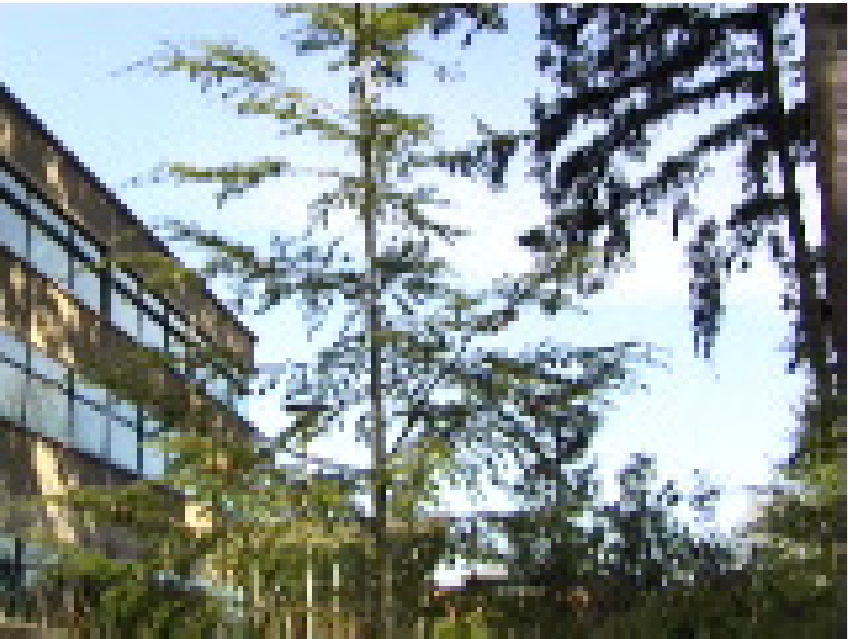}
\end{subfigure}
	\begin{subfigure}[b]{0.083\textwidth}
	\includegraphics[width=\linewidth, height=1cm]{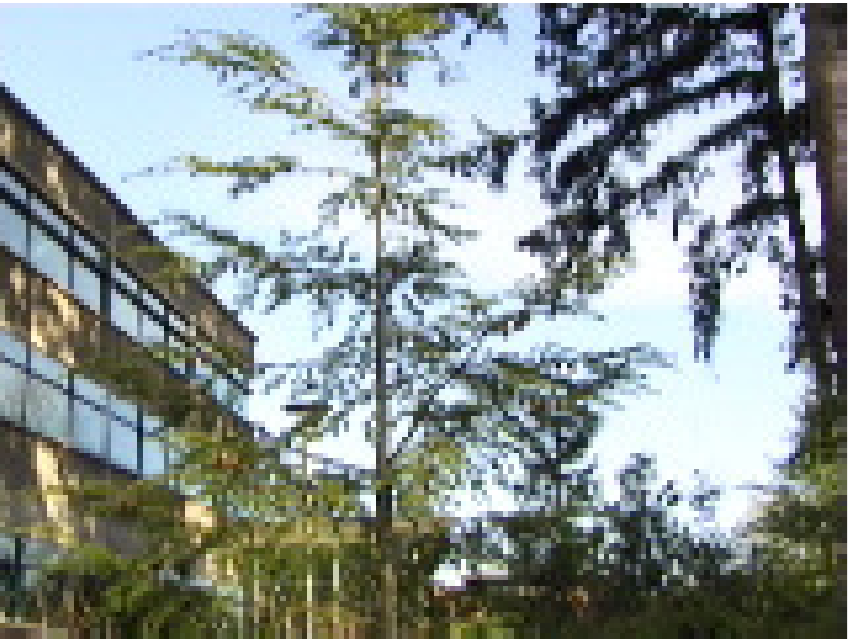}
\end{subfigure}
		\begin{subfigure}[b]{0.083\textwidth}
		\includegraphics[width=\linewidth, height=1cm]{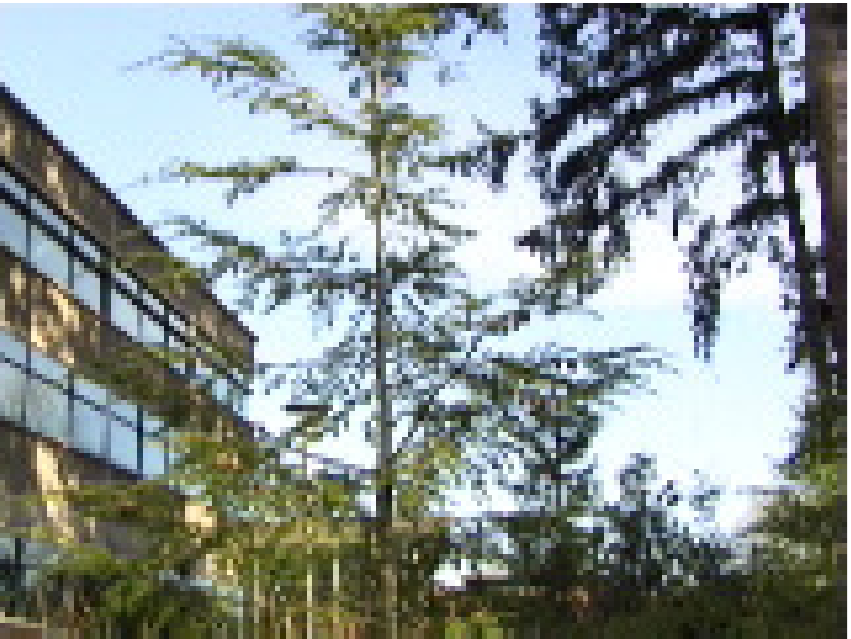}
	\end{subfigure}
		\begin{subfigure}[b]{0.083\textwidth}
			\setlength{\fboxsep}{0pt}%
			\setlength{\fboxrule}{1pt}%
		\cfbox{red}{\includegraphics[width=\linewidth, height=1cm]{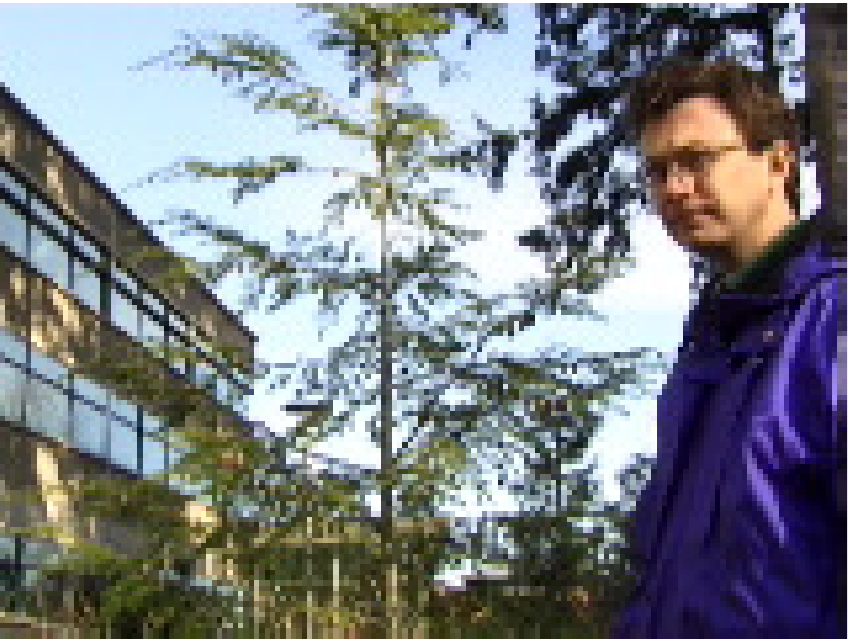}}
	\end{subfigure}
		\begin{subfigure}[b]{0.083\textwidth}
			\setlength{\fboxsep}{0pt}%
			\setlength{\fboxrule}{1pt}%
		\cfbox{red}{\includegraphics[width=\linewidth, height=1cm]{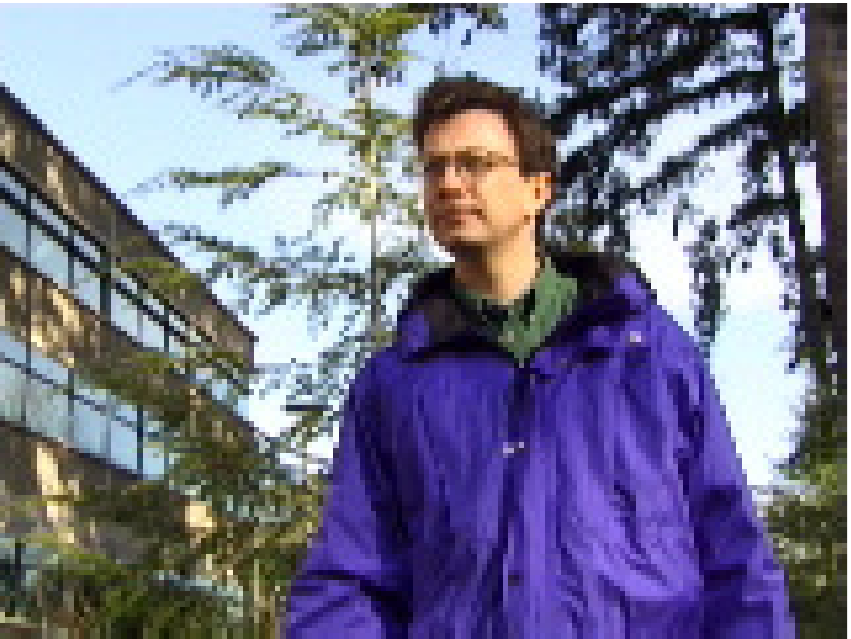}}
	\end{subfigure}
		\begin{subfigure}[b]{0.083\textwidth}
			\setlength{\fboxsep}{0pt}%
			\setlength{\fboxrule}{1pt}%
		\cfbox{red}{\includegraphics[width=\linewidth, height=1cm]{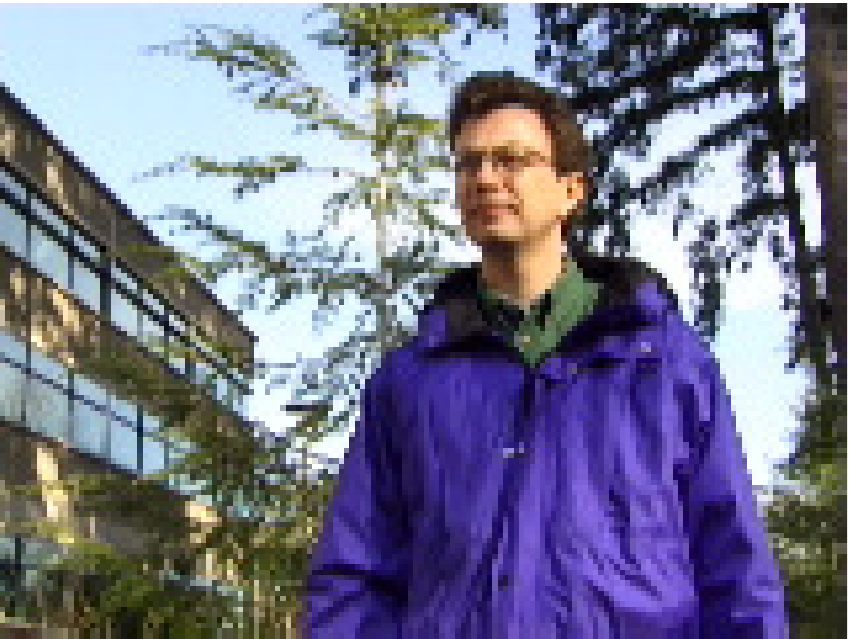}}
	\end{subfigure}
		\begin{subfigure}[b]{0.083\textwidth}
			\setlength{\fboxsep}{0pt}%
			\setlength{\fboxrule}{1pt}%
		\cfbox{red}{\includegraphics[width=\linewidth, height=1cm]{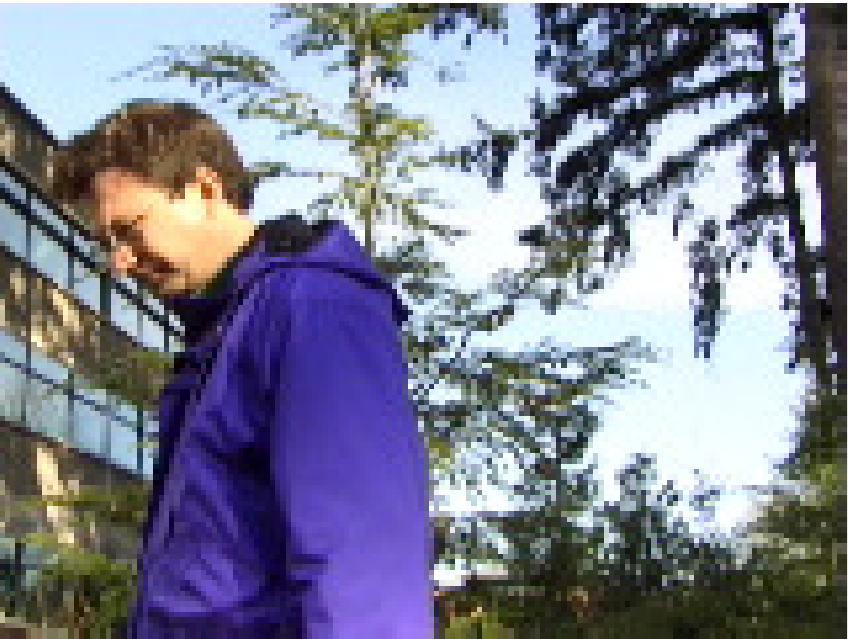}}
	\end{subfigure}
		\begin{subfigure}[b]{0.083\textwidth}
		\includegraphics[width=\linewidth, height=1cm]{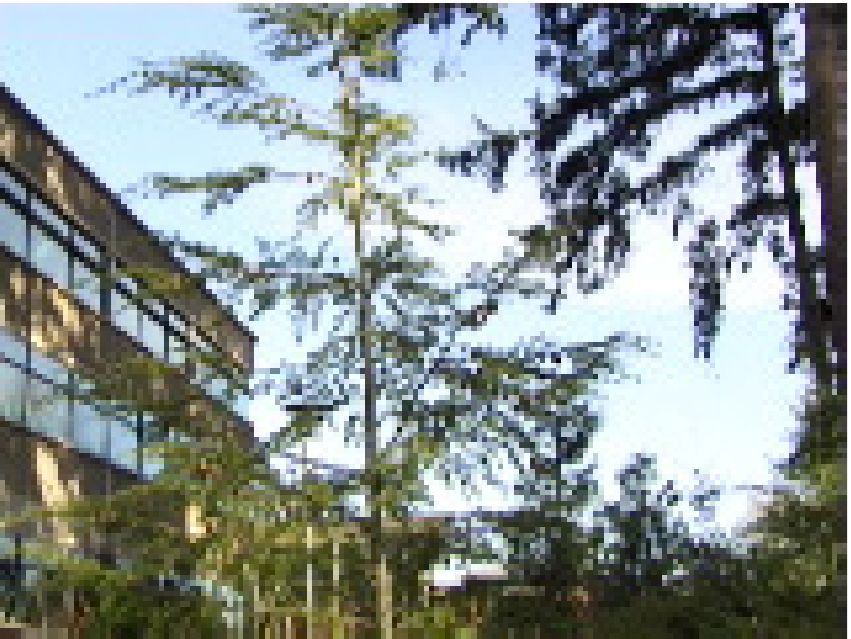}
	\end{subfigure}
\begin{subfigure}[b]{0.083\textwidth}
	\includegraphics[width=\linewidth, height=1cm]{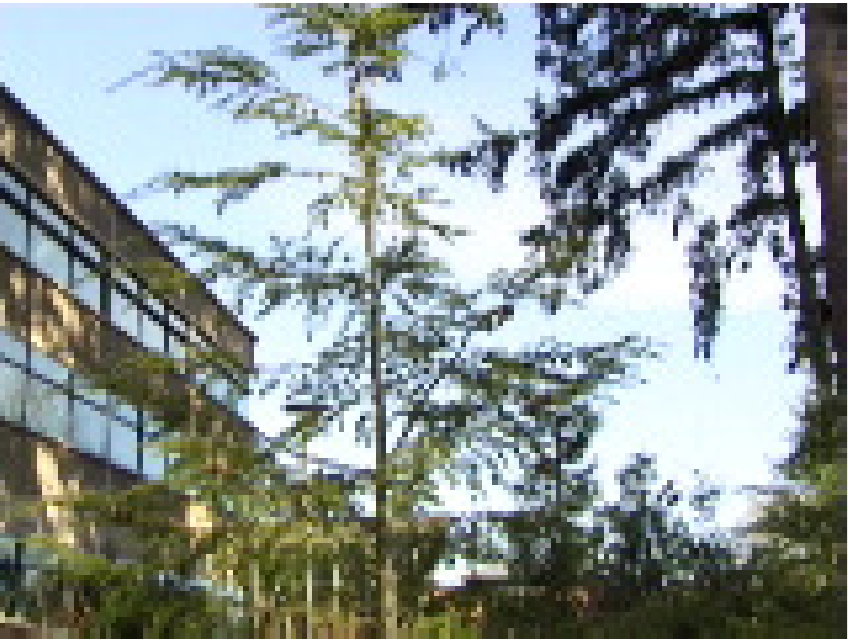}
\end{subfigure}
\begin{subfigure}[b]{0.083\textwidth}
	\includegraphics[width=\linewidth, height=1cm]{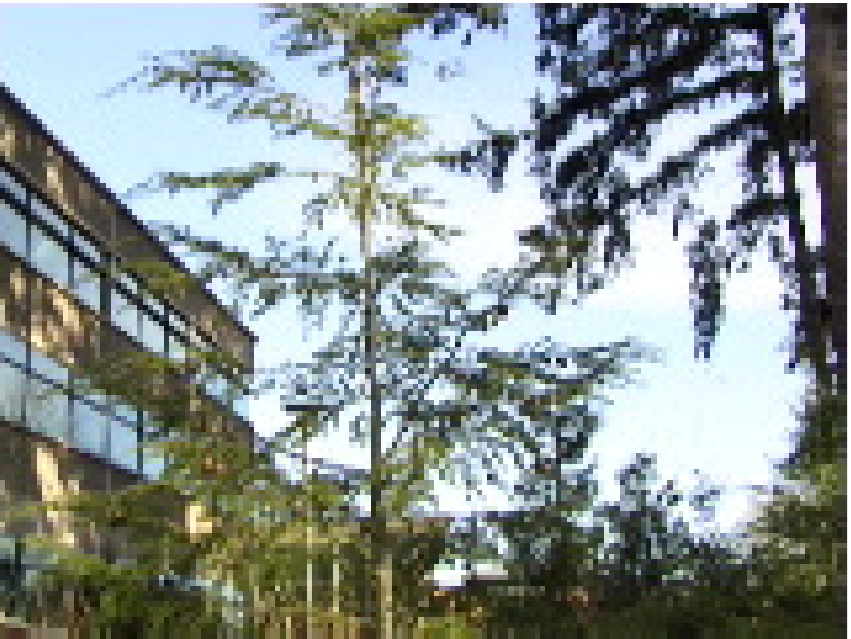}
\end{subfigure}
\caption{Frames from the waving tree video, Highlighted frames are detected as outliers by CoP, FMS and ROMA\_N}
\label{fwavingtree}
\end{figure*} 
\begin{figure*}[h]
	\begin{subfigure}[b]{0.083\textwidth}
		\includegraphics[width=\linewidth, height=1cm]{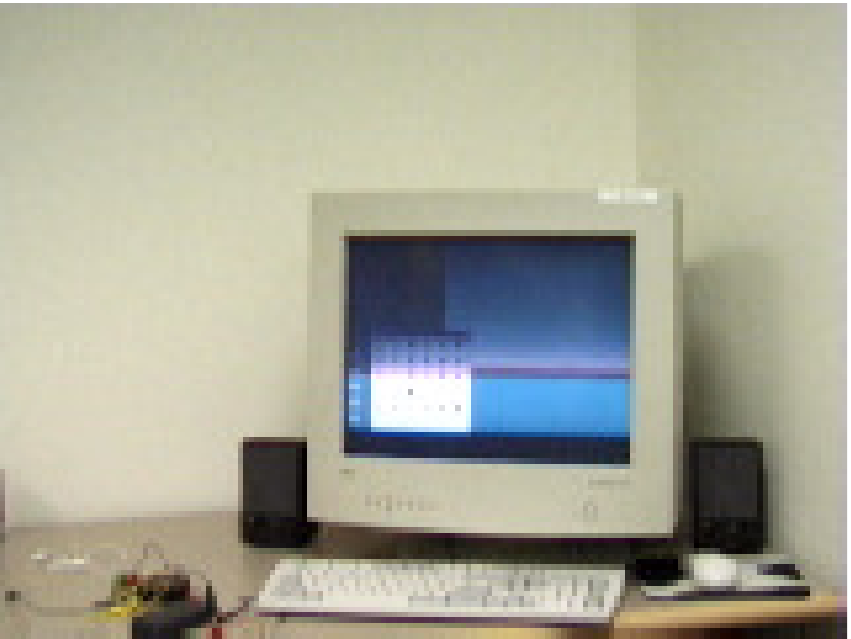}
	\end{subfigure}
	\begin{subfigure}[b]{0.083\textwidth}
		\includegraphics[width=\linewidth, height=1cm]{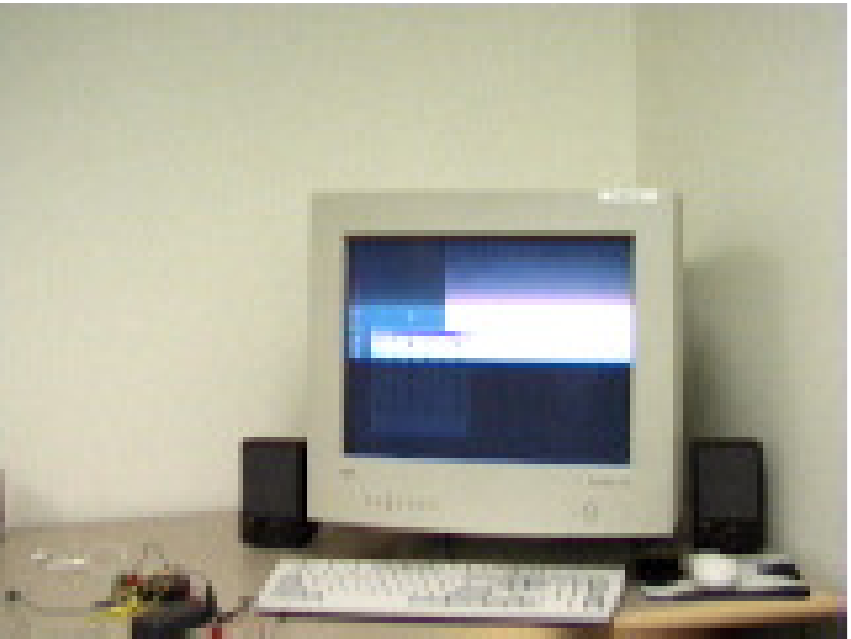}
	\end{subfigure}
	\begin{subfigure}[b]{0.083\textwidth}
		\includegraphics[width=\linewidth, height=1cm]{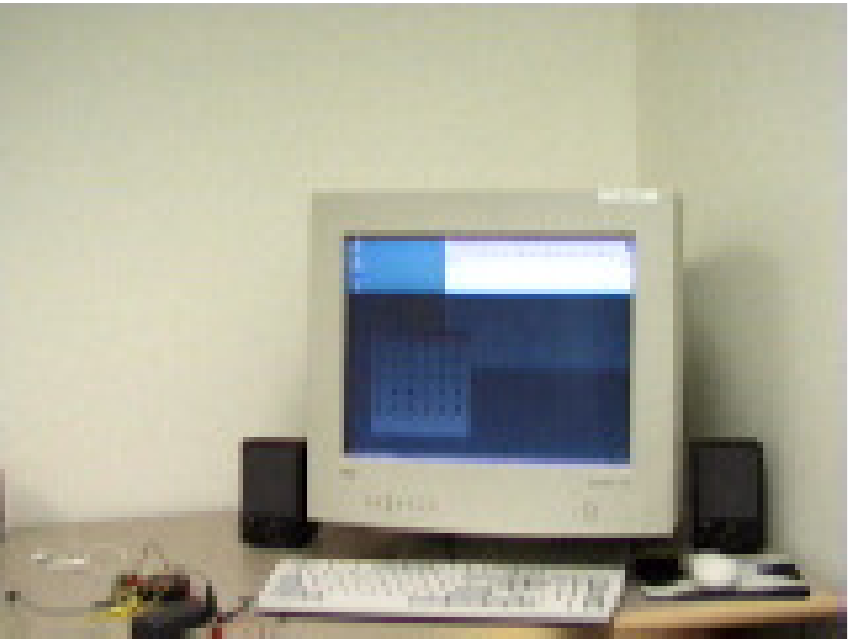}
	\end{subfigure}
	\begin{subfigure}[b]{0.083\textwidth}
		\setlength{\fboxsep}{0pt}%
		\setlength{\fboxrule}{1pt}%
		\cfbox{red}{\includegraphics[width=\linewidth, height=1cm]{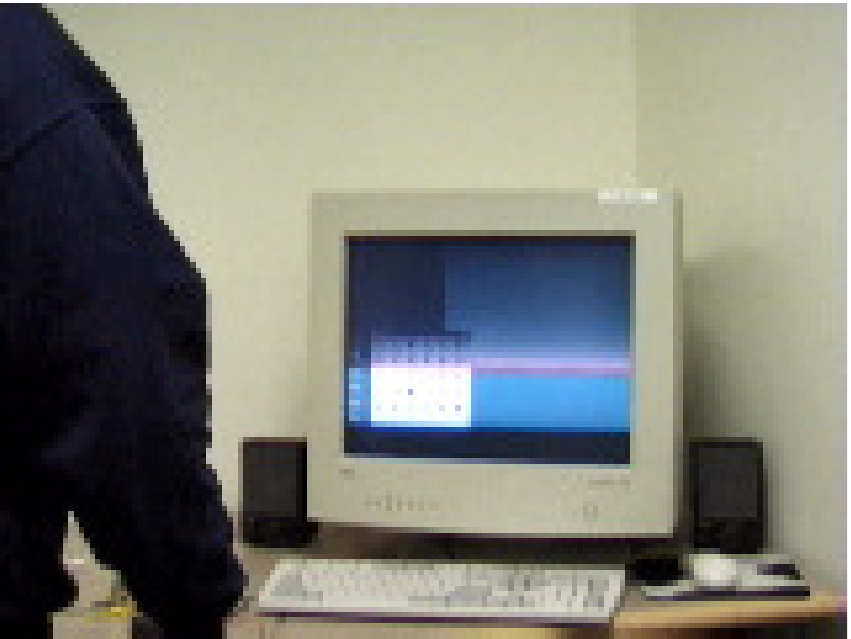}}
	\end{subfigure}
	\begin{subfigure}[b]{0.083\textwidth}
		\setlength{\fboxsep}{0pt}%
		\setlength{\fboxrule}{1pt}%
		\cfbox{red}{\includegraphics[width=\linewidth, height=1cm]{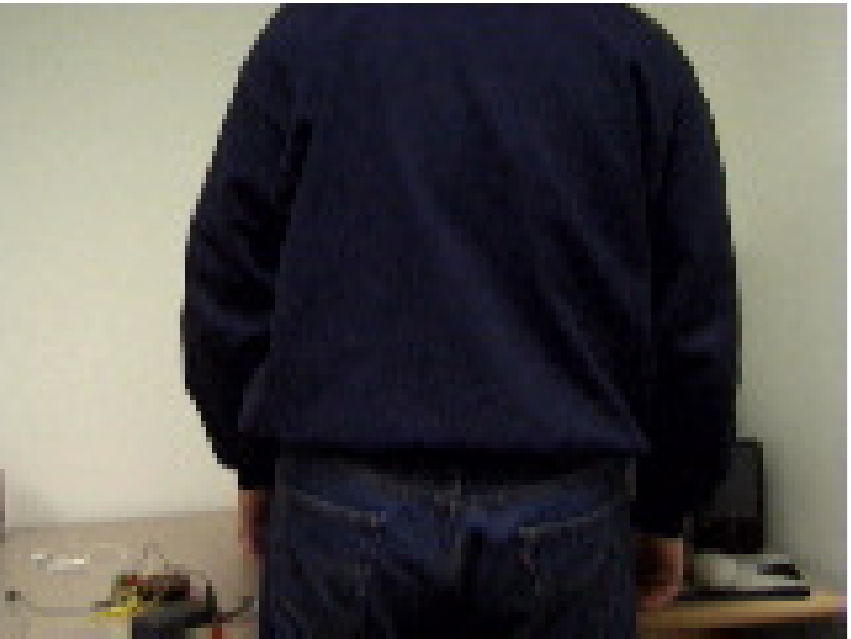}}
	\end{subfigure}
	\begin{subfigure}[b]{0.083\textwidth}
		\setlength{\fboxsep}{0pt}%
		\setlength{\fboxrule}{1pt}%
		\cfbox{red}{\includegraphics[width=\linewidth, height=1cm]{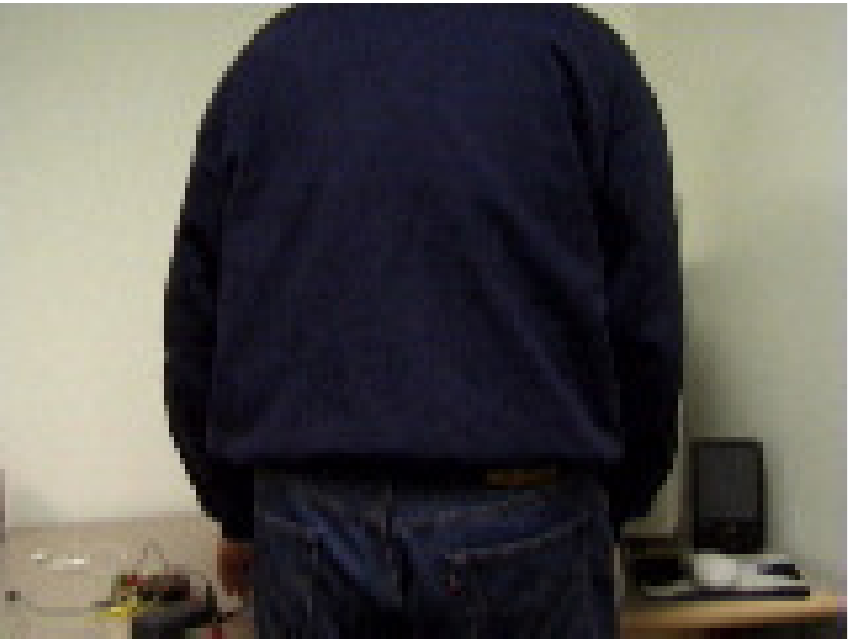}}
	\end{subfigure}
	\begin{subfigure}[b]{0.083\textwidth}
		\setlength{\fboxsep}{0pt}%
		\setlength{\fboxrule}{1pt}%
		\cfbox{red}{\includegraphics[width=\linewidth, height=1cm]{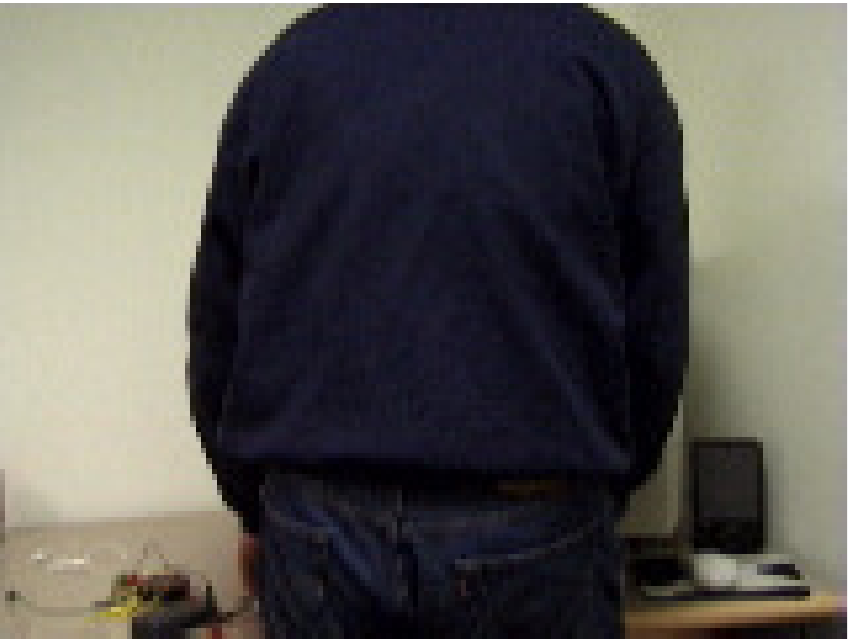}}
	\end{subfigure}
	\begin{subfigure}[b]{0.083\textwidth}
		\setlength{\fboxsep}{0pt}%
		\setlength{\fboxrule}{1pt}%
		\cfbox{red}{\includegraphics[width=\linewidth, height=1cm]{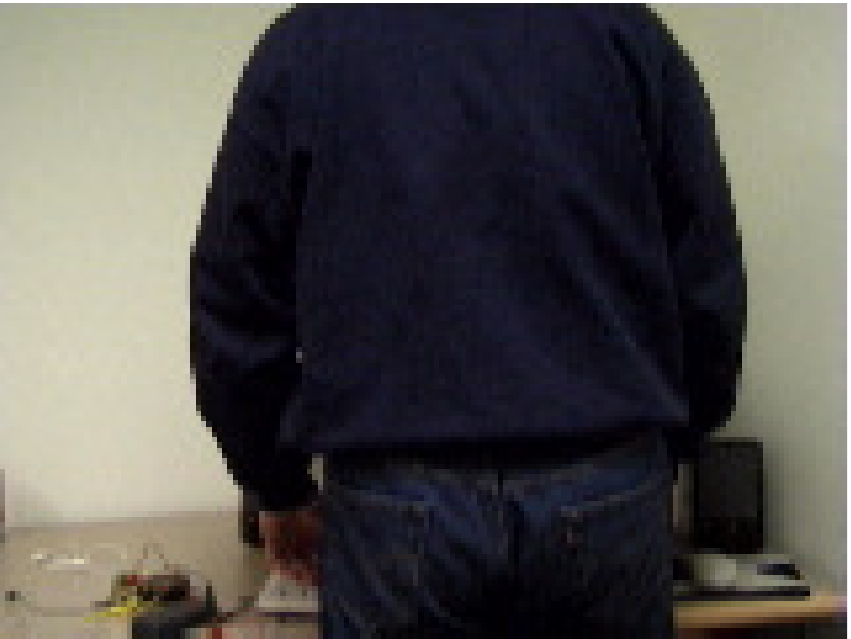}}
	\end{subfigure}
	\begin{subfigure}[b]{0.083\textwidth}
				\setlength{\fboxsep}{0pt}%
		\setlength{\fboxrule}{1pt}%
		\cfbox{red}{\includegraphics[width=\linewidth, height=1cm]{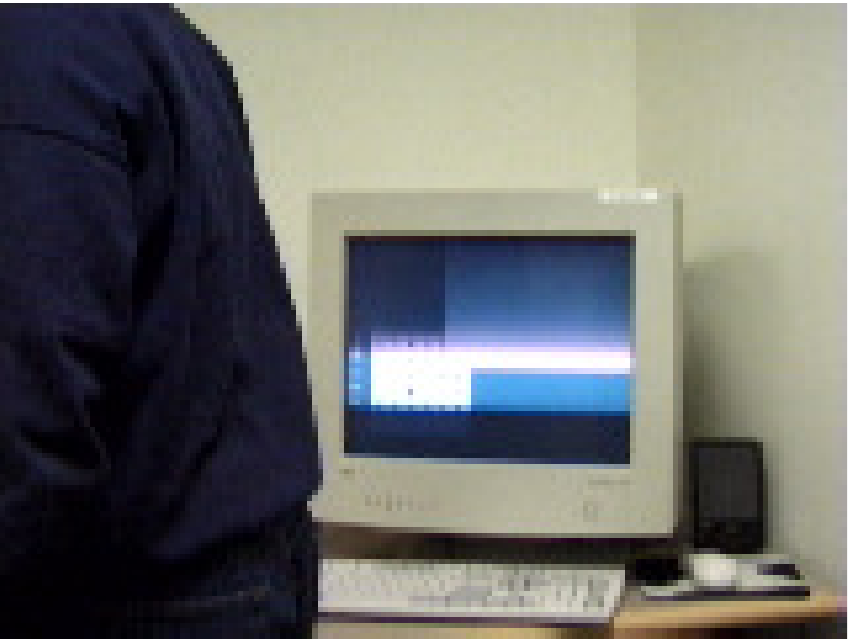}}
	\end{subfigure}
	\begin{subfigure}[b]{0.083\textwidth}
		\includegraphics[width=\linewidth, height=1cm]{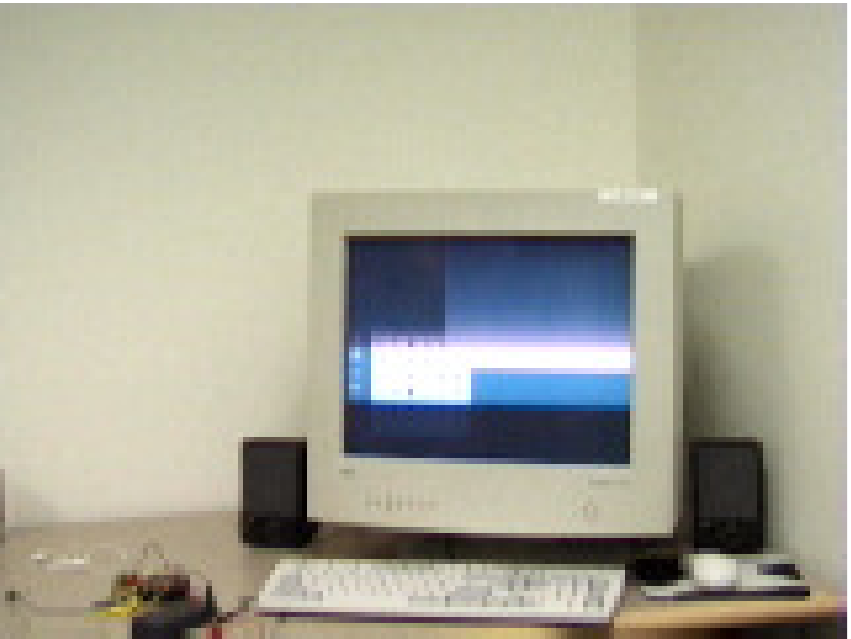}
	\end{subfigure}
	\begin{subfigure}[b]{0.083\textwidth}
		\includegraphics[width=\linewidth, height=1cm]{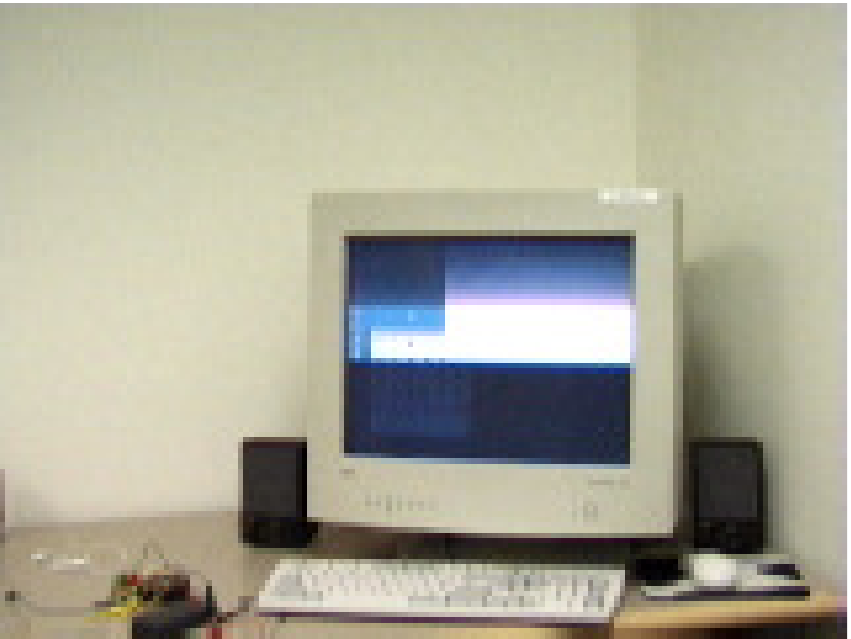}
	\end{subfigure}
	\caption{Frames from the Camouflage video, Highlighted frames are detected as outliers by CoP and ROMA\_N}
	\label{fcam}
\end{figure*} 
\begin{figure}[h]
\begin{subfigure}[b]{0.09\textwidth}
	\setlength{\fboxsep}{0pt}%
	\setlength{\fboxrule}{1pt}%
	\cfbox{red}{\includegraphics[width=\linewidth, height=1cm]{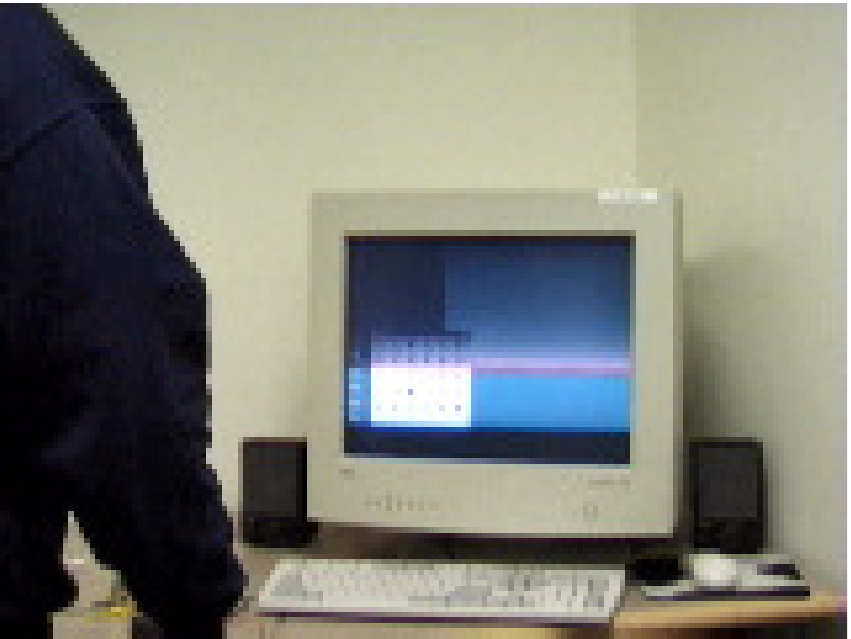}}
\end{subfigure}
\begin{subfigure}[b]{0.09\textwidth}
	\setlength{\fboxsep}{0pt}%
	\setlength{\fboxrule}{1pt}%
	\cfbox{red}{\includegraphics[width=\linewidth, height=1cm]{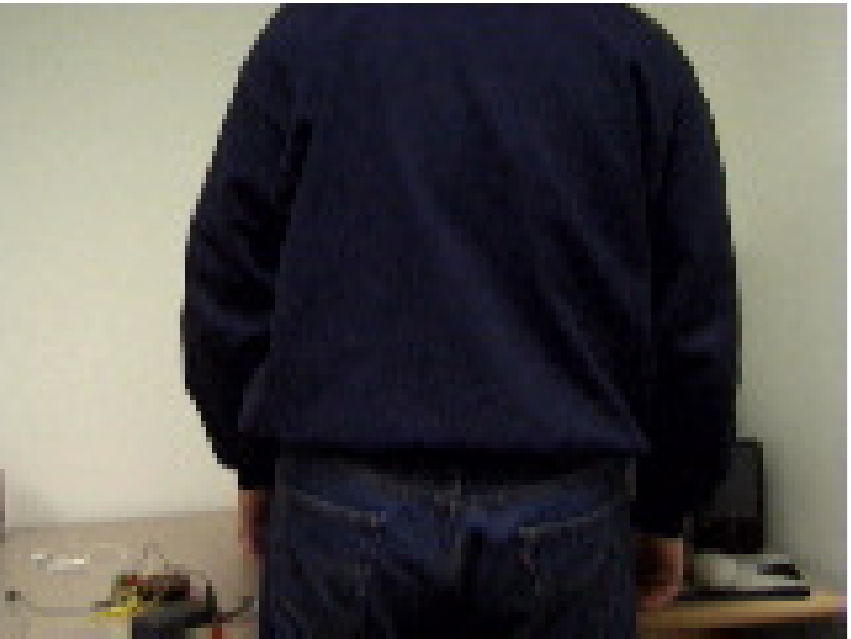}}
\end{subfigure}
\begin{subfigure}[b]{0.09\textwidth}
\includegraphics[width=\linewidth, height=1cm]{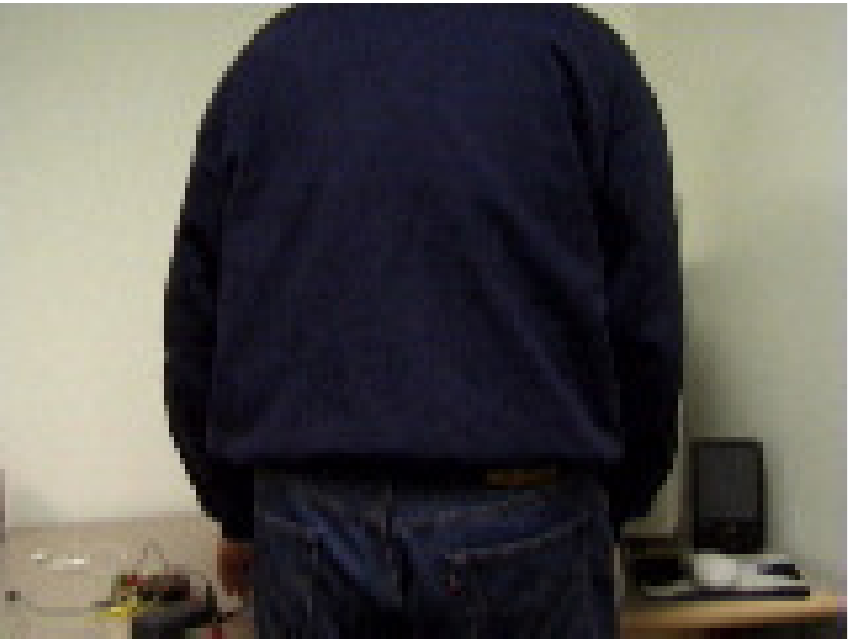}
\end{subfigure}
\begin{subfigure}[b]{0.09\textwidth}
\includegraphics[width=\linewidth, height=1cm]{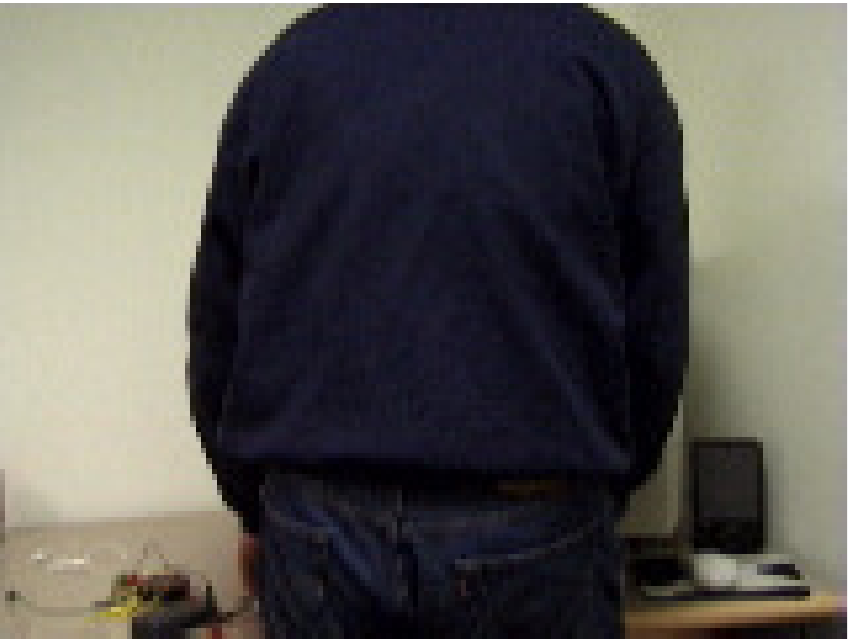}
\end{subfigure}
\begin{subfigure}[b]{0.09\textwidth}
	\setlength{\fboxsep}{0pt}%
	\setlength{\fboxrule}{1pt}%
	\cfbox{red}{\includegraphics[width=\linewidth, height=1cm]{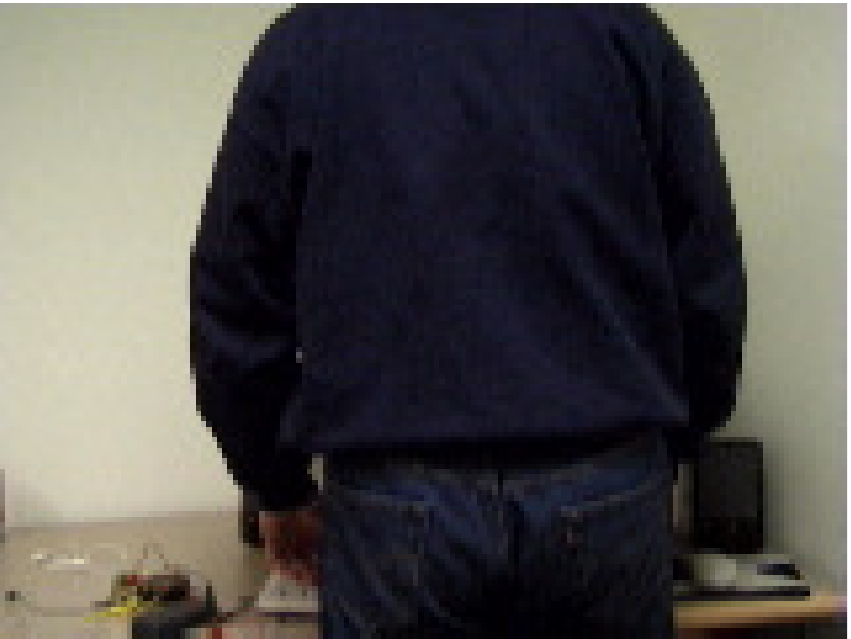}}
\end{subfigure}
\caption{Only highlighted frames are detected as outliers by FMS}
\label{fcamsub}
\end{figure}
\begin{figure*}[h]
	\begin{subfigure}[b]{0.083\textwidth}
		\includegraphics[width=\linewidth, height=1.5cm]{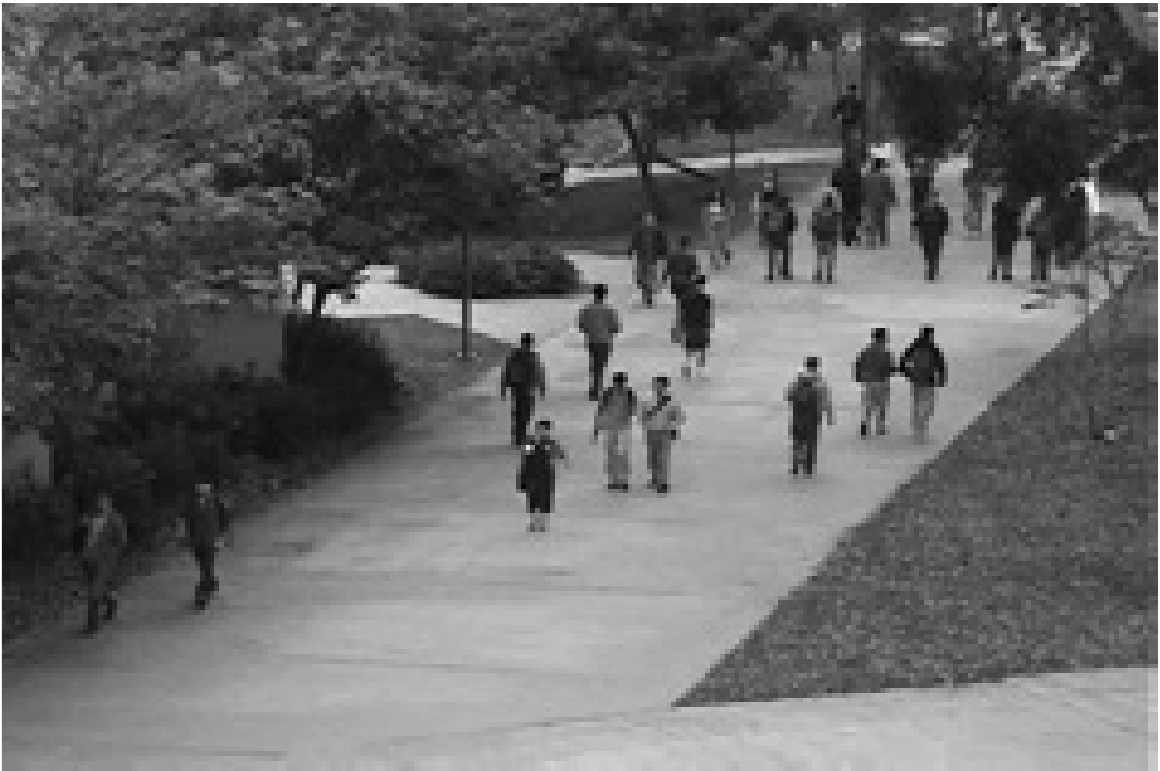}
	\end{subfigure}
	\begin{subfigure}[b]{0.083\textwidth}
		\includegraphics[width=\linewidth, height=1.5cm]{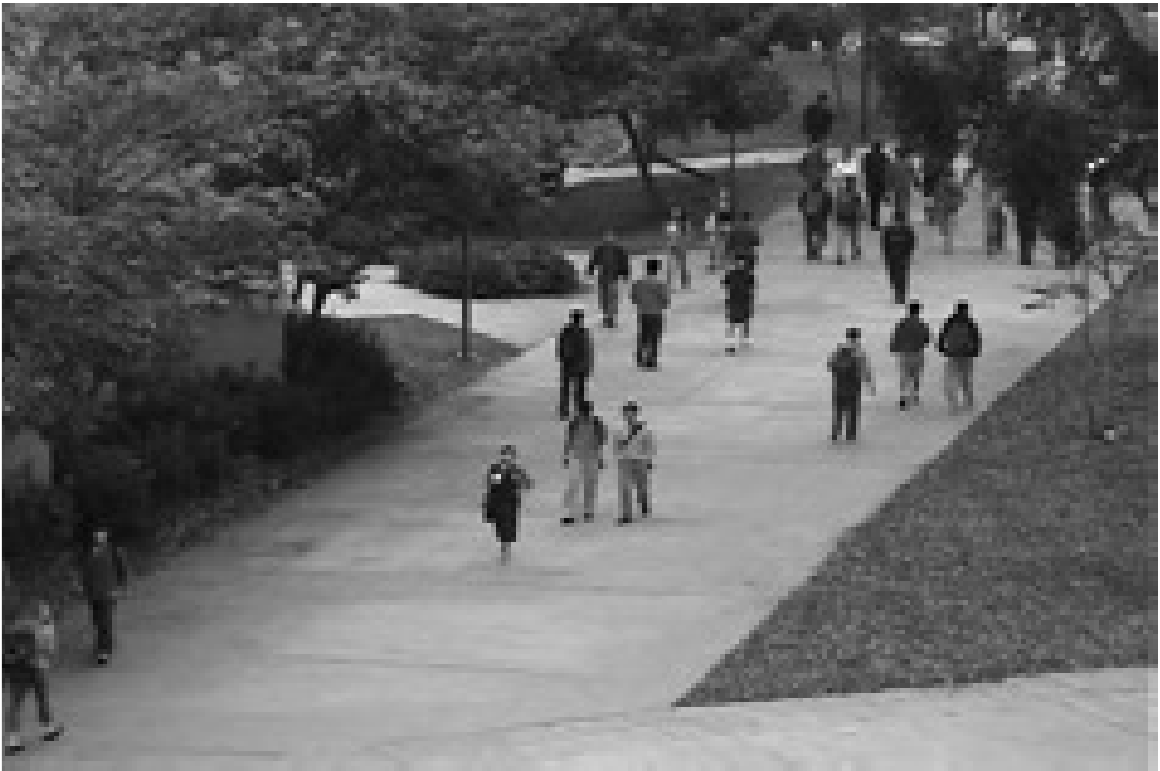}
	\end{subfigure}
	\begin{subfigure}[b]{0.083\textwidth}
		\includegraphics[width=\linewidth, height=1.5cm]{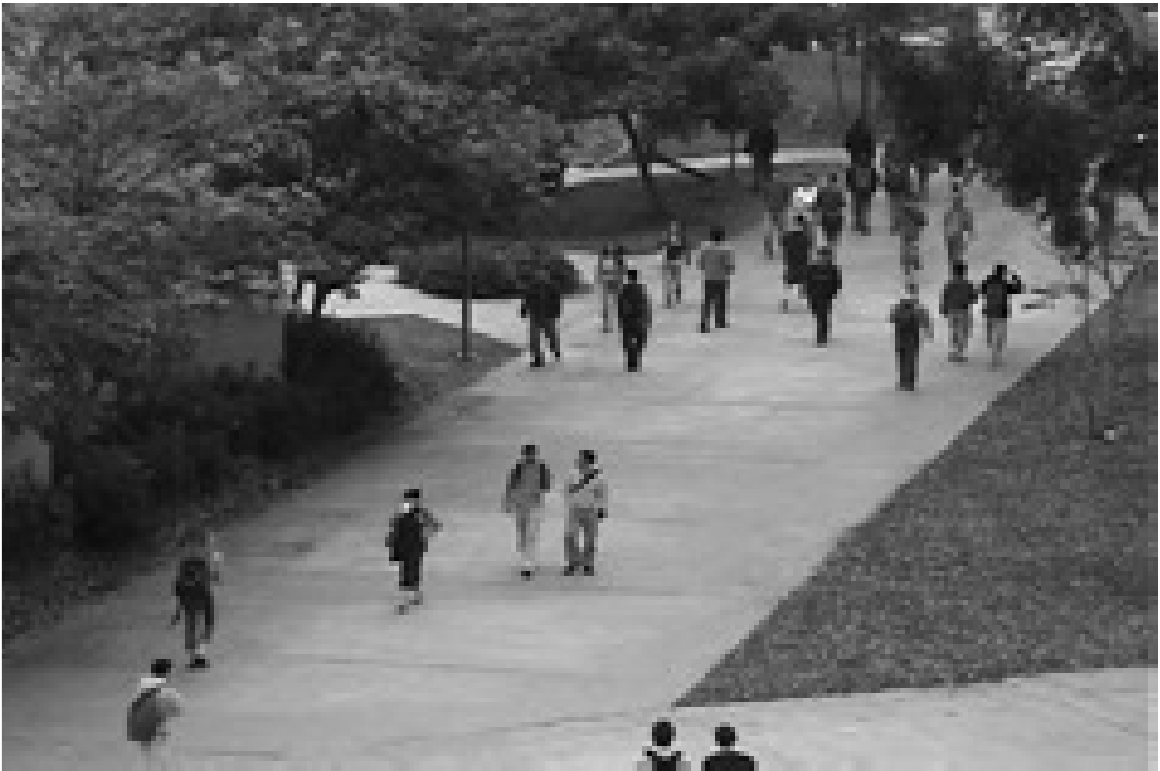}
	\end{subfigure}
	\begin{subfigure}[b]{0.083\textwidth}
		\setlength{\fboxsep}{0pt}%
		\setlength{\fboxrule}{1pt}%
		\cfbox{red}{\includegraphics[width=\linewidth, height=1.5cm]{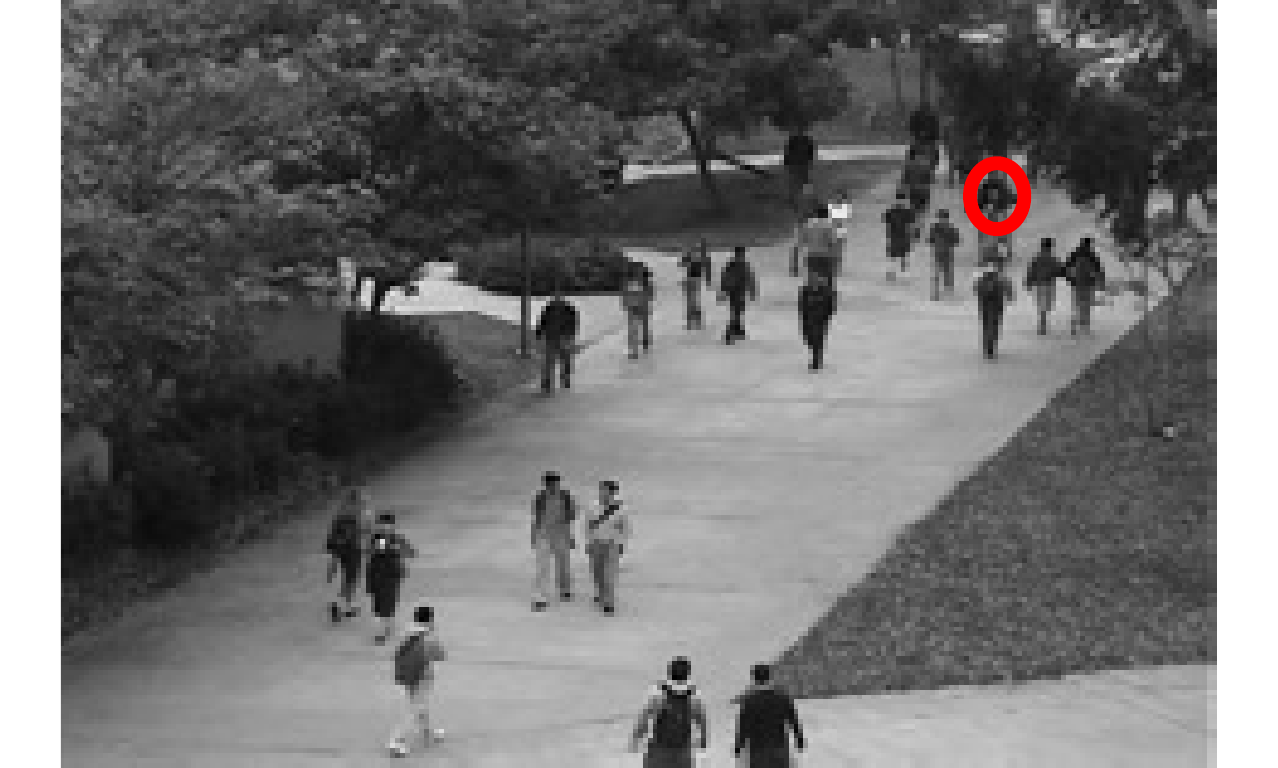}}
	\end{subfigure}
	\begin{subfigure}[b]{0.083\textwidth}
		\setlength{\fboxsep}{0pt}%
		\setlength{\fboxrule}{1pt}%
		\cfbox{red}{\includegraphics[width=\linewidth, height=1.5cm]{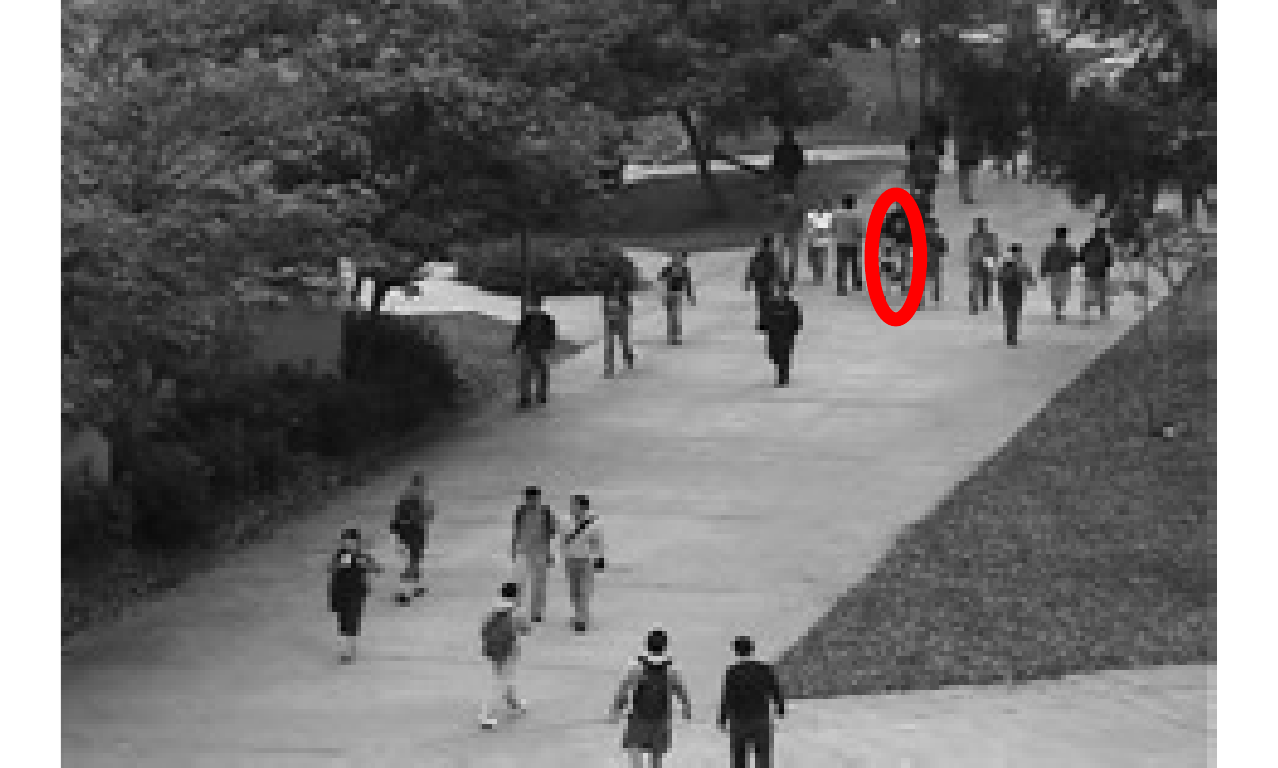}}
	\end{subfigure}
	\begin{subfigure}[b]{0.083\textwidth}
		\setlength{\fboxsep}{0pt}%
		\setlength{\fboxrule}{1pt}%
		\cfbox{red}{\includegraphics[width=\linewidth, height=1.5cm]{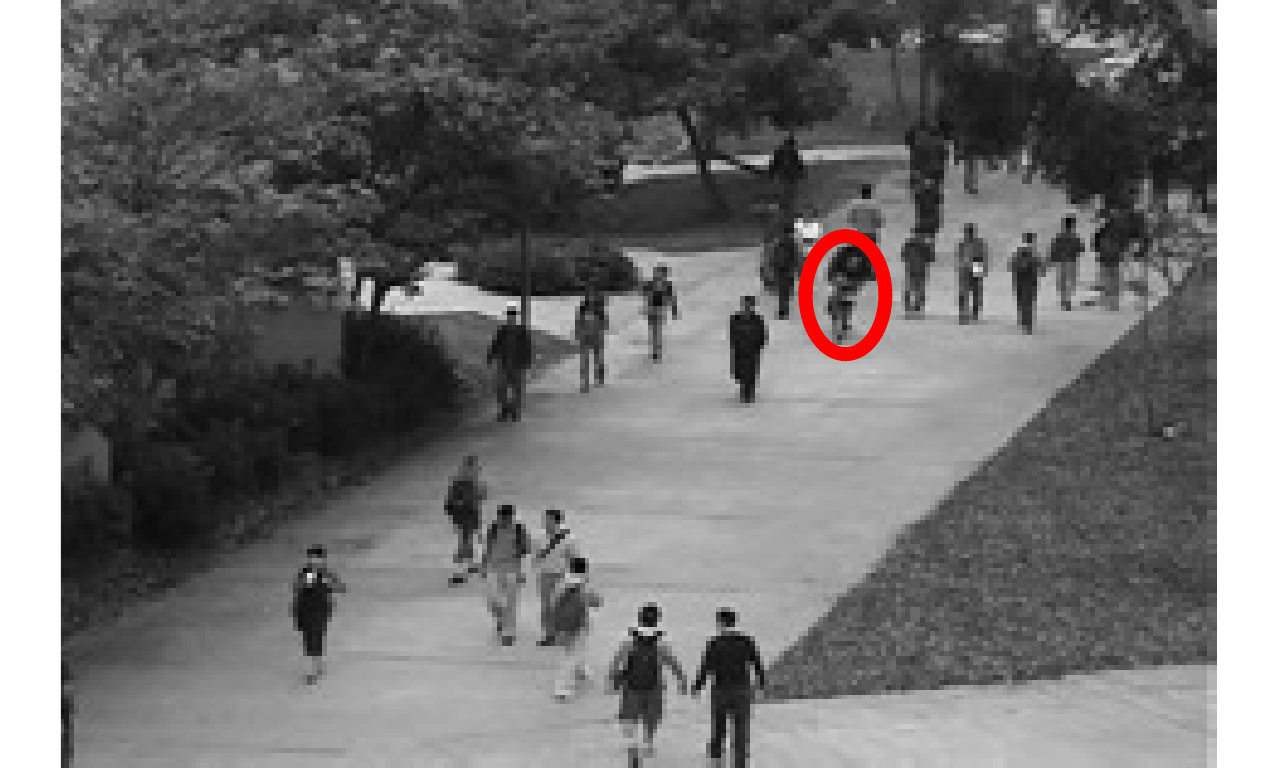}}
	\end{subfigure}
	\begin{subfigure}[b]{0.083\textwidth}
		\setlength{\fboxsep}{0pt}%
		\setlength{\fboxrule}{1pt}%
		\cfbox{red}{\includegraphics[width=\linewidth, height=1.5cm]{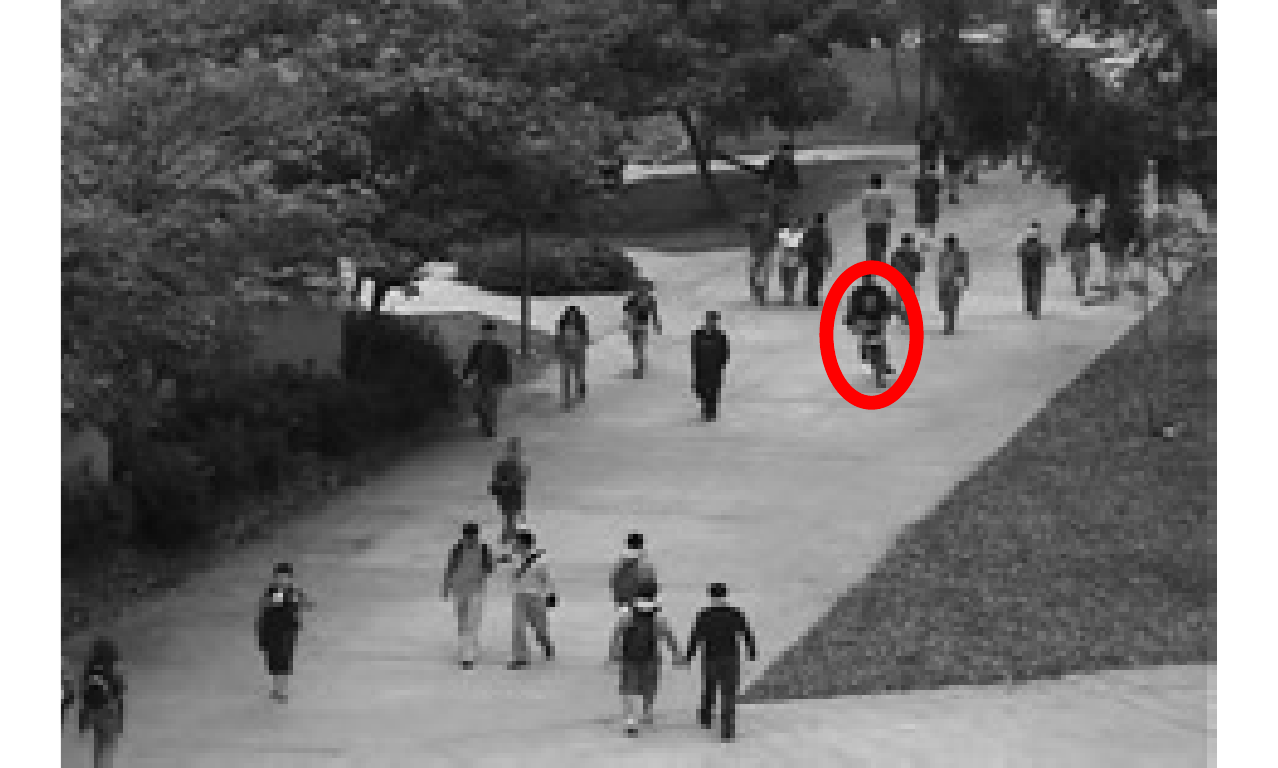}}
	\end{subfigure}
	\begin{subfigure}[b]{0.083\textwidth}
		\setlength{\fboxsep}{0pt}%
		\setlength{\fboxrule}{1pt}%
		\cfbox{red}{\includegraphics[width=\linewidth, height=1.5cm]{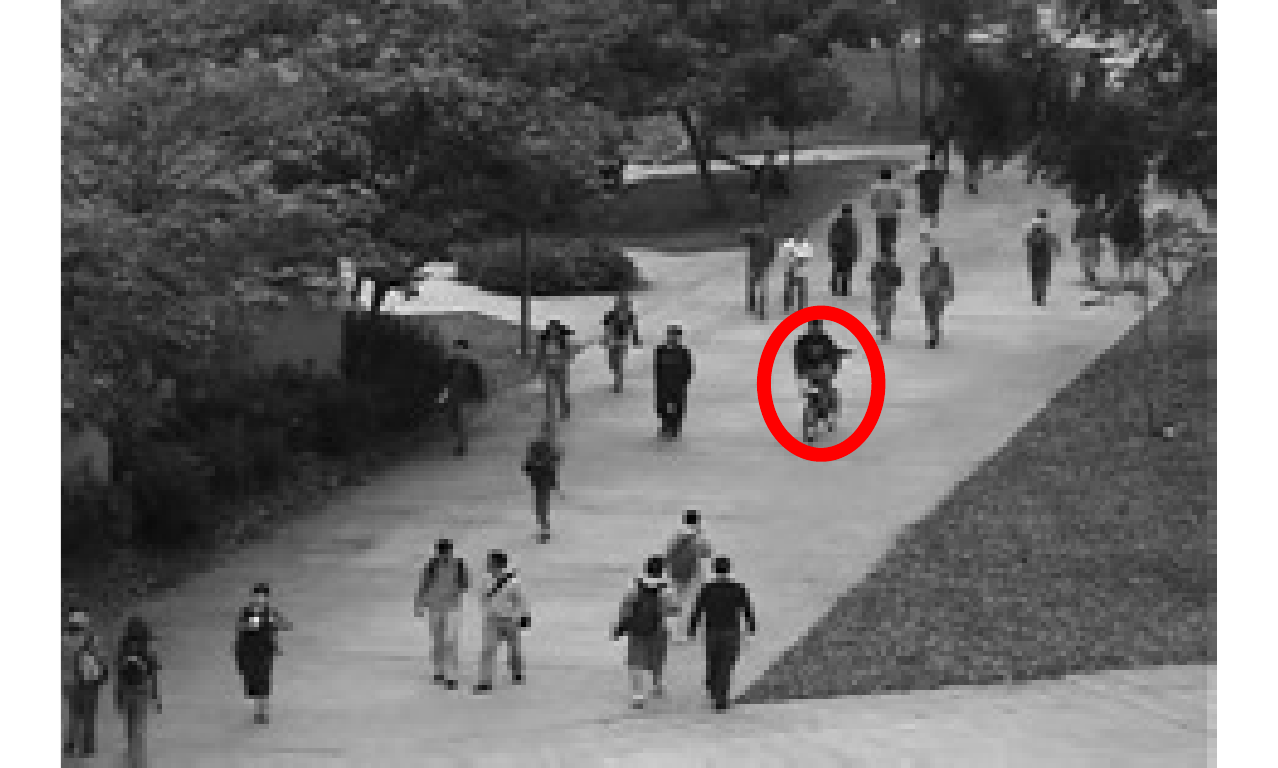}}
	\end{subfigure}
	\begin{subfigure}[b]{0.083\textwidth}
		\setlength{\fboxsep}{0pt}%
		\setlength{\fboxrule}{1pt}%
		\cfbox{red}{\includegraphics[width=\linewidth, height=1.5cm]{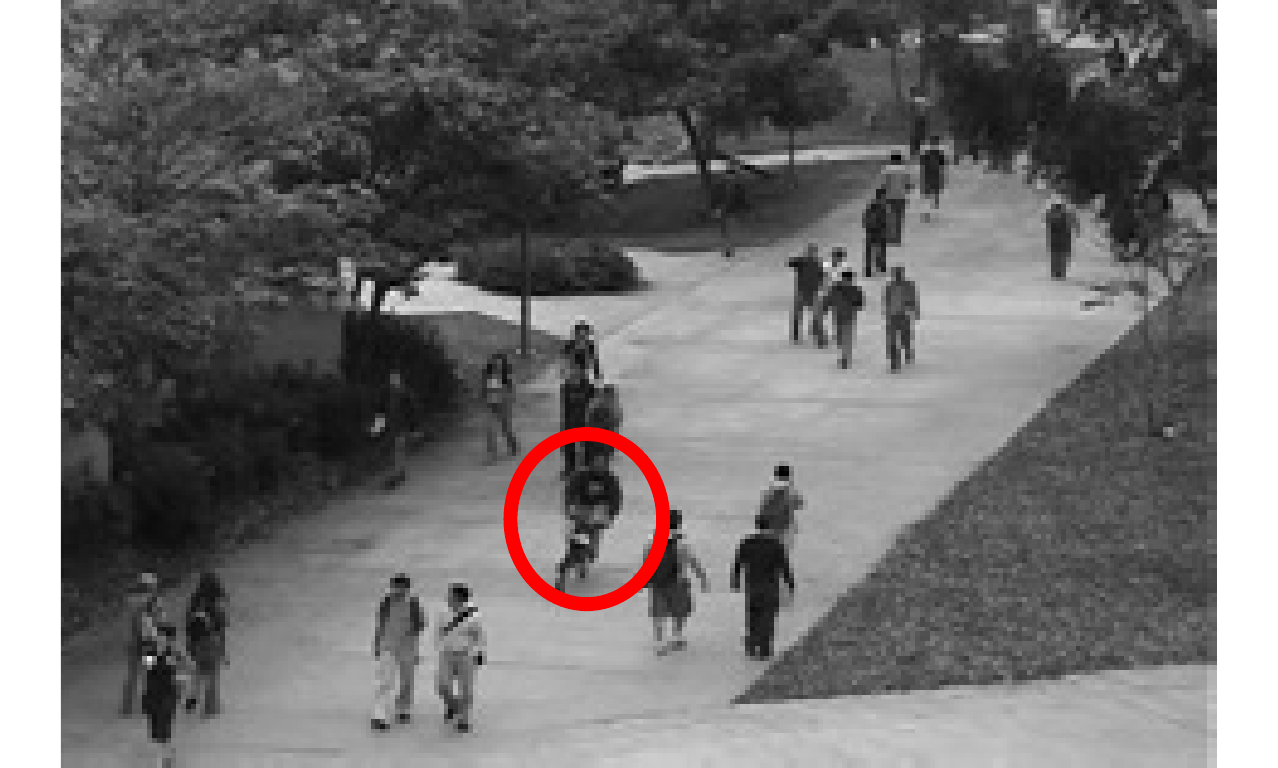}}
	\end{subfigure}
	\begin{subfigure}[b]{0.083\textwidth}
		\includegraphics[width=\linewidth, height=1.5cm]{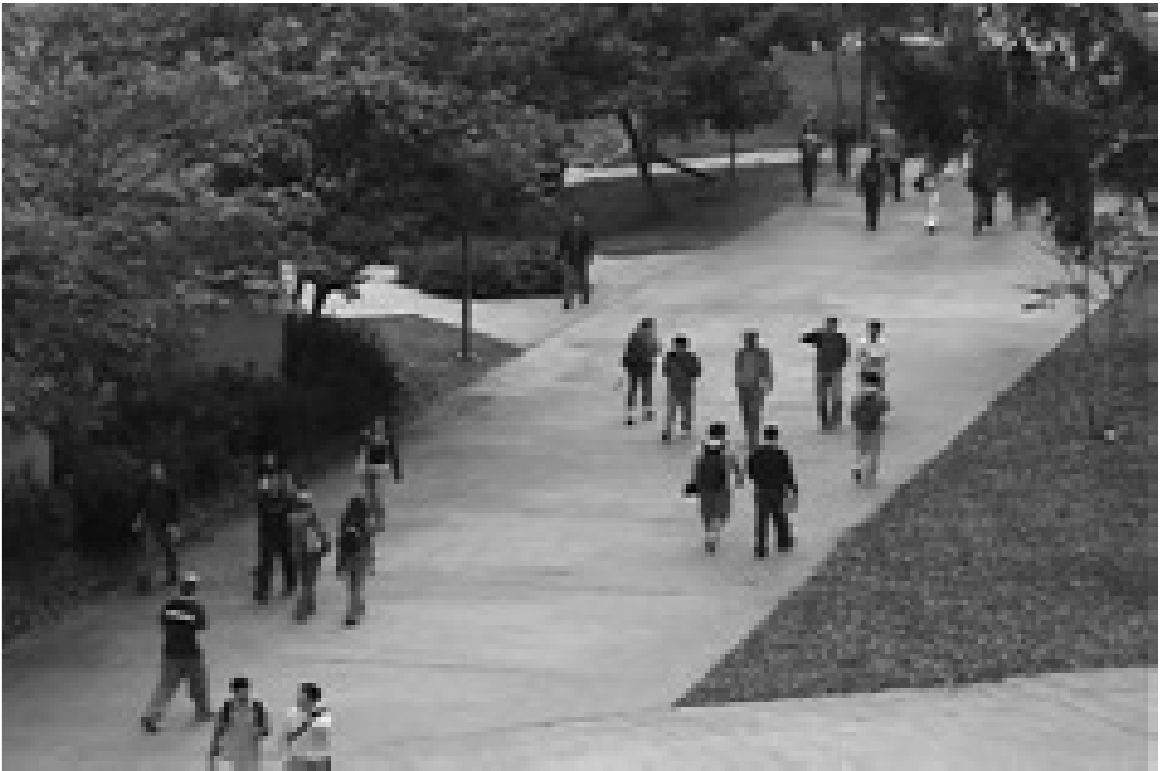}
	\end{subfigure}
	\begin{subfigure}[b]{0.083\textwidth}
		\includegraphics[width=\linewidth, height=1.5cm]{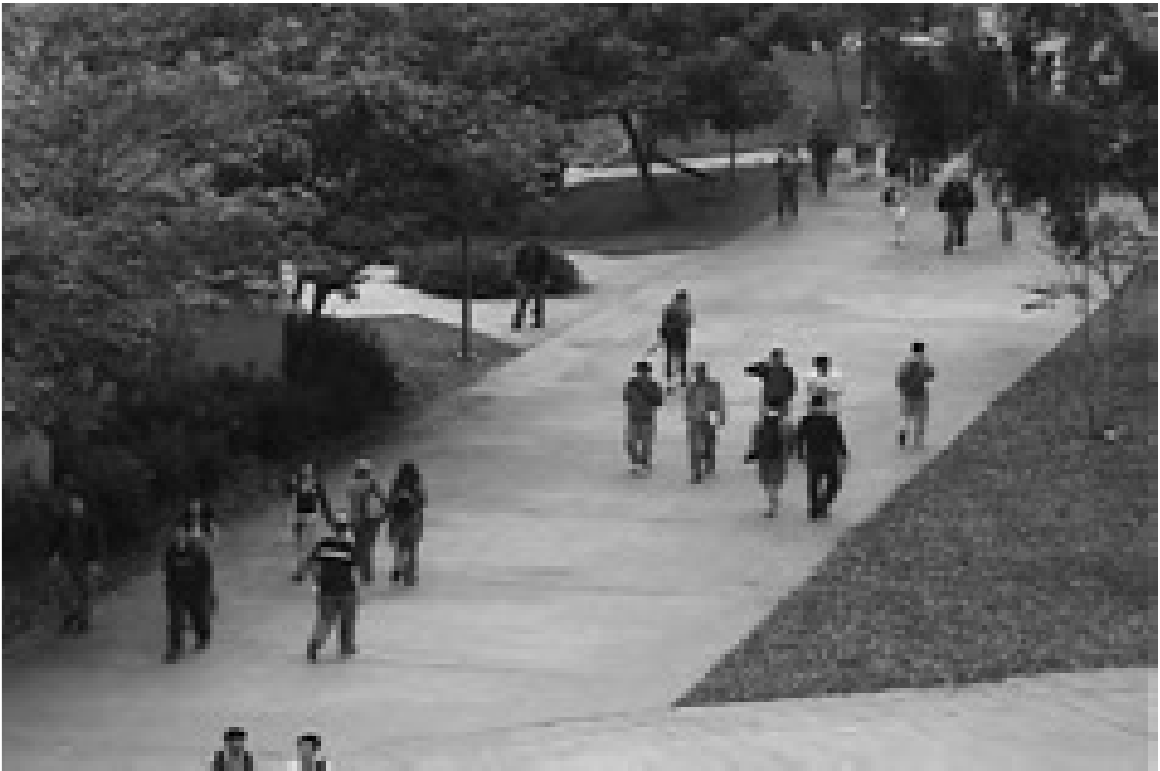}
	\end{subfigure}
	\caption{Pedestrian video 1 - Highlighted frames have the cyclist marked in red, and are detected as outliers by ROMA\_N and CoP }
	\label{fped1}
\end{figure*} 
\begin{figure*}[h]
	\begin{subfigure}[b]{0.083\textwidth}
		\includegraphics[width=\linewidth, height=1.5cm]{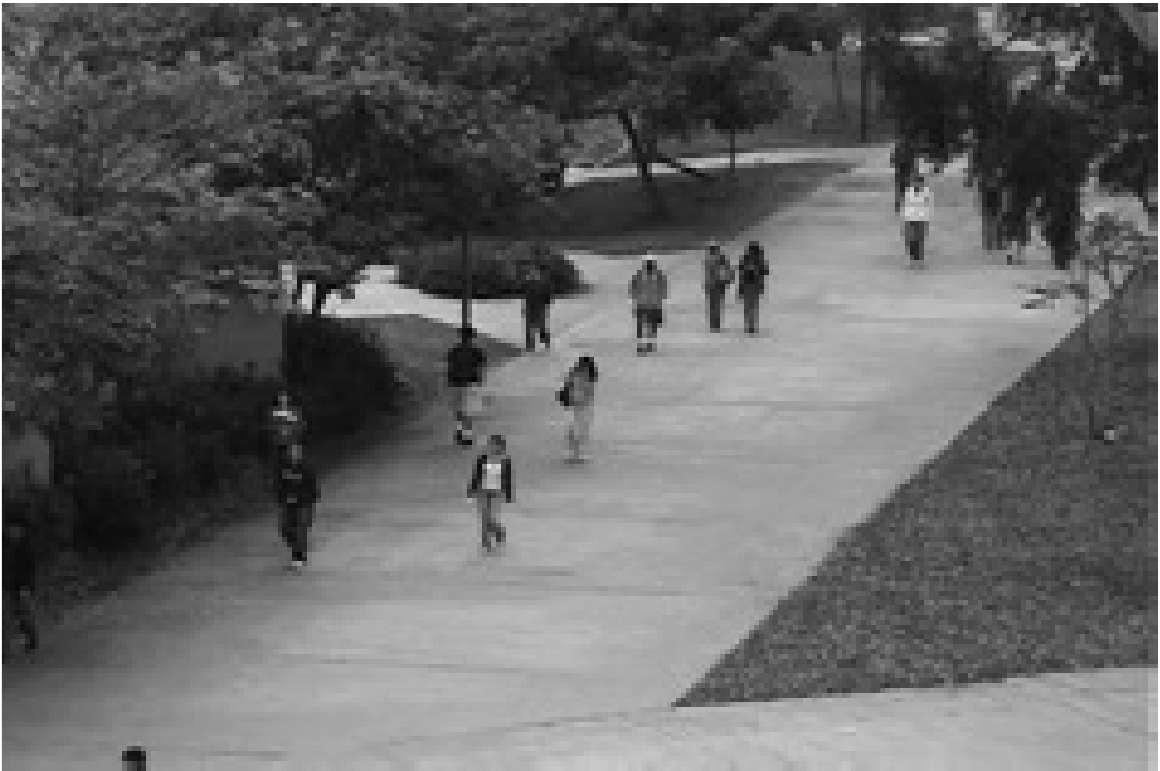}
	\end{subfigure}
	\begin{subfigure}[b]{0.083\textwidth}
		\includegraphics[width=\linewidth, height=1.5cm]{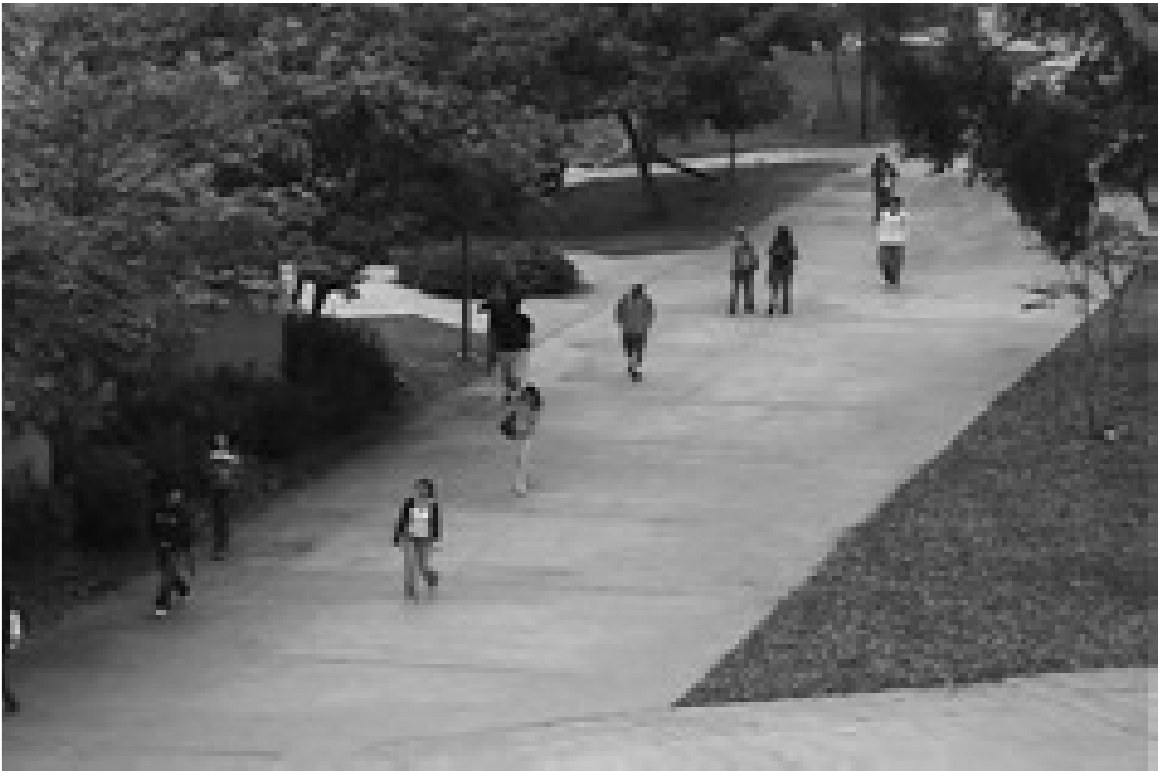}
	\end{subfigure}
	\begin{subfigure}[b]{0.083\textwidth}
		\includegraphics[width=\linewidth, height=1.5cm]{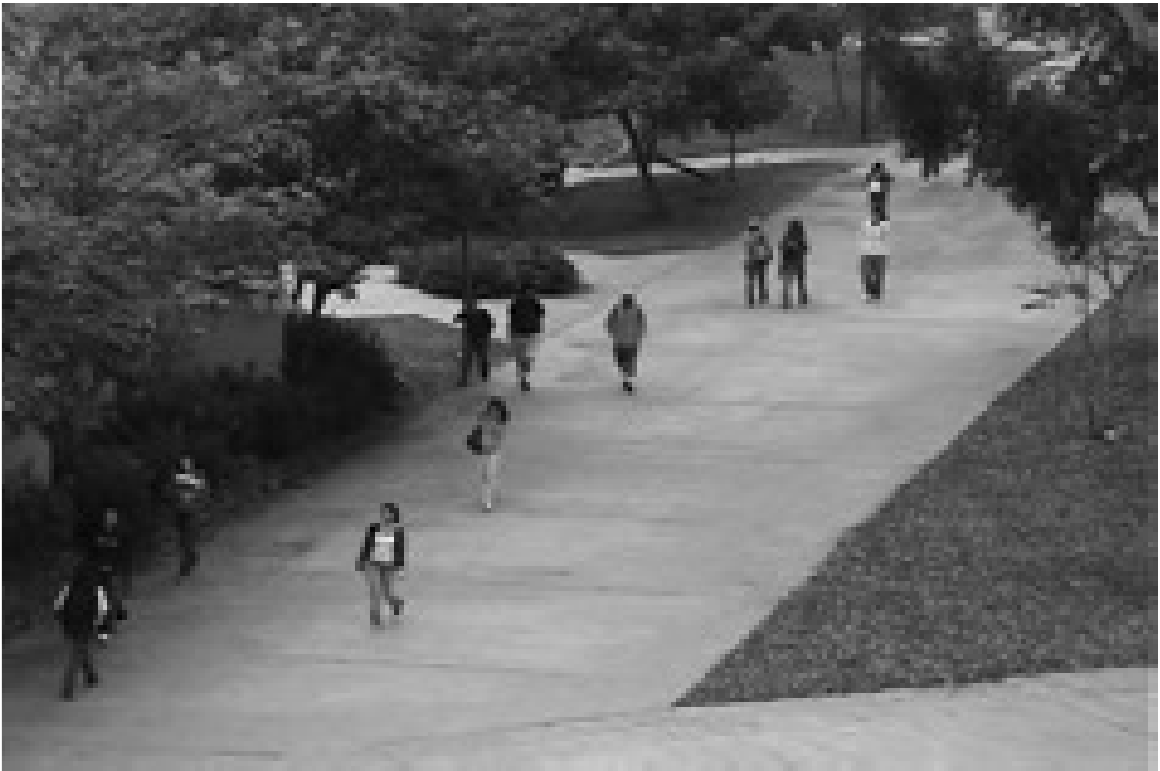}
	\end{subfigure}
	\begin{subfigure}[b]{0.083\textwidth}
		\includegraphics[width=\linewidth, height=1.5cm]{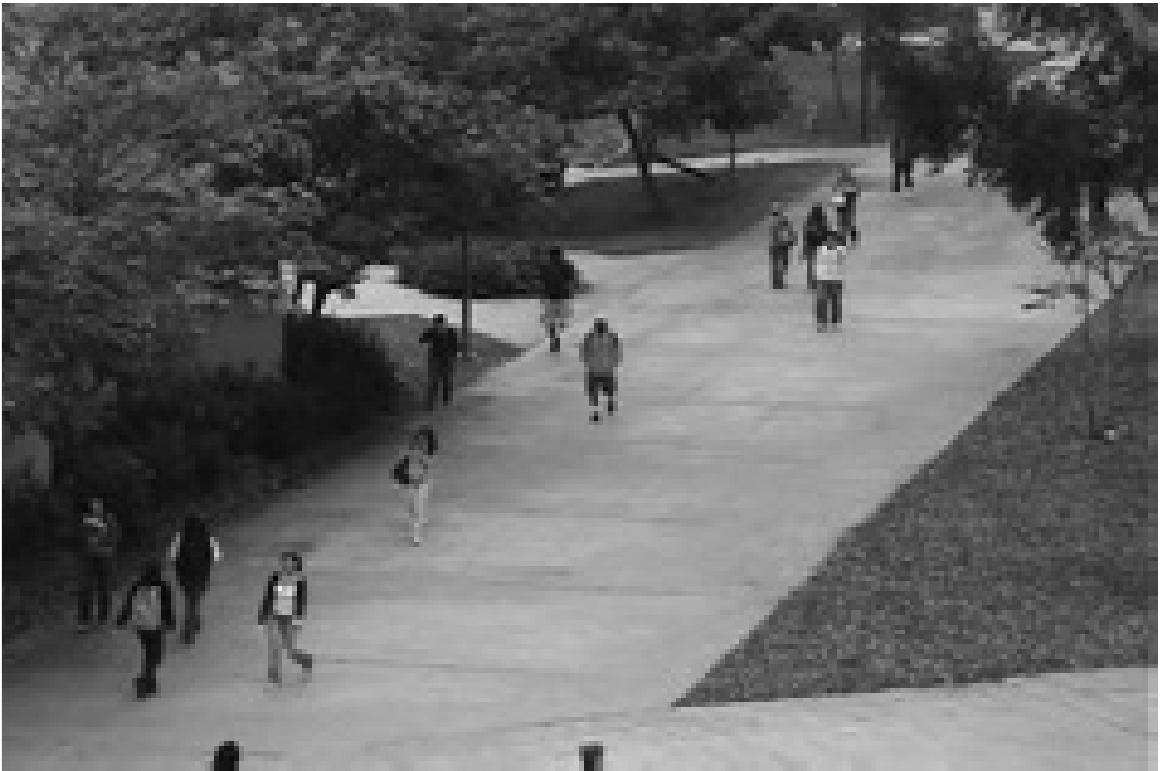}
	\end{subfigure}
	\begin{subfigure}[b]{0.083\textwidth}
		\setlength{\fboxsep}{0pt}%
		\setlength{\fboxrule}{1pt}%
		\cfbox{red}{\includegraphics[width=\linewidth, height=1.5cm]{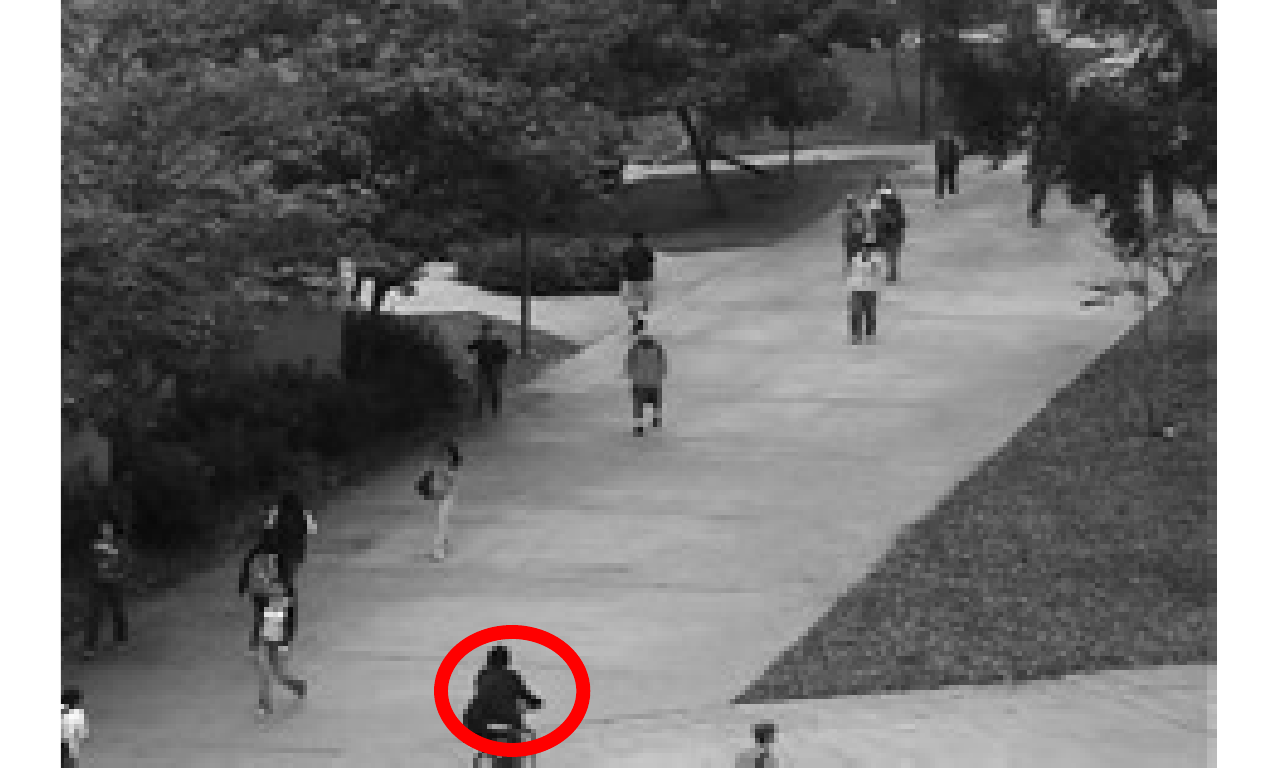}}
	\end{subfigure}
	\begin{subfigure}[b]{0.083\textwidth}
		\setlength{\fboxsep}{0pt}%
		\setlength{\fboxrule}{1pt}%
		\cfbox{red}{\includegraphics[width=\linewidth, height=1.5cm]{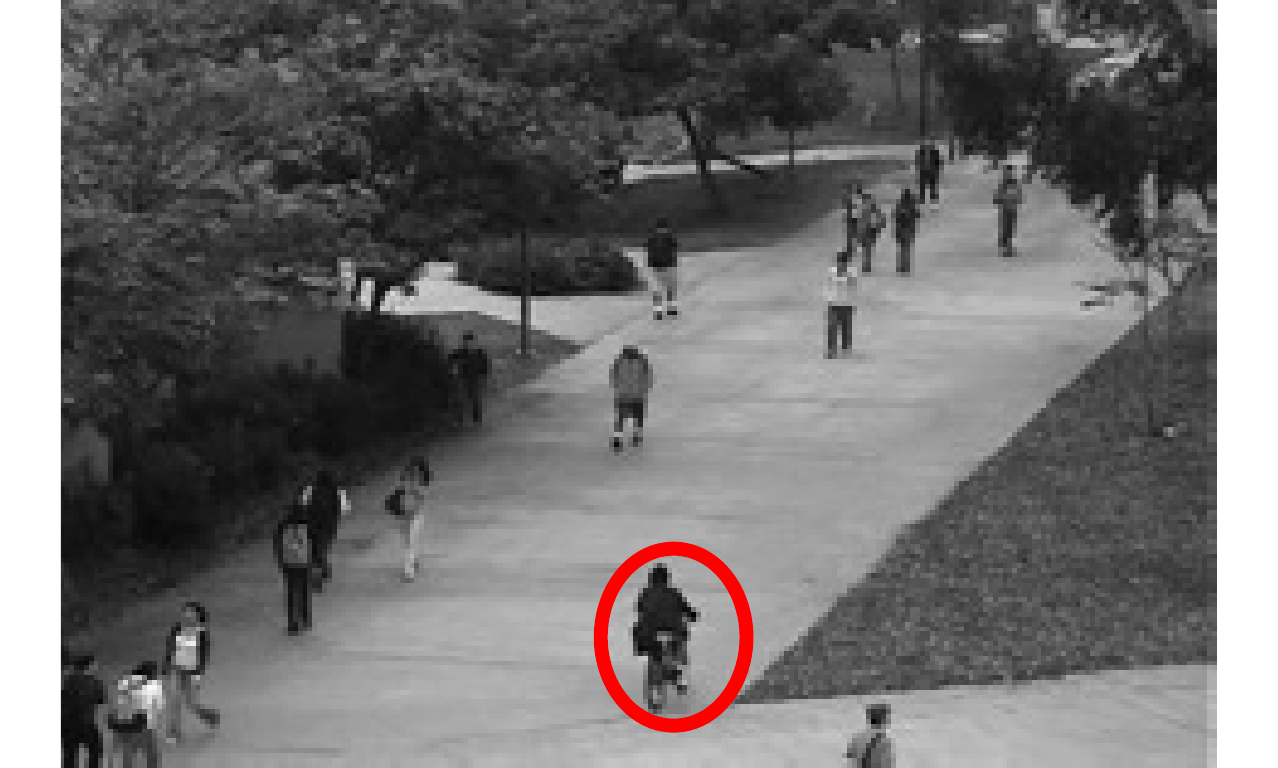}}
	\end{subfigure}
	\begin{subfigure}[b]{0.083\textwidth}
		\setlength{\fboxsep}{0pt}%
		\setlength{\fboxrule}{1pt}%
		\cfbox{red}{\includegraphics[width=\linewidth, height=1.5cm]{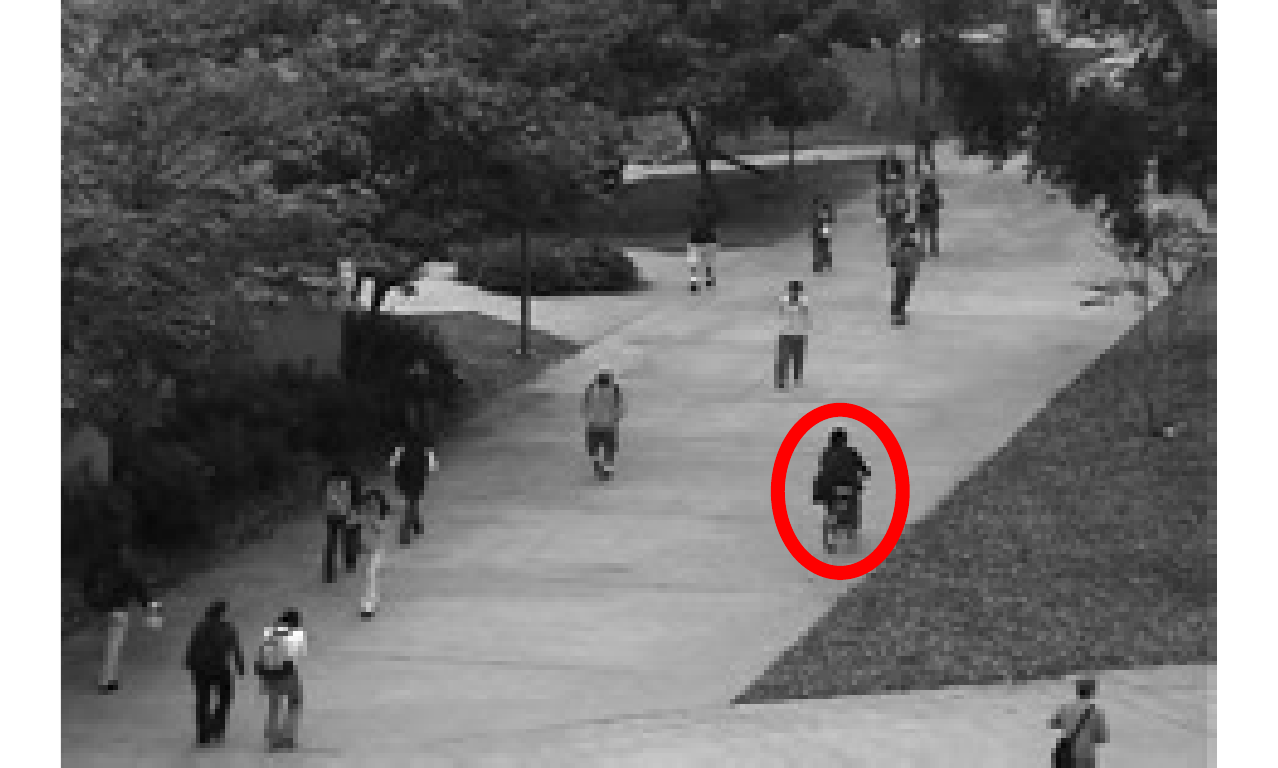}}
	\end{subfigure}
	\begin{subfigure}[b]{0.083\textwidth}
		\setlength{\fboxsep}{0pt}%
		\setlength{\fboxrule}{1pt}%
		\cfbox{red}{\includegraphics[width=\linewidth, height=1.5cm]{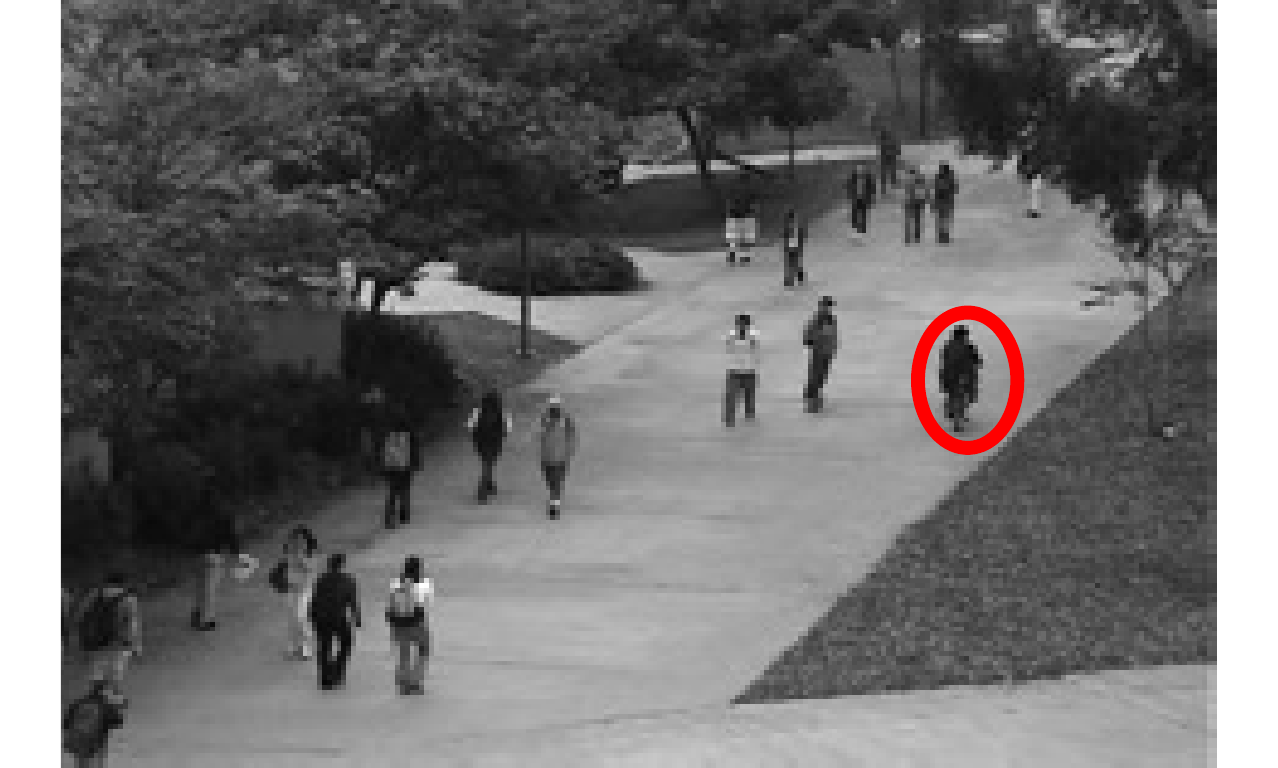}}
	\end{subfigure}
	\begin{subfigure}[b]{0.083\textwidth}
		\setlength{\fboxsep}{0pt}%
		\setlength{\fboxrule}{1pt}%
		\cfbox{red}{\includegraphics[width=\linewidth, height=1.5cm]{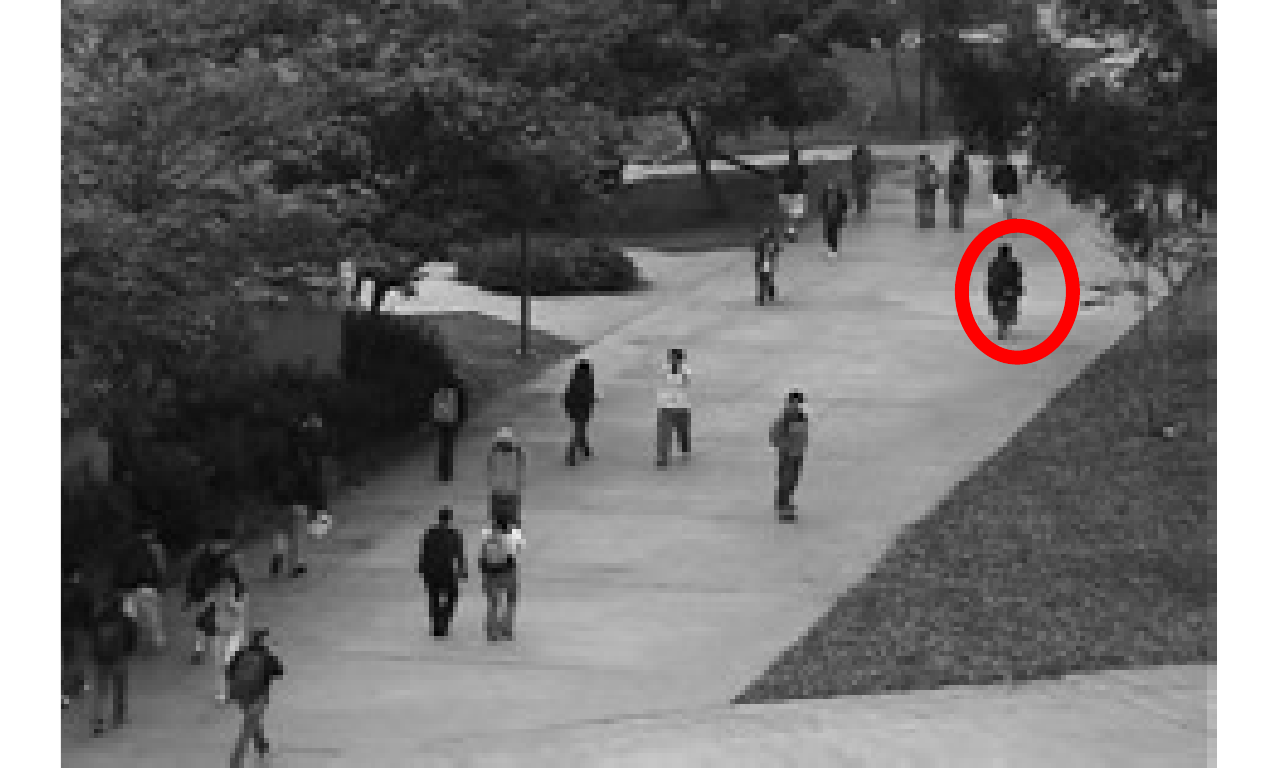}}
	\end{subfigure}
	\begin{subfigure}[b]{0.083\textwidth}
		\includegraphics[width=\linewidth, height=1.5cm]{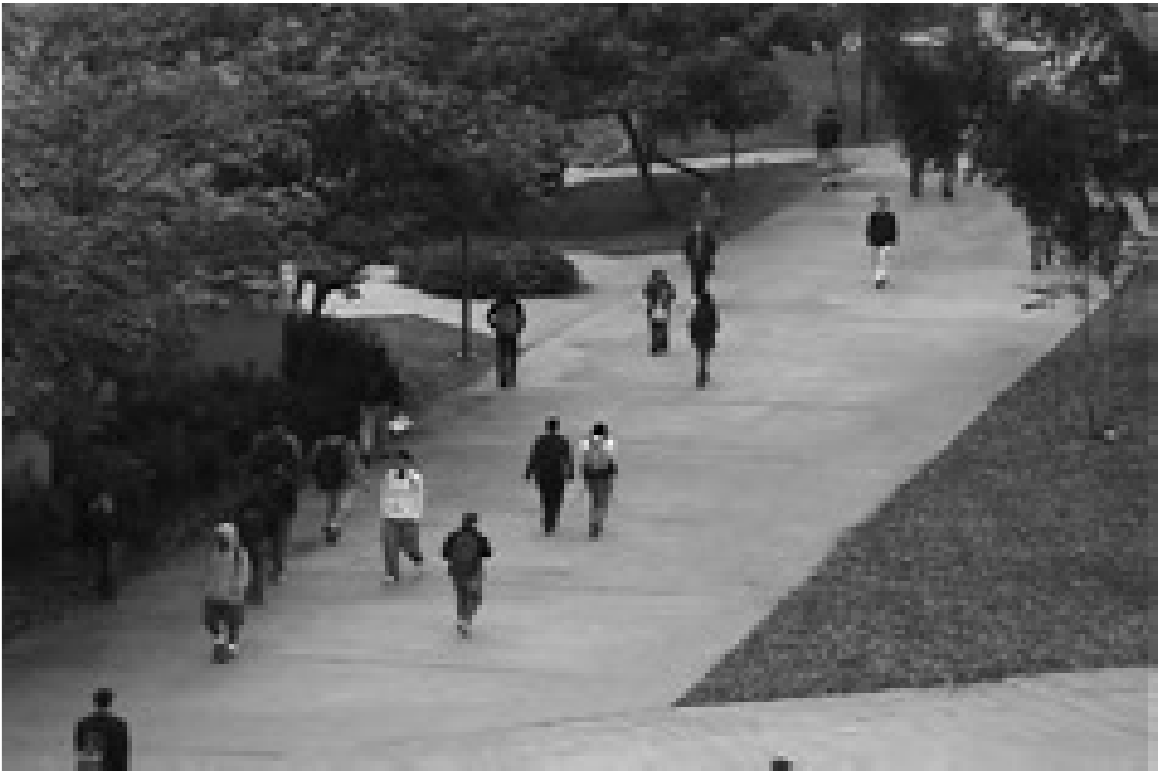}
	\end{subfigure}
	\begin{subfigure}[b]{0.083\textwidth}
		\includegraphics[width=\linewidth, height=1.5cm]{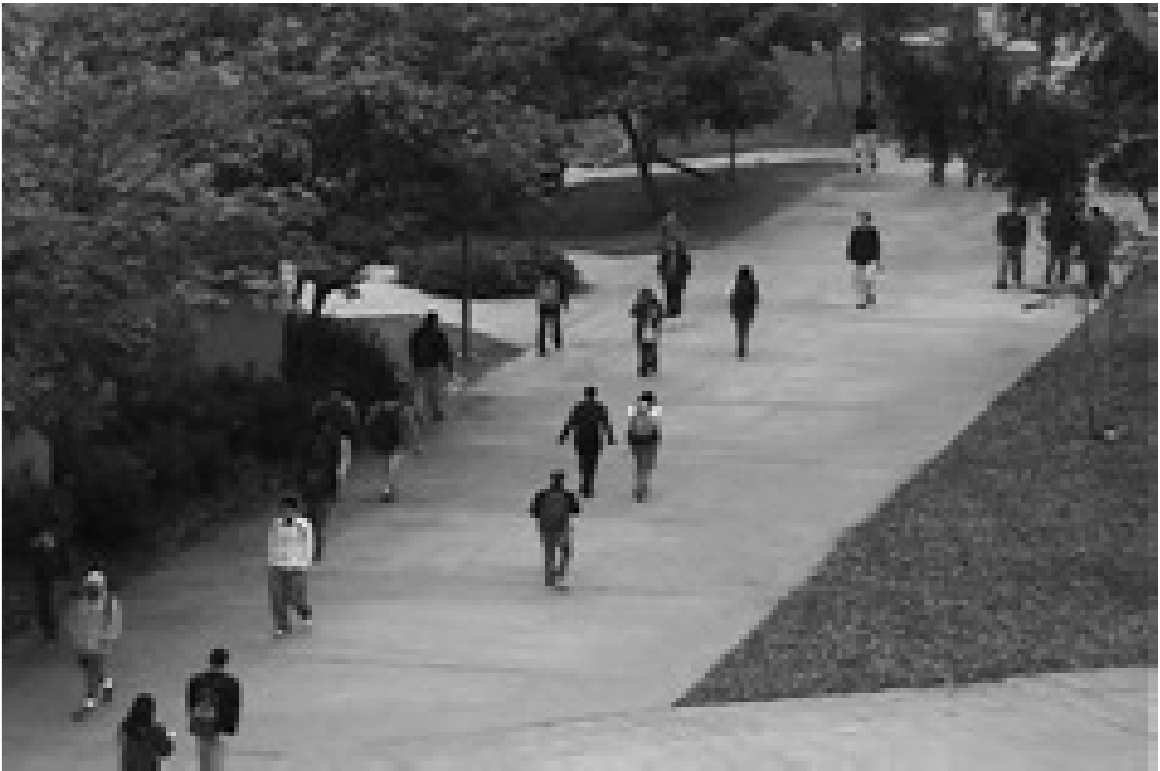}
	\end{subfigure}
	\caption{Pedestrian video 2 - Highlighted frames have the cyclist marked in red, and are detected by only ROMA\_N and not FMS or CoP}
	\label{fped2}
\end{figure*} 
\subsection{Threshold Adaptation}
A lot of the time in real data, the principal angles $\theta_{ij}$'s are not distributed in $[0,\pi]$. They may fall in a different range, say for instance in images, the image pixel values are all positive and hence the angles cannot be greater than $\frac{\pi}{2}$. If we use the theoretical threshold which assumes $\theta_{ij}$'s are distributed in $[0,\pi]$ for this case, a lot of the outliers will be flagged incorrectly as inliers as the threshold here is too high. Hence we propose a scheme of mean adaptation where the term signifying the mean in the threshold, i.e. $\frac{\pi}{2}$ is replaced by the sample mean of all the $\theta_{ij}$'s formed by the data points. This is logical, as the expected value of the principal angles $\theta_{ij}$ formed by randomly chosen inliers and outliers is $\frac{\pi}{2}$ and the sample mean in that case will tend to this value. By using this adaptation technique, we are able to handle a broader set of real data applications. In all the below experiments, threshold was adapted as described above for ROMA\_N and the results are presented.
\subsection{Image separation}
Here we perform an experiment of image clustering, where the algorithms CoP, FMS and ROMA\_N are given a matrix with each vector corresponding to an image coming from different classes. Then the output is checked to see if the algorithms can cluster at least one class out, the rest being outliers. 
\begin{table}[h]
	\caption{MNIST Image separation}
	\label{table4}
	\begin{tabularx}{\linewidth}{@{}l*{10}{C}c@{}}
		\hline
		Algorithm &\multicolumn{2}{c|}{Inlier - digit 0}&\multicolumn{2}{c}{Inlier - digit 1}\\
		&\multicolumn{1}{c|}{IR \%}&\multicolumn{1}{c|}{FID}&\multicolumn{1}{c|}{IR \%}&\multicolumn{1}{c}{FID}\\
		\hline
		FMS&34.48&\multicolumn{1}{c|}{0.5}&76.74&0.2\\
		CoP&38.53&\multicolumn{1}{c|}{0.7}&68.3&23\\
		ROMA\_N&72.34&\multicolumn{1}{c|}{7.8}&86.51&0.6\\
		\hline
	\end{tabularx}
\end{table}
First we do the experiment with MNIST dataset \cite{lecun1998gradient}. We keep $N_{\mathcal{I}} = 1000$ of the digit $0$ images and $N_{\mathcal{O}} = 200$ of other digits in a matrix $\textbf{M} \in \mathbb{R}^{784\text{x}1200}$, where each data point is formed by vectorizing $28\text{x}28$ handwritten image. Hence digit $0$ are the inliers and the outliers come from other digits. We look at inlier recovery percentage (IR\%) which is percentage of inliers recovered by the algorithm and false inlier detection (FID) which is the average number of outliers wrongly flagged as inliers in this experiment. The parameters for FMS and CoP, namely $r$ and $n_s$ are tuned such that it recovers a high number of inliers while discarding the outliers for digit 0. We set $r=50$ for FMS and $r=n_s=150$ for CoP and classify a point as in inlier if $\|\textbf{m}_i-\hat{\textbf{U}}\hat{\textbf{U}}^T\textbf{m}_i\|_2/\|\textbf{m}_i\|_2 \leq 0.2$, where $\hat{\textbf{U}}$ is the estimated subspace given by the algorithm. The same experiment is performed with digit $1$ as inlier and others as outliers with same $N_{\mathcal{I}}$ and $N_{\mathcal{O}}$ values. The digits images were chosen randomly from the dataset for each class in each trial and the results shown in Table \ref{table4} were averaged over 10 trials. FMS has the best FID across the two digits while having a decent IR\%. CoP does well in terms of FID for digit 0 while has a poor FID of 23 for digit 1. ROMA\_N with no parameter tuning has the best IR\% for both digits but the FID for digit 1 is a little high at 7.8. This advantage of not having to set parameters will be evident from the next experiment.

We perform a similar activity but this time with a different dataset - the Caltech 101 image set \cite{fei2007learning}, where we have images from 101 different categories. We use images of airplanes, bikes and cars, resized to $100\text{x}100$ grey scale images, and create $\textbf{M}  \in \mathbb{R}^{10000\text{x}N}$, where each data point is in $\mathbb{R}^{10000}$, after vectorizing the image. We try different combinations of  2 classes at the ratio of $400$ from the first class to $100$ from the other and see if the algorithms were able to cluster at least 1 class out correctly. 
\begin{table}[h]
	\caption{Caltech 101 Image clustering error}
	\label{table2}
	\begin{tabularx}{\linewidth}{@{}l*{10}{C}c@{}}
		\toprule
		Images \& mixture ratio &\multicolumn{1}{c}{FMS}&\multicolumn{1}{c}{CoP}&\multicolumn{1}{c}{ROMA\_N}\\
		\midrule
		Bikes/Cars(400/100)&1&0&0\\
		Airplanes/Cars(400/100)&1&0&0\\
		Airplanes/Bikes(400/100)&1&1&0\\
		\bottomrule
	\end{tabularx}
\end{table}
In this case, one of the classes can be seen as a structured outlier while the other being an inlier. If one of the output cluster of the algorithms, i.e. estimated inliers or outliers contain only images from one class, clustering error was considered 0. If one of the output set was empty the clustering error was considered 1, otherwise the clustering error is computed as the ratio of total number of wrongly clustered points to the total number of points. We use the same parameter setting for FMS and CoP. As can be seen from Table \ref{table2}, using this setting, FMS fails to separate the data in any case, while CoP fails to separate them in one of the cases, each failure case being total failures as one of the output i.e. inlier/outlier estimate of the algorithm was an empty set. ROMA\_N correctly separates out each class in all the cases. This experiment highlights the importance of using such a tuning free algorithm that can adapt to different datasets without the user having to set a value.
\subsection{Video Activity Detection}
The next experiment is detecting activity in a video, which is a structured outlier detection problem \cite{rahmani2016coherence}. We have videos with a background and some activity occurring in some frames like a person entering or leaving. The aim is to detect those frames where the activity occurs. The frames with activity can be viewed as outliers and the background frames are inliers. We compare the performance of ROMA\_N in this context with CoP and FMS in video data from the wallflower paper \cite{wallflower-principles-and-practice-of-background-maintenance}. First we will use the Waving tree video used in \cite{rahmani2016coherence}. We set the parameters for CoP same as in example D.2 in \cite{rahmani2016coherence} while taking $n_s = 30$ samples to build the subspace basis. We use $r=3$ and for FMS and CoP, classify a point as in inlier if $\|\textbf{m}_i-\hat{\textbf{U}}\hat{\textbf{U}}^T\textbf{m}_i\|_2/\|\textbf{m}_i\|_2 \leq 0.2$. The video, where there is a waving tree in the background and a person enters and leaves the frame in between is shown in Fig. \ref{fwavingtree}. The outliers are the frames containing the person. It was seen that ROMA\_N performs equally as well as CoP or FMS in detecting outlier frames highlighted in Fig. \ref{fwavingtree}. Using Heckel's algorithm, i.e. \cite{heckel2015robust} for outlier detection for this activity did not yield a result as the algorithm could not detect any outlier frames confirming that it is only suited to unstructured outliers strictly following th model. We do this activity in the Camouflage video from the same dataset, a similar video where a person enters and leaves the frame in a static background of a computer screen. Outlier frames were detected correctly by CoP and ROMA\_N as shown in Fig. \ref{fcam}, but FMS failed to detect a few outliers as seen in Fig. \ref{fcamsub}. 

We repeat the same experiment in videos from a different dataset - the UCSD anomaly detection dataset \cite{mahadevan2010anomaly}. This dataset contains videos of a pedestrian pathway - a crowded background. An anomaly is when there are cycles, skaters or carts that move along the pathway and the aim is to detect those frames where there are such movements. This is more challenging since the background is dynamic and especially since, here we do not use the training data to train for the background. We simply use two of the testing videos with cycle movement as our input, making the problem unsupervised. Given these frames with no prior training, the aim is to detect those frames where there is a cycle through the pathway, which otherwise contains pedestrians only. We use the same parameter settings that succeeded well for CoP and FMS in the previous dataset. For the first video, CoP and ROMA\_N detects all outlier frames and FMS failed to detect most as shown in Fig. \ref{fped1}. While ROMA\_N identified a large number of background frames (92 out of 107), CoP could identify only 40 inlier frames. In the second video, both FMS and CoP failed to detect any outlier frame while ROMA identified all outlier frames as seen in Fig. \ref{fped2}. Hence the parameter settings which worked well in the wallflower dataset does not work well here, while ROMA\_N by virtue of being parameter free adapts better.
\section{Conclusion}\label{s4}
In this paper a simple, fast, parameter free algorithm for robust PCA was proposed, which does outlier removal without assuming the knowledge of dimension of the underlying subspace and the number of outliers in the system. The performance was analyzed both theoretically and numerically. The importance of this work lies in the parameter free nature of the proposed algorithm since estimating unknown parameters or tuning for free parameters in an algorithm is a cumbersome task. Here the inliers lie in a single low dimensional subspace and unstructured outliers, structured outliers and a mixture of both were considered.
\appendices
\section{Some useful Lemmas, proofs of Lemmas \ref{lthetao} and \ref{lthetai}}\label{app1}
\begin{lemma}[Lemma 12 from \cite{cai2013distributions}]\label{l1}
	Let $\mathbf{x}_1, \mathbf{x}_2... \in \mathbb{S}^{n-1}$ be random points independently chosen with uniform distribution in $\mathbb{S}^{n-1}$, and let $\theta_{ij}$ be defined as in (\ref{etheta}), then pdf of $\theta_{ij}$ is given by:
	\begin{equation}\label{e4}
	h(\theta) = \dfrac{1}{\sqrt{\pi}}\dfrac{\Gamma(\frac{n}{2})}{\Gamma(\frac{n-1}{2})} (sin\theta)^{n-2} \hskip20pt \theta \in [0, \pi]
	\end{equation}
\end{lemma}
Lemma \ref{l1} implies that the angles $\theta_{ij}$ are identically distributed $\forall i, j$, $i\neq j$ with the pdf $h(\theta)$. The expectation of this distribution is $\frac{\pi}{2}$ and the angles concentrate around $\frac{\pi}{2}$ as $n$ grows. Using the results in \cite{cai2013distributions}, we will state the following:
\begin{remark}\label{r1}
	$h(\theta)$ can be approximated by the pdf of Gaussian distribution with mean $\frac{\pi}{2}$ and variance $\frac{1}{n-2}$, for higher dimensions specifically for $n \geq 5$. In fact $\theta_{ij}$ converges weakly in distribution to $\mathcal{N}(\frac{\pi}{2},\frac{1}{n-2})$ as $n \to \infty$. 
\end{remark}
The remark has been validated in \cite{cai2013distributions}.
\begin{lemma}\label{lvirtual}
The angle between two points $i$ and $j$, $\theta_{ij}$ have the same statistical properties as the angle between two points chosen uniformly at random from $\mathbb{S}^{n-1}$, for all of the following cases:
	\begin{itemize}
		\item[a)] Under Assumption 1 for inliers and outliers when either $i$ or $j \in \mathcal{O}$.
		\item[b)] Under Assumption 2 for outliers and Assumption 1 or Assumption 3 for inliers when $i \in \mathcal{I}$ and $j \in \mathcal{O}$ or vice versa.
	\end{itemize} 
\end{lemma}
 \begin{proof}
 		Part a) - By using Assumption \ref{amain}, we set the outliers to be chosen uniformly at random from all the points in $\mathbb{S}^{n-1}$ and the subspace $\mathcal{U}$ is also chosen uniformly at random. Hence to an outlier, all the other points are just a part of a set of uniformly chosen independent points in $\mathbb{S}^{n-1}$.\\
 		Part b) - When the inliers follow Assumption 1 or Assumption 3, the subspace $\mathcal{U}$ is chosen uniformly at random. Even though the outliers are structured, under Assumption 2, when one looks at two points, one from $\mathcal{O}$ and the other from $\mathcal{I}$, they are like uniformly chosen independent points in $\mathbb{S}^{n-1}$. Hence all the angles between an inlier and outlier have the same statistical properties as the angle between two points chosen uniformly at random from $\mathbb{S}^{n-1}$.	
 \end{proof}
Using the above results, we can prove Lemmas \ref{lthetao} and \ref{lthetai}.
\begin{proof}[\textbf{Proof of Lemma \ref{lthetao}}]
 The results follow directly from Lemma \ref{l1}, Remark \ref{r1} and Lemma \ref{lvirtual}.\\
\end{proof}
\begin{proof}[\textbf{Proof of Lemma \ref{lthetai}}]
	Assuming that $r$ is also large enough i.e. $r >5$, the points within the subspace are uniformly chosen points in the hypersphere $\mathbb{S}^{r-1}$ and their distribution is as in equation (\ref{e4}) with $n$ replaced by $r$. Hence from Lemma \ref{l1}, Remark \ref{r1} and Assumption \ref{amain}, the lemma is proved.
\end{proof}
\section{Properties of $\phi_{ij}$}\label{app3}
Here, we will look at $\phi_{ij}$ as defined in (\ref{ephi}).
We will use the Gaussian approximation of the pdf $h(.)$ and obtain the approximate mean and variance values for $\phi_{ij}$ using the following lemma. 
\begin{lemma}\label{l3}
	Let $U \sim \mathcal{N}(\mu, \sigma^2)$, define a random variable $V$ as:
	\begin{align*}
	V=\begin{cases}
	U & \text{for } U\leq \mu\\
	2\mu-U & \text{for } U > \mu\\
	\end{cases}
	\end{align*}
	The expectation and variance of $V$ are given by $\mathbb{E}(V) = \mu-\sqrt{\frac{2}{\pi}}\sigma$ and $var(V) = \sigma^2(1-\frac{2}{\pi})$. 
	Also $V > \mu-c\sigma$ w.p $ 2F_{\mathcal{N}}(c)-1$.
\end{lemma}
\begin{proof}The cdf of $V$ for $v \leq \mu$ is given by,
	\begin{align*}
	F_V(v) &= \mathbb{P}(V \leq v) \\&= \mathbb{P}(\{U \leq v|U\leq \mu \}\cup \{2\mu-U \leq v| U> \mu\})\\
	&=\mathbb{P}(U \leq v) + 1- \mathbb{P}(U \leq 2\mu-v) 
	\end{align*}
	The conditioning vanishes because $v \leq \mu$. The cdf is 1 when $v > \mu$. The pdf is given after differentiation. Note that since the pdf of $U$ is symmetric around $\mu$, $f_U(2\mu-v) = f_U(v)$. Thus $f_V(v) = 2f_U(u)$ for $U\leq \mu$ and $0$ otherwise.The moment generating function of $V$, $M_V(t)$ is thus given by 
	\begin{align*}
	M_V(t) &= \mathbb{E}(e^{Vt})
	= \int\limits_{-\infty}^{\mu}e^{vt}\dfrac{2}{\sqrt{2\pi}\sigma}e^{-\frac{(v-\mu)^2}{2\sigma^2}}dv \\
	&=2e^{\mu t+\frac{\sigma^2t^2}{2}}F_{\mathcal{N}}(-\sigma t) \hskip10pt
	\end{align*}
	where $F_{\mathcal{N}}(.)$ is the standard normal cdf. The last step was using a change of variable in integration $z = \frac{v-\mu}{\sigma} -\sigma t$ and definition of $F_{\mathcal{N}}(.)$. Using MGF, we can easily derive the moments of $V$ as $\mathbb{E}(V) = \Big[\dfrac{d M_V(t) }{dt}\Big]_{t=0} = \mu-\sqrt{\dfrac{2}{\pi}}\sigma$ and $\mathbb{E}(V^2)  = \Big[\dfrac{d^2 M_V(t) }{dt^2}\Big]_{t=0}= \sigma^2+\mu^2-2\mu\sigma\sqrt{\dfrac{2}{\pi}}$, using the result that $\dfrac{dF_{\mathcal{N}}(\sigma t)}{dt} = -\frac{\sigma}{\sqrt{2\pi}}e^{-\frac{\sigma^2 t^2}{2}}$. This also gives the variance as $var(V) = \mathbb{E}(V^2)-(\mathbb{E}(V))^2 = \sigma^2(1-\dfrac{2}{\pi})$.
	\begin{align*}
	    \mathbb{P}(V \leq \mu-k\sigma) &= \int\limits_{-\infty}^{\mu_V-k\sigma_V}2f_U(v)dv =2F_{\mathcal{N}}(-k)
	\end{align*}
	Hence $V > \mu-c\sigma$ w.p $1-2F_{\mathcal{N}}(-c) = 2F_{\mathcal{N}}(c)-1$. 
\end{proof}
\begin{corollary}\label{c2}
	 When $\mathbf{x}_i, \mathbf{x}_j$ are two points chosen uniformly at random from $\mathbb{S}^{n-1}$, $\mathbb{E}(\phi_{ij}) \approx \dfrac{\pi}{2} - \sqrt{\dfrac{2}{\pi(n-2)}}$ and $var(\phi_{ij}) \approx  \dfrac{1-\frac{2}{\pi}}{n-2}$ and $\phi_{ij}>\dfrac{\pi}{2} - \dfrac{c}{\sqrt{n-2}}$ with probability $2F_{\mathcal{N}}(c) -1$. 
\end{corollary}
\begin{proof}
	Using Lemma \ref{l3}, assuming the Gaussian approximation of $\theta_{ij}$ and identifying $U=\theta_{ij}$ and $V = \phi_{ij}$, this result is obtained.
\end{proof}
\begin{proof}[\textbf{Proof of Lemma \ref{lphilower}}]
	The result follow from the arguments in Lemma \ref{lvirtual} in Appendix \ref{app1}, noting that $\phi_{ij}$ follows $\theta_{ij}$ and from the above corollary. 
\end{proof}
\section{Proofs of results in section \ref{sguaran}}\label{app2}
To prove the results, the following Lemma is first proved:
\begin{lemma}\label{luseful}
	Under Assumption 1, for $i \in \mathcal{I}$:
	\begin{equation}\label{eqibnd}
	\begin{aligned}
		(N_{\mathcal{I}}-2)p_{\mathcal{I}}^2-(N_\mathcal{I}-3)p_{\mathcal{I}}\leq \mathbb{P}(q_i>\zeta) \leq p_{\mathcal{I}}^2,
	\end{aligned}
	\end{equation}
	where $p_{\mathcal{I}} =  \Big(2F_{\mathcal{N}}\Big(C_N\sqrt{\frac{r-2}{n-2}}\Big)-1\Big)$
\end{lemma} 
\begin{proof}
Let $i \in \mathcal{I}$. We will first look at the probability $\mathbb{P}(\phi_{ij}>\zeta)$, when $j \in \mathcal{I}$. Since under assumption 1, the inliers are selected uniformly at random from $\mathbb{S}^{r-1}$, we can apply corollary \ref{c2} on $\phi_{ij}$, which gives the below for $i,j \in \mathcal{I}$:
\begin{equation}\label{ephiin}
\begin{aligned}
	\phi_{ij}>\dfrac{\pi}{2} - &\dfrac{c_1}{\sqrt{r-2}}\hskip20pt w.p\text{	}2F_{\mathcal{N}}(c_1) -1\\
	\Rightarrow \mathbb{P}(\phi_{ij}>\zeta)  = p_{\mathcal{I}}&= 2F_{\mathcal{N}}\Big(C_N\sqrt{\frac{r-2}{n-2}}\Big)-1
\end{aligned}
\end{equation}
\begin{align*}
\mathbb{P}&(q_i>\zeta) \stackrel{a)}{=} \mathbb{P}(\underset{j\in\mathcal{I}, j\neq i}{min}\phi_{ij}>\zeta)= \mathbb{P}(\bigcap_{j\in\mathcal{I}, j\neq i}\phi_{ij}>\zeta)\\
& = \mathbb{P}(\phi_{i1}>\zeta,\phi_{i2}>\zeta,...\phi_{i{N_{\mathcal{I}}}}>\zeta) \stackrel{b)}{\leq} \mathbb{P}(\phi_{i1}>\zeta, \phi_{i2}>\zeta)
\end{align*}
a) is because any angle between an outlier and an inlier is above $\zeta$ by design and hence the minimum angle by an inlier is with an inlier itself. The probability of intersection of a set of events is less than probability of the intersection of a subset, hence b). Since the angles are pairwise independent\cite{cai2012phase}, $\mathbb{P}(\phi_{i1}>\zeta, \phi_{i2}>\zeta) = p_{\mathcal{I}}^2$, which gives the upper bound in the lemma. For the lower bound we make use of results in \cite{kwerel1975most} and pairwise independence to get:
\begin{equation*}
\begin{aligned}
\mathbb{P}(\bigcap_{j\in\mathcal{I}, j\neq i}\phi_{ij}&>\zeta) \geq (N_{\mathcal{I}}-2)\mathbb{P}(\phi_{ij}>\zeta,\phi_{ik}>\zeta)\\
&\qquad\qquad\qquad-(N_\mathcal{I}-3)\mathbb{P}(\phi_{ij}>\zeta)\\
&\qquad=(N_{\mathcal{I}}-2)p_{\mathcal{I}}^2-(N_\mathcal{I}-3)p_{\mathcal{I}}\\
\Rightarrow \mathbb{P}(q_i>\zeta) &\geq (N_{\mathcal{I}}-2)p_{\mathcal{I}}^2-(N_\mathcal{I}-3)p_{\mathcal{I}}
\end{aligned}
\end{equation*}
\end{proof}
The upper bound from Lemma \ref{luseful} can be further tightened by using the following result from \cite{kwerel1975most}.
\begin{lemma}\label{l15}
	When either $\lfloor(N_{\mathcal{I}}-2)p_{\mathcal{I}}+1\rfloor = \lfloor(N_{\mathcal{I}}-1)p_{\mathcal{I}}\rfloor$ or $p_{\mathcal{I}}(2- p_{\mathcal{I}} )\geq \dfrac{\lfloor 1+(N_{\mathcal{I}}-2)p_{\mathcal{I}}\rfloor}{N_{\mathcal{I}}-1}$, with $p_{\mathcal{I}}$ as in Lemma \ref{luseful}, then denoting $z = \lfloor(N_{\mathcal{I}}-2)p_{\mathcal{I}}\rfloor $, we have
	\begin{equation}
	\mathbb{P}(q_i>\zeta) \leq \dfrac{(N_{\mathcal{I}}-1)(N_{\mathcal{I}}-2)p_{\mathcal{I}}^2 - z(2p_{\mathcal{I}}(N_{\mathcal{I}}-1)-(z+1))}{(N_{\mathcal{I}}-z)(N_{\mathcal{I}}-1-z)}
	\end{equation}
\end{lemma}
\begin{proof}
	This is a straightforward substitution in condition 1 b) in Corollary 3 to Theorem 3 in \cite{kwerel1975most}.
\end{proof}
\begin{proof}[\textbf{Proof of Lemma \ref{lrough} and Lemma \ref{lroughnoise}}]
		We know that for $i \in \mathcal{I}$, by Jensen's inequality, $\mathbb{E}(q_i) \leq \mathbb{E}(\phi_{ij})$, $i,j \in \mathcal{I}$. Under Assumption 1, results in Corollary \ref{c2} can be used to obtain the value for $\mathbb{E}(\phi_{ij})$, $i,j \in \mathcal{I}$ which gives, 
	\begin{equation}\label{einlemma}
	\mathbb{E}(q_i) \leq \dfrac{\pi}{2} - \sqrt{\frac{2}{\pi(r-2)}}
	\end{equation}
	For the algorithm to recover a sizable amount of inliers we want this expected inlier score to be less than the threshold $\zeta$. If the upper bound in the previous equation is less than $\zeta$, we achieve this objective. Hence the condition is
	\begin{align*}
	\dfrac{\pi}{2} - &\sqrt{\frac{2}{\pi(r-2)}} \leq \dfrac{\pi}{2} - \dfrac{C_N}{\sqrt{n-2}}
	\end{align*}
	\begin{align*}
	\Rightarrow \dfrac{C_N^2}{n-2} \leq \frac{2}{\pi(r-2)} \Rightarrow r \leq \dfrac{2(n-2)}{\pi C_N^2}+2 
	\end{align*}
	When there is added Gaussian noise in the inliers, from Lemma \ref{lnoise}, we know $\mathbb{E}(\phi_{ij})$ increases by at most $\Delta\theta_{w.c} \leq 2 \cos^{-1}(1-\frac{1}{2\sqrt{snr}})$. Hence then $\mathbb{E}(\phi_{ij})\leq \dfrac{\pi}{2} - \sqrt{\frac{2}{\pi(r-2)}}+2 \cos^{-1}(1-\frac{1}{2\sqrt{snr}})$. thus equation (\ref{einlemma}) changes to:
	\begin{equation}\label{einlemmanoise}
	\mathbb{E}(q_i) \leq \dfrac{\pi}{2} - \sqrt{\frac{2}{\pi(r-2)}}+2 \cos^{-1}(1-\frac{1}{2\sqrt{snr}})
	\end{equation} 
	The condition changes to:
	\begin{equation*}
	 \dfrac{\pi}{2} - \sqrt{\frac{2}{\pi(r-2)}}+2 \cos^{-1}(1-\frac{1}{2\sqrt{snr}}) \leq \dfrac{\pi}{2} - \dfrac{c}{\sqrt{n-2}}
	\end{equation*}
	which when simplified gives the condition in Lemma \ref{lroughnoise}.
\end{proof}
\begin{proof}[\textbf{Proof of Theorem \ref{trevcond}}]
For an algorithm to have ERP($\alpha$),
\begin{align*}
\mathbb{P}(\hat{\mathcal{I}}=\mathcal{I})\geq 1-\alpha \Rightarrow \alpha \geq 1-\mathbb{P}(\hat{\mathcal{I}}=\mathcal{I})\
\end{align*}
We know $1-\mathbb{P}(\hat{\mathcal{I}}=\mathcal{I}) = \mathbb{P}(\text{Missing at least 1 inlier})$. It misses an inlier $\mathbf{x}_i$, if $q_i>\zeta$. Hence for the algorithm to have ERP($\alpha$), $ \alpha \geq  \mathbb{P}(\bigcup_{i \in \mathcal{I}}q_i>\zeta)$.
which means the algorithm does not have ERP($\alpha$) if $\alpha <  \mathbb{P}(\bigcup_{i \in \mathcal{I}}q_i>\zeta)$. We know,
\begin{equation*}
\begin{aligned}
\mathbb{P}(\bigcup_{i \in \mathcal{I}}q_i>\zeta) &\geq \mathbb{P}(q_i>\zeta)\hskip10pt i \in \mathcal{I}\\
&\geq (N_{\mathcal{I}}-2)p_{\mathcal{I}}^2-(N_\mathcal{I}-3)p_{\mathcal{I}}
\end{aligned}
\end{equation*}
The last step is from (\ref{eqibnd}) in Lemma \ref{luseful}. If $\alpha$ is less than this lower bound, then it cannot be greater than $\mathbb{P}(\bigcup_{i \in \mathcal{I}}q_i>\zeta)$. i.e. the algorithm is guaranteed to not have ERP($\alpha$) if $\alpha \leq (N_{\mathcal{I}}-2)p_{\mathcal{I}}^2-(N_\mathcal{I}-3)p_{\mathcal{I}}$.
\end{proof}
\begin{proof}[\textbf{Proof of Lemma \ref{lerpb}}]
	We will look at the probability of recovering all the inliers. For this $q_i, \forall i \in \mathcal{I}$ need to be $\leq \zeta$
	\begin{align*}
	\mathbb{P}(\hat{\mathcal{I}} = \mathcal{I}) &= \mathbb{P}(\bigcap_{i \in \mathcal{I}}q_i \leq \zeta)=1-\mathbb{P}(\bigcup_{i \in \mathcal{I}}q_i>\zeta)\\
	&\geq 1-N_\mathcal{I}\mathbb{P}(q_i>\zeta)
	\end{align*}
	The last step is by union bound and the identical nature of distributions. This proves the lemma.
\end{proof}
\begin{proof}[\textbf{Proof Lemma \ref{lnoise}}:]
	Here we will look at the worst case, when both the vectors are perturbed such that angle moves away from each other. Since both these changes are statistically the same, we will denote the angle between a vector $\mathbf{m}$ and the perturbed vector $\mathbf{m}+\textbf{e}$ by $\Delta\theta$ and on an average the worst case change in angle is $2\mathbb{E}(\Delta\theta)$. We will hence look at $\Delta\theta$.
	\begin{align*}
	\cos(\Delta\theta) &= \dfrac{(\mathbf{m}+\textbf{e})^T\mathbf{m}}{\|\mathbf{m}\|_2\|\mathbf{m}+\textbf{e}\|_2}= \dfrac{\|\mathbf{m}\|_2^2+\textbf{e}^T\mathbf{m}}{\|\mathbf{m}\|_2\|\mathbf{m}+\textbf{e}\|_2}\\
	&= \dfrac{\|\mathbf{m}\|_2^2+\frac{\|\mathbf{m}+\textbf{e}\|_2^2 - \|\mathbf{m}\|_2^2-\|\textbf{e}\|_2^2 }{2}}{\|\mathbf{m}\|_2\|\mathbf{m}+\textbf{e}\|_2} \\
	&= \dfrac{(\|\mathbf{m}\|_2-\|\textbf{e}\|_2)(\|\mathbf{m}\|_2+\|\textbf{e}\|_2)+\|\mathbf{m}+\textbf{e}\|_2^2}{2\|\mathbf{m}\|_2\|\mathbf{m}+\textbf{e}\|_2}\\
	&\geq\dfrac{\|\mathbf{m}\|_2-\|\textbf{e}\|_2+\|\mathbf{m}+\textbf{e}\|_2}{2\|\mathbf{m}\|_2}\hskip20pt\text{(triangle inequality)}
	\end{align*}
	Taking expectation and using Jensen's inequality that $\mathbb{E}(\|\mathbf{m}+\textbf{e}\|_2)\geq\|\mathbb{E}(\mathbf{m}+\textbf{e})\|_2 $ and using $\mathbb{E}(\mathbf{m}+\textbf{e}) = \mathbf{m}$
	\begin{align*}
	\mathbb{E}(\cos(\Delta\theta)) &\geq \dfrac{1}{2}+\dfrac{\|\mathbb{E}(\mathbf{m}+\textbf{e})\|_2}{2\|\mathbf{m}\|_2} - \dfrac{\mathbb{E}(\|\textbf{e}\|_2)}{2\|\mathbf{m}\|_2}= 1- \dfrac{\mathbb{E}(\|\textbf{e}\|_2)}{2\|\mathbf{m}\|_2}
	\end{align*}
	We know $\mathbb{E}(\|\textbf{e}\|_2) = \mathbb{E}(\sqrt{e_1^2+e_2^2+...e_n^2})$. By Jensen's inequality, since square root is concave, $\mathbb{E}(\sqrt{e_1^2+e_2^2+...e_n^2}) \leq \sqrt{\mathbb{E}(e_1^2+e_2^2+...e_n^2)} \Rightarrow \mathbb{E}(\|\textbf{e}\|_2) \leq \sqrt{n}\sigma$. Using this result,
	\begin{align*}
	\mathbb{E}(\cos(\Delta\theta)) &\geq 1- \dfrac{\sqrt{n}\sigma}{2\|\mathbf{m}\|_2} = 1-\frac{1}{2\sqrt{snr}}
	\end{align*}
	We assume that the noise does not rotate the vector by more than $\pi/2$ and hence change in angle is always acute. In this setting $\cos$ is a concave function, and hence $\cos(\mathbb{E}(\Delta\theta))\geq \mathbb{E}(\cos(\Delta\theta))$. Thus we can derive,
	\begin{align*}
	\cos(\mathbb{E}(\Delta\theta)) &\geq 1-\frac{1}{2\sqrt{snr}}\\
	\Rightarrow \mathbb{E}(\Delta\theta) &\leq \cos^{-1}(1-\frac{1}{2\sqrt{snr}})
	\end{align*}
	Hence $\Delta\theta_{w.c} = 2\mathbb{E}(\Delta\theta) \leq 2 \cos^{-1}(1-\frac{1}{2\sqrt{snr}})$. 
\end{proof}
\section{Proofs of results in Section \ref{sna}}\label{app4}
\begin{proof}[\textbf{Proof of Lemma \ref{tsoui}}]
Part a): Under Assumption 2, when $\theta_{max}^{\mathcal{O}} \leq \zeta$, all the outlier scores $na_i^\zeta = N_{\mathcal{I}}$ $ w.p \geq 1-\frac{N_{\mathcal{O}}^sN_{\mathcal{I}}}{N^2(N-1)}$, since the angles made with all the other outliers are below $\zeta$. Hence using the $na^\zeta_i$ value for the reference outlier, all outliers will be classified to one cluster and the other cluster only has inliers.\\
Part b): Take any $i \in \mathcal{I}$, with inliers following Assumption 1. From theorem \ref{tna}, we know $na^\zeta_i\geq N_{\mathcal{O}}$ with high probability. Hence it can be expressed as:
\begin{align*}
na^\zeta_i = N_{\mathcal{O}}^s + \sum\limits_{j \in \mathcal{I}, j \neq i}\mathbb{I}_{\{\phi_{ij}>\zeta\}}
\end{align*}
where $\mathbb{I}_{\{\phi_{ij}>\zeta\}}$ is an indicator random variable which is 1 when $\phi_{ij}>\zeta$.  We know for an inlier, from (\ref{ephiin}) , $\mathbb{P}(\phi_{ij}>\zeta) =(2F_{\mathcal{N}}(C_N\sqrt{\frac{r-2}{n-2}})-1)$. Then $\mathbb{E}(\mathbb{I}_{\{\phi_{ij}>\zeta\}}) = (2F_{\mathcal{N}}(C_N\sqrt{\frac{r-2}{n-2}})-1)$. Hence for $i \in \mathcal{I}$
\begin{align*}
\mathbb{E}(na^\zeta_i) &= N_{\mathcal{O}}^s + \sum\limits_{j \in \mathcal{I}, j \neq i}\mathbb{E}(\mathbb{I}_{\{\phi_{ij}>\zeta\}})\\
\Rightarrow \mathbb{E}(na^\zeta_i) - N_{\mathcal{O}}^s &= (N_{\mathcal{I}}-1)(2F_{\mathcal{N}}(C_N\sqrt{\frac{r-2}{n-2}})-1)
\end{align*}
In this part $N_\mathcal{I}-N_{\mathcal{O}}^s = \delta N$ and due to the assumption $na^\zeta_i  = N_{\mathcal{I}}, \forall i \in \mathcal{O}$ with high probability.  If the expected inlier $na^\zeta_i$ value is less than $N_{\mathcal{O}}^s+ \frac{\delta N}{2}$, then the classification of an inlier to the cluster containing outliers would not happen on an average. The condition then becomes:
\begin{align*}
(N_{\mathcal{I}}-1)(2F_{\mathcal{N}}(C_N\sqrt{\frac{r-2}{n-2}})-1) &< \frac{\delta N}{2}\\
\Rightarrow \delta N>2(N_{\mathcal{I}}-1)&(2F_{\mathcal{N}}(C_N\sqrt{\frac{r-2}{n-2}})-1)
\end{align*}
Part c): If $N_{\mathcal{I}}<N_{\mathcal{O}}$, then since by assumption on $\theta_{max}^{\mathcal{O}}$, $na^\zeta_i = N_{\mathcal{I}}, \forall i \in \mathcal{O}$ $w.p \geq 1-\frac{N_{\mathcal{O}}^sN_{\mathcal{I}}}{N^2(N-1)}$ and $na^\zeta_i > N_{\mathcal{O}}, \forall i \in \mathcal{I}$ $w.p \geq 1-\frac{N_{\mathcal{O}}^sN_{\mathcal{I}}}{N^2(N-1)}$. Hence for any $i \in \mathcal{I}, na^\zeta_i $ will be always closer to another inlier score than $N_{\mathcal{I}}$. This means the clustering is always exact $w.p \geq 1-\frac{2N_{\mathcal{O}}^sN_{\mathcal{I}}}{N^2(N-1)}$
\end{proof}
\begin{proof}[\textbf{Proof of Lemma \ref{tsosi}}]
	Part a): Since the assumption is made that $\theta_{max}^{\mathcal{I}} \leq \zeta$ and $\theta_{max}^{\mathcal{O}} \leq \zeta$, $\forall i \in \mathcal{I}, na^\zeta_i = N_{\mathcal{O}}$ and $\forall i \in \mathcal{O}, na^\zeta_i = N_{\mathcal{I}}$ each $w.p \geq 1-\frac{N_{\mathcal{O}}^sN_{\mathcal{I}}}{N^2(N-1)}$. This is because there ae no angles within each set above $\zeta$. If $N_{\mathcal{I}} \neq N_{\mathcal{O}}$, then the clustering is exact $w.p \geq 1-\frac{2N_{\mathcal{O}}^sN_{\mathcal{I}}}{N^2(N-1)}$.\\
	Part b): Along with conditions in part a), If $\theta_{max}^{\mathcal{I}} \leq \theta_{min}^{\mathcal{O}}$, then the reference point which the algorithm nominally denotes as inlier reference point is always an inlier since an outlier cannot have an angle less than $\theta_{min}^{\mathcal{O}}$. This would imply that $\hat{\mathcal{I}}_{op} = \mathcal{I}$ with the same probability $ \geq 1-\frac{2N_{\mathcal{O}}^sN_{\mathcal{I}}}{N^2(N-1)}$.
\end{proof}
{
	\bibliographystyle{IEEEtran}
	\bibliography{arxiv_draft}
}

\end{document}